\documentclass{article}

\usepackage[english]{babel}

\usepackage[letterpaper,top=2cm,bottom=2cm,left=3cm,right=3cm,marginparwidth=1.75cm]{geometry}

\usepackage{amsmath, amsthm}
\usepackage{amssymb}
\usepackage{graphicx}
\usepackage{subcaption}
\usepackage[colorlinks=true, allcolors=blue]{hyperref}
\usepackage{hyperref}
\usepackage{url}
\usepackage{authblk}
\usepackage{natbib}
\usepackage[shortlabels]{enumitem}

\DeclareMathOperator\supp{supp}

\newtheorem{theorem}{Theorem}

\newtheorem{proposition}{Proposition}[section]
\newtheorem{lemma}{Lemma}[section]
\newtheorem{corollary}{Corollary}[section]
\newtheorem{remark}{Remark}[section]
\newtheorem*{definition}{Definition}
\newtheorem{example}{Example}

\newcommand{\sapprox}{\tilde\sigma_\theta}
\newcommand{\smean}{\bar\sigma}
\newcommand{\gmax}{\gamma_{\max}}

\title{The Effect of Stochasticity in Score-Based Diffusion Sampling: a KL Divergence Analysis}

\author[a]{Bernardo P. Schaeffer}
\author[a]{Ricardo M. S. Rosa}
\author[a]{Glauco Valle}

\affil[a]{Instituto de Matemática, Universidade Federal do Rio de Janeiro}

\begin{document}
\maketitle

\begin{abstract}
%\boldmath
Sampling in score-based diffusion models can be performed by solving either a reverse-time stochastic differential equation (SDE) parameterized by an arbitrary stochasticity function or a probability flow ODE, corresponding to setting this stochasticity function to zero. In this work, we investigate the effect of this stochasticity on the generation process through the evolution of Kullback-Leibler (KL) divergences, obtaining general KL divergence bounds and a novel analysis of the impact of the time-profile of the score error on model performance. For exact score functions, stochasticity has a contractive effect, decreasing KL divergence along the sampling trajectory. For approximate scores, however, a trade-off arises between correcting accumulated errors and amplifying current score errors, meaning stochasticity can either improve or degrade generation performance. Theoretical considerations indicate that the gain from stochasticity depends on the time-localization of the trained model error. We test this in experiments on both toy and benchmark data sets, also comparing the KL divergence evolution with the obtained bounds. We also present a fully analytical example, where all the relevant quantities can be computed, and the optimal stochasticity function can be characterized via an optimal control analysis.
\end{abstract}

\tableofcontents

\section{Introduction}\label{sec: intro}

Score-based diffusion models are a generative modeling framework based on the fact that the process of a stochastic differential equation (SDE) admits a time-reversal process, which is given by another SDE depending on the score of the (forward) distribution \citep{karras, song2021scorebasedmodeling}. The forward process is used for training, and the reverse process is used for sampling. Since the main interest is in the distribution itself and not in particular paths of the SDE, a whole family of SDEs, indexed by a stochasticity parameter $\gamma=\gamma(t)$, including an ordinary differential equation (the probability flow ODE, corresponding to $\gamma(t) = 0$), can be used for sampling. A major motivation of this work is to provide insight into the experimental evidence that stochasticity can either benefit or harm performance \citep{song2021implicitmodels, song2021scorebasedmodeling, karras}.

We estimate the evolution, along the sampling process, of the Kullback-Leibler divergences (KL divergence), also known as relative entropies, between the generated and the target distributions, as they depend on the free stochasticity parameter. At each sampling instant, the KL evolution is affected both by accumulated errors (from the starting sampling distribution and/or previous steps of sampling) and by the error resulting from the approximation of the score function at the current instant. The interaction between these two kinds of error can be used to understand the influence (beneficial or detrimental) of stochasticity on the final performance, extending previous approaches to the study of stochasticity in diffusion model sampling \citep{SDEDrag,opt_choice, ma2024}.
We conduct numerical experiments to further investigate the effect of this balance of errors. We also present theoretical error estimates in a general setting, including estimates implicitly discussed in the literature, and use numerical experiments to illustrate the obtained theoretical results.
Experiments are performed on toy data sets, where all quantities can be accurately computed; on typical image benchmarks such as MNIST and CIFAR-10; and also on a geological 3D problem, where, besides visual accuracy, it is important to match the distribution of rock statistics. Finally, we provide a detailed analysis of a fully analytical example.

Diffusion models are currently the state of the art for generative image and video synthesis \citep{ramesh2022hierarchicaltextconditionalimagegeneration,videoworldsimulators2024}, and are also being used successfully in many scientific and technological applications, such as natural language processing \citep{zou2023surveydiffusionmodelsnatural}, text-to-speech synthesis \citep{zhang2023surveyaudiodiffusionmodels}, medical imaging \citep{kazerouni2023diffusionmodelsmedicalimage}, time series analysis \citep{lin2023diffusionmodelstimeseries}, protein modeling and design \citep{alphafold_3,li2025thermodynamicsproteindesigndiffusion}, bioinformatics \citep{bioinformatics}, global weather forecasting \citep{gencast}, porous media synthesis \citep{frontgeologico}, and more \citep{yang2024diffusionmodelscomprehensivesurvey}. Despite their remarkable success in applications, the precise evaluation of the performance of a generative model presents challenges such as the independent quantification of sample quality and distribution coverage \citep{precision/recall_2018, feature_lik_divergence_2023}. Informally, sample quality can be defined as the plausibility of generated samples in the context of original data, and distribution coverage as the amount to which the original distribution is captured by the generated distribution. Some examples of poor sample quality are low-resolution generated images and hallucinations \citep{aithal}, while examples of poor distribution coverage include mode collapse and lack of variability \citep{Srivastava2017Veegan, Shumailov2024, frontgeologico}.

The KL divergence is a standard information-theoretic measure of discrepancy between probability distributions. Due to its asymmetry, it provides two distinct notions of approximation to compare a generated distribution $\tilde{p}$ with a target distribution $p$ through the divergences $H(\tilde{p}|p)$ and $H(p|\tilde{p})$. Both divergences measure generated sample quality and target distribution coverage, but $H(\tilde{p}|p)$ places more weight on the former, whereas $H(p|\tilde{p})$ emphasizes the latter. These properties are referred to in variational inference as zero-forcing and zero-avoiding behavior; see, for instance, \citep[Section 21.2]{Murphy2012MachineL} and \citep[Section 10.1]{bishop}. The difference can be seen from the definition $H(p_1|p_2):=\int p_1 \log (p_1/p_2)\,dx$: the integrand is large where $p_2(x)\ll p_1(x)$, so that these regions contribute more to $H(p_1|p_2)$ than to $H(p_2|p_1)$. Different applications may be more sensitive either to sample quality or to distribution coverage, making both divergences relevant for evaluating a generative model from a theoretical perspective. Indeed, standard metrics for image generation, such as Inception Score (IS) or Fréchet Inception Distance (FID), typically measure both aspects in a single score \citep{feature_lik_divergence_2023}.

Motivated by this, we investigate the impact of stochastic sampling on model performance through both divergences. In particular,  the knowledge available for the target distributions $p_t$ facilitates the use of log-Sobolev inequalities (LSI) for a more precise control of $H(\tilde{p}_t|p_t)$ (associated with sample quality) from above. The analysis of $H(\tilde{p}_t|p_t)$ complements the previous literature on KL divergence bounds for diffusion models, which typically considers only $H(p|\tilde{p})$ for obtaining fully data-independent bounds (see Section \ref{sec: related work}). A more careful study of $H(\tilde{p}|p)$ is also motivated by the observed occurrence of hallucinations in diffusion models \citep{aithal}. In Section \ref{sec: different KLs}, we show an example where stochasticity can be beneficial for one divergence while being detrimental to the other.

The framework we consider is the one introduced by \cite{song2021scorebasedmodeling} and advanced by \cite{karras}, with the diffusion of the data distribution defined by the (forward) SDE
\begin{equation}\label{sde}
    dX_t = f(X_t,t)\,dt + g(t)\,dW_t,
\end{equation}
in $\mathbb{R}^n,$ on a time interval $[0, T],$ with probability density $p_t(x),$ where $f=f(x, t)$ and $g=g(t)$ are given and $\{W_t\}_{t\geq 0}$ is an $n$-dimensional Wiener process. The associated reverse-time SDE takes the form
\begin{equation}\label{rev sde}
    d\tilde{X}_\tau = \left(-\bar{f}(\tilde{X}_\tau, \tau) + \frac{1}{2}\bar{g}^2(\tau)(1+\gamma(\tau))\nabla\log \bar{p}_\tau(\tilde{X}_\tau)\right)\,d\tau + \sqrt{\gamma(\tau)} \bar{g}(\tau) \,dW_\tau,
\end{equation}
in the reversed time variable $\tau=T-t,$ where $\bar{f}(x, \tau):=f(x,T-\tau)$, $\bar{g}( \tau):=g(T-\tau)$, $\bar{p}_\tau(x):=p_{T-\tau}(x),$ and where $\gamma(\tau) \geq 0$ is arbitrary. It is well known~\citep{anderson, karras} that, under compatible initial conditions, the reverse-time SDE \eqref{rev sde} has the same densities as the forward SDE \eqref{sde} under the change of variables $\tau=T-t$. This result is condensed in Lemma \ref{lemma rev time}, with a short proof provided for the sake of completeness. Note that \eqref{rev sde} gives a family of equations, indexed by the free functional parameter $\gamma(\tau)\geq0$.

If the (Stein) score function $\nabla\log p(t,\cdot)$ of the forward process given by \eqref{sde} is known exactly or approximately, for all $0<t\leq T,$ and if at time $t=T$ the distribution $p_T$ is close to a Gaussian distribution or any other distribution which is easy to sample from, the reverse-time SDE \eqref{rev sde} can then be used as a generic sampler for $p_0.$ In practice, the score function $\nabla\log p(t,\cdot)$ can be approximated by a neural network $s_\theta(\cdot, t)$ trained with denoising score-matching \citep{denoisingSM, karras}. Then, sampling is performed via the approximate reverse-time SDE
\begin{equation}\label{rev sde approx}
    d\tilde{X}_\tau = \left(-\bar{f}(\tilde{X}_\tau, \tau) + \frac{1}{2}\bar{g}^2(\tau)(1+\gamma(\tau))\bar{s}_\theta(\tilde{X}_\tau, \tau)\right)\,d\tau + \sqrt{\gamma(\tau)} \bar{g}(\tau) \,dW_\tau,
\end{equation}
where $\bar{s}_\theta(x, \tau) = s_\theta(x, T - \tau)$ is the reverse-time approximate score. In the machine-learning literature \citep{song2021scorebasedmodeling}, two common choices for $\gamma$ correspond to setting $\gamma\equiv0$ (Probability Flow ODE), and $\gamma\equiv1$ (Anderson/Song's reverse-time SDE).

In summary, the diffusion model algorithm consists of
\begin{enumerate}
    \item choosing an adequate forward process \eqref{sde} for which $p_T$ is close to a known and easy-to-sample distribution $q$;
    \item training a score model $s_\theta(x, t)$ to approximate $\nabla\log p_t(x),$ for all $0<t\leq T,$ by minimizing the global loss function of denoising score-matching;
    \item sampling the starting condition $\tilde{X}_0$ from the prior distribution $q(x)\approx p_T(x);$ and
    \item numerically integrating, from $\tau=0$ to $\tau=T,$ the approximate reverse-time equation \eqref{rev sde approx} with the starting condition $\tilde{X}_0.$
\end{enumerate}
This pipeline generates a final sample $\tilde{X}_T\sim\tilde{p}_T$, with the aim that $\tilde{p}_T$ be as close as possible to the unknown data distribution $p_0$. Its main sources of error are
\begin{enumerate}[label = {(E\arabic*)}]
    \item \label{errorscore} approximation errors of the score function;
    \item \label{errorprior} approximation errors of the starting prior distribution $p_T$;
    \item \label{errornumeric} numerical integration errors of the (approximate) reverse-time equation.
\end{enumerate}

In this work, we investigate the effect of the free parameter $\gamma(t)$ on the correction and propagation of errors from sources \ref{errorscore} and \ref{errorprior}. Thus, for theoretical results, we consider the exact solution of either \eqref{rev sde} or \eqref{rev sde approx}, while for the numerics we simply use Euler and Euler-Maruyama schemes with a very fine time mesh. Regarding discretization, it is well known that deterministic sampling has the advantage of achieving good performance even with a small number of steps, which makes it much faster in practice since the main computational burden of sampling is in the evaluation of the neural score function. Computational efficiency has also motivated the development of fast stochastic samplers \citep{jolicoeurmartineau2021gottafastgeneratingdata, bao2022analyticdpm, xu2023restart, xue2023sasolver}. In the literature, an implicit assumption often found is that the only hurdle to stochastic sampling comes from discretization, which has been shown in \cite{karras} to not be the case for all trained models. Therefore, the search for more efficient discretization schemes for both deterministic and stochastic sampling can benefit from a more careful study of whether stochasticity is advantageous in a continuous setting.

In Section \ref{subsection: exact scores}, we consider the exact-score reverse-time SDE \eqref{rev sde} with an approximate starting prior, isolating the source \ref{errorprior} of error. In this case, Theorem \ref{teo 2} explicitly shows the advantage of stochasticity, in the form of KL divergence bounds which are decreasing in $\gamma(t)$. A previous result has been obtained in \cite[Section D.4]{SDEDrag}, restricted to the choices of $\gamma(\tau)=\gamma=1$ and $\gamma(\tau)=\gamma=0$. As we discuss in Section \ref{rem: gamma and discretization}, increasing $\gamma$ makes the problem stiff, adding an effective faster time scale to the dynamics. This poses a challenge to the numerical approximation of the problem, limiting the choice of $\gamma$ in practice.

Next, in Section \ref{subsection: approx scores}, we consider the more realistic setting of an approximate score function $s(x, t) = \nabla\log p_t(x) + \epsilon(x,t),$ where $\epsilon=\epsilon(x,t)$ is the associated error, addressing both errors \ref{errorprior} and \ref{errorscore}. Here, stochasticity still produces a correction term in the evolution of the generated distribution, but it also amplifies the score error $\epsilon$ by a multiplicative factor, making its overall impact more subtle. In Proposition \ref{prop 2}, we 
explicitly state the conditions to ensure that the formula for the time derivative of the KL divergence holds. Then, if $\bar{p}_\tau$ satisfies a log-Sobolev inequality with constant $C_\text{LSI}(\tau),$ as in \eqref{HLSI}, Theorem \ref{teo 4} yields the KL divergence estimate
\[ 
    H(\tilde{p}_\tau|\bar{p}_\tau) \leq e^{-\int_0^\tau \alpha(s)\,ds} H(\tilde{p}_0|\bar{p}_0) + \frac{1}{2}\int_0^\tau \bar{g}(s)^2(1+\gamma(s))E_{\tilde{p}_s}\left[\epsilon_s\cdot\nabla\log h_s\right]e^{-\int_s^\tau \alpha(r)\,dr}\,ds,
\]
where $\alpha(s)=C_\text{LSI}(s)^{-1} \gamma(s) \bar{g}(s)^2$ and $h_s=\tilde{p}_s/\bar{p}_s$. When $\gamma$ is strictly positive, a more explicit estimate is given by equation \eqref{perturbed bound 2}, which eliminates the explicit dependence on $\nabla\log h$ and depends on $\epsilon$ only through its $L^2(\tilde{p})$ norm.

In this setting, we obtain in Section \ref{subsec: instantaneous_opt} a qualitative criterion for the final impact of the stochastic term by minimizing the time derivative of the KL divergence. In broad terms, stochasticity is beneficial at sampling time $\tau$ when the ratio between accumulated errors and the current score error surpasses a certain threshold. This would make the performance of stochastic sampling highly dependent on the time-profile of the score error: an error that increases too rapidly near the end of the sampling process would make deterministic sampling preferable, whereas an error concentrated near the beginning or middle of the process would allow the model to benefit from stochasticity. This generalizes to the non-asymptotic case (finite $\gamma$, intermediate $\tau$ values) the main conclusion in \cite{opt_choice}.

Section \ref{sec: experiments} shows numerical experiments on several data sets. First, we compare the estimated values of the KL divergences with the theoretical bounds, on two toy data sets with different available LSI constants, sampling with both exact and learned score functions. Then, for toy and also more complex data sets such as MNIST and CIFAR-10, we relate the insights of Section \ref{subsec: instantaneous_opt} to the time-profile of the score error of trained models, identifying factors that contribute to the observed differences in the effect of $\gamma$ across trained models. Finally, in Section \ref{subsec: variable_gamma} we investigate the case of variable $\gamma=\gamma(t)$, performing a grid-search over intervals to include stochasticity, considering step functions $\gamma(t)$. Based on the results and the theoretical considerations, we believe that, in general, the best choice is to perform a grid-search near the end of sampling, searching for intervals $[t_1, t_2]$ with $t_1>\epsilon$ and $t_2\ll T$.

Finally, Section \ref{sec: analytical example} presents the analysis of a fully analytical example, based on a single Gaussian data set with linear perturbations in the score. In this setting, the reverse-time SDE can be explicitly solved, all the relevant quantities can be computed, and it is possible to solve the optimal control problem for the choice of $\gamma(t)$ minimizing each KL divergence. Remarkably, even in this simple setting the effect of $\gamma$ is dependent on the relation between the score error and the prior error, which can also be thought of as the error accumulated from the initial part of sampling.

\subsection{Related work}\label{sec: related work}

Some recent articles also obtain KL divergence bounds for score-based diffusion models. For instance, \cite{Chen2022Improved, benton2024nearly, conforti2025klconvergence} provide complexity bounds for $H(p_\text{data}|p_\text{generated})$, accounting for errors arising from the prior distribution approximation, score approximation and discretization errors. These works consider a forward Ornstein-Uhlenbeck diffusion, and the particular choice of $\gamma\equiv1$. Working with $H(p_\text{data}|p_\text{generated})$ makes the expectation over the score error to be taken with respect to $p_\text{data}$ instead of $p_\text{generated}$, which produces bounds that are entirely independent of the generated distribution.
Although our motivation is different, there are similarities in the techniques, and, in particular, the formula in our Lemma \ref{lemma H_deriv} corresponds to the formula in Lemma C.1 of \cite{Chen2022Improved}.

Bounds for $H(p_\text{data}|p_\text{generated})$ are also obtained in \cite[Section 2.4]{stoch_interp}, for more general forward processes in the context of stochastic interpolants. For a general reverse SDE with positive stochasticity parameter, it is shown that this divergence is bounded by the score-matching error, whereas a similar result is not guaranteed to hold for ODE sampling. A similar result has been previously obtained in \citep[Section 4.1]{song2021maxlikelihood} for $\gamma\equiv1$.  The papers \cite{ma2024, chen2024follmer} extend the work of \cite{stoch_interp} by proposing a choice of $\gamma(t)$ which minimizes their bound. In our formulation, this corresponds to setting $\gamma\equiv 1$, as discussed in more detail in Remark \ref{remark: opt_gamma}.

For ODE sampling, bounds for both KL divergences can be found in \cite[Section 3]{boffi-vanden}. These bounds are formulated in terms of the implicit score of the generated distribution and the learned score,
and are used for establishing the convergence of a self-consistent training procedure instead of evaluating the impact of the score error on model performance.

Some recent works specifically analyze the effect of stochasticity at sampling. In \cite{SDEDrag}, under the hypothesis of exact score functions, log-Sobolev inequalities for $p_t$ are used to obtain a bound for $H(p_\text{generated}|p_\text{data})$ which shows the decrease in this divergence along sampling when $\gamma\equiv1$, while when $\gamma\equiv0$ the divergence remains constant. In \cite{opt_choice}, an asymptotic analysis of $H(p_\text{data}|p_\text{generated})$ with $(\text{score error})\to0$ is performed to compare the ODE with the asymptotic SDE $\gamma\to\infty$. Their conclusion reflects a particular aspect of the analysis in Section \ref{subsec: instantaneous_opt}: if the score errors are located in a region in the beginning or middle of the sampling process, maximum stochasticity is optimal, while stochasticity is altogether detrimental if the score errors are located exactly at the end of sampling.

In comparison, our work provides more complete KL bounds than those available in the literature, accounting for both divergences, general forward processes, and arbitrary stochasticity functions $\gamma(t)$. We also rely less on formal arguments, 
%giving rigorous
explicitly stating the conditions for the results to hold.
Moreover, the theoretical insights on the evolution of the KL divergences provide a connection between the time-profile of the score error and the effect of stochasticity in a non-asymptotic setting, which is closer to practice.

\section{Background}\label{background}

We present here some background results to be used in the next sections. For completeness, proofs are provided in Appendix \ref{app: proofs sec2}, since these results appear in the literature in a slightly different form.

\subsection{Reverse-time formula}\label{rev time}

Consider a (forward) SDE on the interval $[0,T]$, of the form
\begin{equation}\label{sde 2}
    dX_t = f(X_t,t)\,dt + g(t)\,dW_t,
\end{equation}
where $f$ and $g$ are smooth and $f$ is globally Lipschitz with respect to $x,$ uniformly on $t\in [0, T]$, so that the corresponding initial-value problem admits a unique solution.
Let $p_t$ be its associated density, which solves the Fokker-Planck (FP) equation
\begin{equation}\label{fp}
    \frac{\partial p}{\partial t} = -\nabla\cdot(fp) +\frac{1}{2}g^2\Delta p.
\end{equation}
As mentioned in Section \ref{sec: intro}, we extensively use the following result from \cite[Appendix B.5]{karras}. Earlier formulations can be found in \cite{anderson, song2021scorebasedmodeling}.

\begin{lemma}\label{lemma rev time}
    Consider the family of SDEs given by
    \begin{equation}\label{rev sde 2}
        d\tilde{X}_\tau = \left(-\bar{f}(\tilde{X}_\tau, \tau) + \frac{1}{2}\bar{g}^2(\tau)(1+\gamma(\tau))\nabla\log \bar{p}_\tau(\tilde{X}_\tau)\right)\,d\tau + \sqrt{\gamma(\tau)} \bar{g}(\tau) \,dW_\tau,
    \end{equation}
    for $\tau\in[0,T]$, where $\gamma(\tau)$ is an arbitrary non-negative function, and $\bar{f}(x, \tau):=f(x,T-\tau)$, $\bar{g}( \tau):=g(T-\tau)$ and $\bar{p}_\tau(x):=p_{T-\tau}(x)$. Denote by $\tilde{p}_\tau^\gamma$ the probability densities for the solution of the SDE with functional parameter $\gamma$. Suppose that $\nabla\log\bar p_t$ is Lipschitz in $x$, locally uniformly for $t \in (0,T]$, and the initial density matches the one resulting from the forward process (i.e. $\tilde{p}_0^\gamma = \bar{p}_0 = p_T$). Then the densities $\tilde{p}_\tau^\gamma$ are independent of $\gamma$ and equal to the densities of the forward process (i.e.  $\tilde{p}_\tau=\bar{p}_\tau = p_{T-\tau}$) for all $\tau\in[0,T)$. 
\end{lemma}

Since \eqref{rev sde 2} is an SDE in reverse time that has the same densities as the forward SDE \eqref{sde 2}, it is called a \textbf{reverse-time SDE} associated with \eqref{sde 2}. For the sake of completeness, we include a short proof of this result, which can also be found in \cite{karras}, in Appendix \ref{app: proofs sec2}. The proof uses the same formal computation together with a result of \cite{uniqueness_figalli2008} that guarantees the well-posedness of \eqref{rev sde 2}.

\subsection{Logarithmic Sobolev Inequalities (LSI)}

We recall the definition of the logarithmic Sobolev inequality (LSI). The LSI is known to be important for the exponential convergence towards the asymptotic invariant measure for the Langevin equation and plays a key role in our estimates for the reverse-time SDE, as well.

\begin{definition}
    Given a probability density $p(x)$ in $\mathbb{R}^n$, consider the entropy functional as $\text{Ent}_p(f^2):=\int f^2\log f^2 \;p \,dx - \left(\int f^2 \;p \,dx\right)\log \left(\int f^2 \;p \,dx\right)$. Then, $p$ is said to satisfy a logarithmic Sobolev inequality (LSI) with constant $C$ if
    \begin{equation}\label{LSI}
        \text{Ent}_p(f^2) \leq 2 C \int \|\nabla f\|^2 \;p \,dx
    \end{equation}
        for any smooth function $f$.
\end{definition}

More explicitly, we use the following consequence of the LSI.

\begin{lemma}\label{lemma LSI}
    If $p$ satisfies a logarithmic Sobolev inequality, then \begin{equation}\label{HLSI}
        H(\tilde{p}|p)\leq \frac{C}{2}\int\left|\nabla\log\frac{\tilde{p}}{p}\right|^2 \tilde{p}\,dx
    \end{equation}
    for all probability densities $\tilde{p}\ll p$, i.e. absolutely continuous with respect to $p$.
\end{lemma}

For probability distributions in $\mathbb{R}$, the LSI is equivalent to a more explicit integrability characterization \citep{BOBKOV1999}. In the general case, subgaussianity is a necessary condition, and there are widely known sufficient conditions, such as the Bakry-Emmery log-concavity criterion \citep{bakry-emmery} and the Holley-Stroock bounded perturbation lemma \citep{holley-stroock}. For a Gaussian convolution of a given measure, a sufficient condition is that the measure has compact support \citep{dimension-free}, as we use in Section \ref{section: kl evolution and bounds}.

\section{KL divergence evolution and bounds}\label{section: kl evolution and bounds}

It is well known that the probability densities $q_t(x)$ associated with the time-independent Langevin equation
\begin{equation}\label{langevin}
    dX_t = \nabla\log \pi(X_t)\,dt + \sqrt{2}\,dW_t
\end{equation}
converge exponentially to the stationary distribution $\pi$ when $t\to\infty$, under some conditions on $\pi$ \citep{RT1996, villani}. In particular, under the assumption that $\pi$ satisfies an LSI with constant $C$, it is obtained, in \cite{villani}, an exponential bound on the KL divergence, given by
\begin{equation}\label{langevin entropy decrease}
    H(q_t|\pi)\leq e^{-\frac{2t}{C}}H(q_0|\pi).
\end{equation}
Thus, deviations in the initial distribution are corrected over time by the Langevin diffusion.

Considering now the reverse-time SDE \eqref{rev sde 2} instead of the Langevin equation, we can aggregate the terms with $\gamma$ and write it as
\begin{multline}\label{rev sde 3}
    d\tilde{X}_\tau = \left(-\bar{f}(\tilde{X}_\tau, \tau) + \frac{1}{2}\bar{g}^2(\tau)\nabla\log \bar{p}_\tau(\tilde{X}_\tau)\right)\,d\tau \\
    + \left( \frac{1}{2}\bar{g}^2(\tau)\gamma(\tau)\nabla\log \bar{p}_\tau(\tilde{X}_\tau)\,d\tau+ \sqrt{\gamma(\tau)} \bar{g}(\tau) \,dW_\tau\right).
\end{multline}
Notice that the first term on the right-hand side corresponds to the Probability Flow ODE, i.e. the reverse-time SDE with $\gamma\equiv0$, and the second is closely related to the Langevin equation \eqref{langevin}. Inspired by this form, in \cite[Section 4]{karras} it is proposed that the reverse-time SDE should be seen as the Probability Flow ODE with the addition of a Langevin-like error-correcting term. This suggests that a bound similar to \eqref{langevin entropy decrease} might be found for the reverse-time SDE.

This idea can be formulated and studied more precisely by considering the KL divergences between two densities $\tilde{p}_\tau, \bar{p}_\tau$ of the reverse-time SDE, one with a given initial density $\tilde{p}_0$ and the other with $\bar{p}_0=p_T$, matching the forward process.
By Lemma \ref{lemma rev time}, we know that the densities $\bar{p}_\tau$ correspond to the densities of the forward process for all $\tau\in[0, T)$. As discussed in Section \ref{sec: related work}, the KL divergences are usual metrics to evaluate the performance of the diffusion algorithm. In Section \ref{subsection: exact scores}, we obtain a result analogous to \eqref{langevin entropy decrease}, with $\bar{p}_\tau$ assuming the role of the stationary distribution $\pi$. Section \ref{subsection: approx scores} studies the more general case of an approximated reverse-time SDE, in which $\nabla\log \bar{p}_\tau(x)$ is replaced by $s(x,\tau) := \nabla\log \bar{p}_\tau(x) + \epsilon(x, \tau)$, for an approximation error $\epsilon(x, \tau)$.

\subsection{Exact score functions}\label{subsection: exact scores}

In this section, we consider only the approximation error in the initial condition, that is, when $\tilde{p}_\tau$ and $\bar{p}_\tau$ are solutions of the same equation \eqref{rev sde 2}, but with different initial conditions $\tilde{p}_0$ and $\bar{p}_0=p_T$. The proofs in this and the next subsection rely on a general formula for the time derivative of the KL divergence between solutions of Fokker-Planck equations sharing a diffusion coefficient, which we state and prove in Appendix \ref{app: proofs sec3} (Lemma \ref{lemma H_deriv}).

\begin{proposition}\label{prop 1}
    Let $\tilde{p}_\tau$ and $\bar{p}_\tau$ be positive classical solutions, in $C^{2,1}(\mathbb{R}^n \times (0,T))$, of the Fokker-Planck equation associated with the reverse-time SDE \eqref{rev sde 2}, with initial conditions $\tilde{p}_0$ and $\bar{p}_0 = p_T$, respectively. Assume that $\gamma$ is piecewise-continuous on $(0,T)$, that $\tilde{p}_\tau$ and $\bar{p}_\tau$ are bounded above and below by Gaussians, locally uniformly in $\tau \in (0,T)$, and $\nabla\log\tilde{p}_\tau$, $\nabla\log\bar{p}_\tau$ have at most polynomial growth in $x$. Then, for almost all $\tau \in (0,T)$,
    \begin{equation}\label{H deriv 1}
        \frac{d}{d\tau}H(\tilde{p}_\tau|\bar{p}_\tau) = -\frac{1}{2}\gamma(\tau) \bar{g}^2(\tau) \int\left|\nabla\log\frac{\tilde{p}_\tau}{\bar{p}_\tau}\right|^2 \tilde{p}_\tau\,dx
    \end{equation}
    and
    \begin{equation}\label{H deriv rev}
        \frac{d}{d\tau}H(\bar{p}_\tau|\tilde{p}_\tau) = -\frac{1}{2}\gamma(\tau) \bar{g}^2(\tau) \int\left|\nabla\log\frac{\tilde{p}_\tau}{\bar{p}_\tau}\right|^2 \bar{p}_\tau\,dx.
    \end{equation}
\end{proposition}
 
\begin{proof}
    Inside each interval where $\gamma(\tau)$ is continuous, the drift $a(x,\tau) = -\bar{f}(x,\tau) + \frac12\bar{g}^2(\tau)(1+\gamma(\tau))\nabla\log\bar{p}_\tau(x)$ of the Fokker-Planck equation of \eqref{rev sde 2} is continuous, $C^1$ in $x$ (since $\bar{p} \in C^{2,1}$), and of polynomial growth, because $f$ is Lipschitz and $\nabla\log\bar{p}_\tau$ has polynomial growth. Equation \eqref{H deriv 1} then follows from Lemma \ref{lemma H_deriv} (Appendix \ref{app: proofs sec3}) applied to the pair $(\tilde{p}, \bar{p})$ with $a_1 = a_2 = a$, and equation \eqref{H deriv rev} from the same lemma applied to the pair $(\bar{p}, \tilde{p})$.
\end{proof}

\begin{remark}\label{remark: hypotheses scope}
    The hypotheses of Proposition \ref{prop 1} (and of Proposition \ref{prop 2} below) are stated directly on the solutions $\tilde{p}_\tau$, $\bar{p}_\tau$, rather than on the data $(\tilde{p}_0, f, g, \gamma, \epsilon)$. For $\bar{p}$, they are verified in the standard linear-drift setting (see Theorem \ref{teo 2} and Appendix \ref{app: proofs sec3}). For $\tilde{p}$, deriving them from conditions on the data is a nontrivial problem, particularly when $\gamma$ is allowed to degenerate (e.g. $\gamma \equiv 0$), in which case the equation loses its parabolic character. We observe that: (i) when $\gamma$ is bounded below by a positive constant, two-sided Gaussian bounds for $\tilde{p}_\tau$ follow from Aronson-type estimates for the fundamental solution \citep[Theorem 1.2]{MENOZZI2021}, under additional conditions on the coefficients and on $\tilde{p}_0$; (ii) for $\gamma \equiv 0$ and a Gaussian prior, $\tilde{p}_\tau$ is the pushforward of $\tilde{p}_0$ by a $C^1$ flow, and the hypotheses can be verified directly; and (iii) in the fully analytical example of Section \ref{sec: analytical example}, all the distributions involved are Gaussian and the hypotheses also hold. In more general cases, a parabolic regularization can be used. A more systematic theoretical treatment of this general case, with the conditions based solely on the data, will be presented in a future work.
\end{remark}

Proposition \ref{prop 1} shows that the KL divergence decreases as $\tau$ increases whenever the stochasticity parameter $\gamma(\tau)$ is positive. If $\gamma\equiv0$, which is the case for the Probability Flow ODE, the divergence remains constant for $\tau\in(0,T)$. However, this result alone is not sufficient to quantify the final divergence decrease due to the dependence of the right hand side of \eqref{H deriv 1} on the unknown $\tilde{p}_\tau$. This can be mitigated by the use of a logarithmic Sobolev inequality. 

The following result is stated for forward SDEs with a homogeneous linear drift, in which case the Gaussian estimates on $\bar{p}_\tau$ required by Proposition \ref{prop 1} are automatically satisfied. A homogeneous linear drift is also the standard choice in practice, since in this case the transition probabilities become Gaussian \citep{song2021scorebasedmodeling, karras}, significantly simplifying the training process.

\begin{theorem}\label{teo 2}
Assume that the drift of the forward SDE has the form $f(x,t)=a(t)x$, that the data $p_0 = \bar{p}_T$ is sub-Gaussian, and that $\gamma$ is piecewise-continuous on $(0,T)$. Let $\tilde{p}_\tau$ and $\bar{p}_\tau$ be positive classical solutions, in $C^{2,1}(\mathbb{R}^n \times (0,T))$, of the Fokker-Planck equation associated with the reverse-time SDE \eqref{rev sde 2}, with initial conditions $\tilde{p}_0$ and $\bar{p}_0 = p_T$, respectively. If $\tilde{p}_\tau$ is bounded above and below by Gaussians, locally uniformly in $\tau \in (0,T)$, $\nabla\log\tilde{p}_\tau$ has at most polynomial growth in $x$, and $\nabla\log\bar{p}_\tau$ is Lipschitz for all $\tau\in(0,T)$, we have the decay estimate on the KL divergence
\begin{equation}\label{bound}
    H(\tilde{p}_\tau|\bar{p}_\tau) \leq e^{-\int_0^\tau C(s) \gamma(s) \bar{g}(s)^2 ds} H(\tilde{p}_0|\bar{p}_0),
\end{equation}
for all $\tau\in(0,T)$, where
\begin{equation}
    C(\tau)=\begin{cases}
        C_\text{LSI}(\tau)^{-1}, & \text{if $\bar{p}_\tau$ satisfies an LSI with constant } C_\text{LSI}(\tau)
        \\ 0, & \text{otherwise.}
    \end{cases}
\end{equation}
If $\tilde{p}_T$ is absolutely continuous with respect to $\bar{p}_T$, equation \eqref{bound} also holds for $\tau=T$.
\end{theorem}

\begin{proof}
    In the linear case, $p_t$ can be written as the Gaussian convolution $p_t=p_0^t\ast \mathcal{N}(0, s(t)^2\sigma(t)^2I)$, where $p_0^t(x):=s(t)^{-n}p_0(x/s(t))$ \citep[Appendix B.1]{karras}. Then, the two-sided Gaussian bounds on $\bar{p}$ automatically hold (Appendix \ref{app: proofs sec3}), and from Proposition \ref{prop 1} we have that, for all $\tau\in(0,T)$,
    \begin{equation*}
        \frac{d}{d\tau}H(\tilde{p}_\tau|\bar{p}_\tau) = -\frac{1}{2}\gamma(\tau) \bar{g}(\tau)^2 \int\left|\nabla\log\frac{\tilde{p}_\tau}{\bar{p}_\tau}\right|^2 \tilde{p}_\tau\,dx\leq 0.
    \end{equation*}

    If $\bar{p}_\tau$ satisfies the log-Sobolev inequality \eqref{HLSI} with constant $C_\text{LSI}(\tau)$, we have
    \begin{align*}
        \frac{d}{d\tau}H(\tilde{p}_\tau|\bar{p}_\tau) &= -\frac{1}{2}\gamma(\tau) \bar{g}(\tau)^2 \int\left|\nabla\log\frac{\tilde{p}_\tau}{\bar{p}_\tau}\right|^2 \tilde{p}_\tau\,dx \leq -\frac{1}{C_\text{LSI}(\tau)}\gamma(\tau) \bar{g}(\tau)^2 H(\tilde{p}_\tau|\bar{p}_\tau),
    \end{align*}
    and thus for all $\tau$ we have
    \begin{equation*}
        \frac{d}{d\tau}H(\tilde{p}_\tau|\bar{p}_\tau) \leq -C(\tau)\gamma(\tau) \bar{g}(\tau)^2 H(\tilde{p}_\tau|\bar{p}_\tau),
    \end{equation*}
    where $C(\tau)=1/C_\text{LSI}(\tau)$ if $\bar{p}_\tau$ satisfies an LSI and $C(\tau)=0$ otherwise. Then, by Grönwall's inequality, we obtain \eqref{bound}.
\end{proof}

\begin{remark}
    If $\tilde{p}_\tau$ satisfies an LSI, an analogous result holds for $H(\bar{p}_\tau|\tilde{p}_\tau)$. However, this seems to be more elusive since $\tilde{p}_\tau$ is the solution of a nonlinear PDE, and the assumed two-sided Gaussian bounds are not enough to ensure an LSI. In Appendix \ref{app: counterexample_lsi}, we show a simple oscillating example which violates the LSI scalar equivalence condition from \cite{BOBKOV1999}.
\end{remark}

If $p_0$ is sub-Gaussian, it is known that $p_t$ satisfies an LSI at least for large $t$ \citep[Section 4.2]{dimension-free}. Moreover, when $p_0$ has compact support, such as with image data, the distribution $p_t$ satisfies an LSI for all $t\in(0,T)$, and $\nabla\log\bar{p}_\tau$ is automatically Lipschitz. This is summarized by the following corollary.

\begin{corollary}\label{coro 1}
    If $p_0 = \bar{p}_T$ has compact support with $\supp (p_0) \subset B_R(0)$, then $\bar{p}_{T-t} = p_t$ satisfies the hypotheses in Theorem \ref{teo 2} with the constant
\begin{equation}\label{dim free lsi}
    C_\text{LSI}(t)=6s(t)^2(4R^2+\sigma(t)^2) e^{\frac{4R^2}{\sigma(t)^2}},
\end{equation}
where $s(t):=e^{\int_0^t a(\xi)\,d\xi}$ and $\sigma(t)^2:=\int_0^t g(\xi)^2/s(\xi)^2d\xi$ (in forward time $t$).
\end{corollary}

\begin{proof}
    The constant \eqref{dim free lsi} comes directly from \cite[Corollary 1]{dimension-free}, observing that $p_t=p_0^t\ast \mathcal{N}(0, s(t)^2\sigma(t)^2I)$ with $p_0^t(x):=s(t)^{-n}p_0(x/s(t))$, and that $\supp(p_o^t)\subset B_{s(t)R}(0)$. We then only need to show that $\nabla\log p_t$ is Lipschitz, which is left to Appendix \ref{app: proofs sec3}.
\end{proof}
More general conditions for ensuring that $\nabla\log p_t$ is Lipschitz can be found in \cite{steph2025regularity}.

\begin{remark}\label{remark: lsi}
    Notice that the constant in \eqref{dim free lsi} could depend on the ambient dimension $n$ via the support radius $R$. For instance, for the important application of image generation, with data supported on a hypercube $[-c, c]^n\subset\mathbb{R}^n$, the squared radius $R^2$ scales linearly with $n$. In this case, this LSI constant becomes large and the bound ceases to reveal useful information on KL divergence decrease.
    
    Nevertheless, we observed that in many cases the available LSI constant underestimates the actual divergence decay rate. For instance, Figure \ref{fig: analytical scores} shows two data sets whose LSI constants are significantly different but the decay rate of $H(\tilde p|\bar p)$ is essentially the same. This might reflect the fact that the inequality is being applied to a very specific $\tilde{p}_\tau$, which is closely related to $\bar{p}_\tau$.
\end{remark}

\begin{remark}
    By the Csiszár-Kullback inequality, which gives $\|p_1-p_2\|^2_{\mathcal{L}^1} \leq 2 H(p_1|p_2)$, Theorem \ref{teo 2} gives a corresponding bound on the $L^1$ distance $\|\tilde{p}_T-p_T\|_{\mathcal{L}^1}$.
\end{remark}

\subsubsection{Remark: the functional parameter \texorpdfstring{$\gamma$}{gamma} at discretization}\label{rem: gamma and discretization}

Since the bound \eqref{bound} shows a KL divergence decay which is exponential on the stochasticity parameter $\gamma$, what is the disadvantage of setting $\gamma$ arbitrarily large? 

The answer is related to the numerical approximation. A simple Euler-Maruyama discretization step of the reverse-time SDE \eqref{rev sde 2} is given by
\begin{align}
    X_{\tau_{i+1}} = & X_{\tau_i} +\left(-\bar{f}(X_{\tau_i}, \tau_i) + \frac{1}{2}\bar{g}^2(\tau_i)(1+\gamma(\tau_i))\nabla\log \bar{p}_{\tau_i}(X_{\tau_i})\right)\Delta\tau_i + \sqrt{\gamma(\tau_i)} \bar{g}(\tau_i) \sqrt{\Delta\tau_i}\xi \nonumber
    \\ = & X_{\tau_i} +\left(-\bar{f}(X_{\tau_i}, \tau_i) + \frac{1}{2}\bar{g}^2(\tau_i)\nabla\log \bar{p}_{\tau_i}(X_{\tau_i})\right)\Delta\tau_i \nonumber
    \\ & + \frac{1}{2}\bar{g}^2(\tau_i)\nabla\log \bar{p}_{\tau_i}(X_{\tau_i})(\gamma(\tau_i)\Delta\tau_i) + \bar{g}(\tau_i) \sqrt{\gamma(\tau_i)\Delta\tau_i}\xi \label{discr rev sde}
\end{align}
where $\xi\sim\mathcal{N}(0,1)$, and $\Delta \tau_i:=\tau_{i+1} - \tau_i$. Thus, comparing \eqref{discr rev sde} with \eqref{rev sde 3}, one can see that $\gamma(\tau_i)\Delta\tau_i$ gives an effective discretization step size for the Langevin-like part of the reverse-time SDE, which, together with the standard step $\Delta\tau_i,$ makes the problem stiff for large $\gamma.$ Setting $\gamma$ too large requires a very fine mesh, which will inevitably increase the computational cost and potentially produce a significant accumulation of discretization errors. In Section \ref{sec: experiments}, we show the trade-off between error correction and discretization errors in numerical experiments; see Figure \ref{fig: multiple gamma}.

\subsection{Approximate score functions}\label{subsection: approx scores}

The previous bounds can be generalized to the case where the reverse-time SDE depends on an approximate score $s(x, \tau)$. This models the setting of $s(x, \tau)$ being a neural network trained by score-matching, as in the diffusion model algorithm described in Section \ref{sec: intro}. By writing $s(x, \tau) = \nabla\log \bar{p}_\tau(x) + \epsilon_\tau(x)$, with an associated error function $\epsilon_\tau(x)$ representing the difference between the exact and approximate scores, the approximated version of equation \eqref{rev sde 2} is given by
\begin{equation}\label{rev sde 2 pert}
    d\tilde{X}_\tau = \left(-\bar{f}(\tilde{X}_\tau, \tau) + \frac{1}{2}\bar{g}^2(\tau)(1+\gamma(\tau))\left(\nabla\log \bar{p}_\tau(\tilde{X}_\tau) + \epsilon_\tau(\tilde{X}_\tau)\right)\right)\,d\tau + \sqrt{\gamma(\tau)} \bar{g}(\tau) \,dW_\tau,
\end{equation}
with the corresponding Fokker-Planck equation
\begin{equation}\label{pert rev_fp}
    \frac{\partial \tilde{p}}{\partial \tau} = -\nabla\cdot\left(\left(-\bar{f} + \frac{1}{2}\bar{g}^2(\gamma+1)(\nabla\log \bar{p} + \epsilon)\right)\tilde{p}\right) +\frac{1}{2}\gamma \bar{g}^2\Delta \tilde{p}.
\end{equation}
Then, we can use Lemma \ref{lemma H_deriv} (Appendix \ref{app: proofs sec3}) to obtain straightforward generalizations of Proposition \ref{prop 1} and Theorem \ref{teo 2}.

\begin{proposition}\label{prop 2}
    Let $\bar{p}_\tau$ and $\gamma$ be as before, and let $\tilde{p}_\tau$ be a positive classical solution, in $C^{2,1}(\mathbb{R}^n\times(0,T))$, of the perturbed Fokker-Planck equation \eqref{pert rev_fp} with initial condition $\tilde{p}_0$, where $\epsilon(\cdot, \tau)$ is of class $\mathcal{C}^1$ and, locally uniformly in $\tau$, has at most polynomial growth in $x$. If the pair $(\tilde{p}, \bar{p})$ satisfies the hypotheses of Proposition \ref{prop 1}, then, for almost every $\tau\in(0,T)$,
    \begin{equation}\label{H deriv pert}
        \frac{d}{d\tau}H(\tilde{p}_\tau|\bar{p}_\tau) = \frac{1}{2}\bar{g}^2(1+\gamma)\int\epsilon\cdot \left(\nabla\log \frac{\tilde{p}_\tau}{\bar{p}_\tau}\right)\tilde{p}_\tau\; \,dx -\frac{1}{2}\gamma \bar{g}^2 \int\left|\nabla\log\frac{\tilde{p}_\tau}{\bar{p}_\tau}\right|^2 \tilde{p}_\tau \,dx,
    \end{equation}
    \begin{equation}\label{H deriv pert rev}
        \frac{d}{d\tau}H(\bar{p}_\tau|\tilde{p}_\tau) = \frac{1}{2}\bar{g}^2(1+\gamma)\int\epsilon\cdot \left(\nabla\log \frac{\tilde{p}_\tau}{\bar{p}_\tau}\right)\bar{p}_\tau\; \,dx -\frac{1}{2}\gamma \bar{g}^2 \int\left|\nabla\log\frac{\tilde{p}_\tau}{\bar{p}_\tau}\right|^2 \bar{p}_\tau \,dx.
    \end{equation}
\end{proposition}
 
\begin{proof}
    The Fokker-Planck equations of $\tilde{p}$ and $\bar{p}$ share the diffusion coefficient $b = \sqrt{\gamma}\,\bar{g}$, and,  by the assumption on $\epsilon$, their corresponding drifts $a_1$ and $a_2$ are both continuous, $C^1$ in $x$ and of polynomial growth, inside each interval of continuity of $\gamma(\tau)$, as discussed in the proof of Proposition \ref{prop 1}. 
    Since $a_1 - a_2 = \frac12\bar{g}^2(1+\gamma)\epsilon$, both formulas follow from Lemma \ref{lemma H_deriv} in Appendix \ref{app: proofs sec3} applied to the pairs $(\tilde{p}, \bar{p})$ and $(\bar{p}, \tilde{p})$ in each interval.
\end{proof}

\begin{remark}\label{remark: error direction}
    Observe, in formulas \eqref{H deriv pert} and \eqref{H deriv pert rev}, that the vector $\nabla\log h_\tau(x) = \nabla\log\tilde{p}_\tau(x) - \nabla\log \bar{p}_\tau(x)$ defines a principal direction in which $\epsilon_\tau(x)$ either propagates or corrects accumulated errors. In particular, the exactly aligned direction $\epsilon_\tau(x) = \alpha \nabla\log (\tilde{p}_\tau/\bar{p}_\tau)$, where $\alpha\in\mathbb{R}$, yields the largest increase in the divergences if $\alpha>0$, and the largest decrease if $\alpha<0$. Interestingly, by writing
    \begin{equation}
        E_{\tilde{p}_\tau}\big[\epsilon_\tau\cdot\nabla\log (\tilde{p}_\tau/\bar{p}_\tau)\big] = E_{\tilde{p}_\tau}\big[(s_\theta (\cdot, \tau)-\nabla\log\bar{p}_\tau)\cdot(\nabla\log \tilde{p}_\tau - \nabla\log\bar{p}_\tau)\big],
    \end{equation}
    we can see that the case $\epsilon_\tau(x) = \nabla\log (\tilde{p}_\tau/\bar{p}_\tau)$ (i.e. with $\alpha=1$) corresponds to the learned score $s_\theta (\cdot, \tau)$ matching the implicit score $\nabla\log \tilde{p}_\tau$ of the current sampling distribution. Finally, note that the component of $\epsilon_s(x)$ perpendicular to $\nabla\log h_s(x)$ does not affect the KL divergences.
\end{remark}

\begin{corollary}\label{coro 2}
    If $\gamma(s)>0$ for all $s\in(0,\tau)$, then
    \begin{equation}\label{perturbed bound rev}
        H(\bar{p}_\tau|\tilde{p}_\tau) \leq H(\bar{p}_0|\tilde{p}_0) +\frac{1}{8}\int_0^\tau \bar{g}^2\frac{(1+\gamma)^2}{\gamma}\left(\int|\epsilon|^2 \bar{p}_s \,dx\right)\;ds.
    \end{equation}
\end{corollary}

A corresponding result for the framework of stochastic interpolants can be found in \cite{stoch_interp}, Lemma 2.22.

\begin{proof}
    Using Young's inequality on equation \eqref{H deriv pert rev}, we have that, for every positive function $\eta=\eta(t)$,
    \begin{align*}
        H(\bar{p}_\tau|\tilde{p}_\tau) \leq H(\bar{p}_0|\tilde{p}_0) + \frac{1}{2}\int_0^\tau \bar{g}^2\int\left((1+\gamma)\eta(s)|\epsilon|^2 +\left(\frac{1+\gamma}{4\eta(s)}-\gamma\right)\left|\nabla\log \frac{\tilde{p}_s}{\bar{p}_s}\right|^2 \right) \bar{p}_s \,dx\;ds.
    \end{align*}
    Taking $\eta(s)=(1+\gamma(s))/4\gamma(s)$, we obtain \eqref{perturbed bound rev}.
\end{proof}

Corollary \ref{coro 2} shows that, for stochastic sampling, it is possible to bound the KL divergence $H(\bar{p}_\tau|\tilde{p}_\tau)$ by a multiple of the score error plus the initial error, as has been observed in the literature \citep{song2021maxlikelihood, stoch_interp}. This is, in fact, a consequence of the divergence decay phenomenon for the SDE, which for $H(\tilde{p}|\bar{p})$ can be better quantified with the use of a log-Sobolev inequality, in the following generalization of Theorem \ref{teo 2}.

\begin{theorem}\label{teo 4}
        Under the hypotheses of Proposition \ref{prop 2}, we have that
    \begin{equation}\label{perturbed bound}
        H(\tilde{p}_\tau|\bar{p}_\tau) \leq e^{-\int_0^\tau \alpha(s)\,ds} H(\tilde{p}_0|\bar{p}_0) + \frac{1}{2}\int_0^\tau \bar{g}(s)^2(1+\gamma(s))E_{\tilde{p}_s}\left[\epsilon_s\cdot\nabla\log \frac{\tilde{p}_s}{\bar{p}_s}\right]e^{-\int_s^\tau \alpha(r)\,dr}ds,
    \end{equation}
    with $\alpha(s):=C(s) \gamma(s) \bar g(s)^2$, where $C(s)=1/C_\text{LSI}(s)$ if $\bar{p}_s$ satisfies an LSI with constant $C_\text{LSI}(s)$ and $C(s)=0$ otherwise, as in Theorem \ref{teo 2}.
    
    Moreover, if $\gamma(s)>0$ for all $s\in (0,\tau)$, then, for any $\delta(s)$ satisfying $0<\delta(s)\leq\gamma(s)$ in $(0,\tau)$, we have
    \begin{equation}\label{perturbed bound 2}
        H(\tilde{p}_\tau|\bar{p}_\tau) \leq e^{-\int_0^\tau \alpha(s, \delta)\,ds} H(\tilde{p}_0|\bar{p}_0) + \frac{1}{8}\int_0^\tau\frac{1}{\delta(s)}\bar{g}(s)^2(1+\gamma(s))^2 E_{\tilde{p}_s}\left[|\epsilon_s|^2\right] e^{-\int_s^\tau \alpha(r, \delta)\,dr}ds,
    \end{equation}
    where $\alpha(s,\delta):=C(s)\bar{g}(s)^2 (\gamma(s) - \delta(s))$.
\end{theorem}

\begin{proof}
\begin{enumerate}
    \item As in the proof of Theorem \ref{teo 2}, using the log-Sobolev inequality \eqref{HLSI} on the derivative of $H$, given by \eqref{H deriv pert}, we obtain
    \begin{align*}
        \frac{d}{ds}H(\tilde{p}_s|\bar{p}_s) \leq \frac{1}{2}\bar g(s)^2(1+\gamma(s))\int\epsilon_s\cdot (\nabla\log h_s)\tilde{p}_s\; \,dx -C(s)\gamma(s) \bar g(s)^2 H(\tilde{p}_s|\bar{p}_s),
    \end{align*}
    which, by Grönwall's inequality, gives equation \eqref{perturbed bound}.

    \item For the second claim, we wish to eliminate the unknown term $\nabla\log(\tilde{p}_s/\bar{p}_s)$ from \eqref{H deriv pert}. Using Young's inequality, for any $\delta_0(s)>0$, we have
    \begin{align*}
        \int\epsilon_s\cdot\nabla\log h_s\;\tilde{p}_s\,dx \leq \int|\epsilon_s| |\nabla\log h_s|\;\tilde{p}_s\,dx \leq \int \left(\frac{1}{4\delta_0(t)}|\epsilon_s|^2 + \delta_0(t)|\nabla\log h_s|^2\right)\tilde{p}_s\,dx,
    \end{align*}
    where $\epsilon(\cdot, s)$ is square-integrable in $\tilde{p}_s$, since it has polynomial growth and $\tilde{p}_s$ is bounded above by a Gaussian, by hypothesis. Then, if $\delta_0(s)\leq\gamma(s)/(\gamma(s)+1)$, we have
    \begin{align}
        \frac{d}{ds}H(\tilde{p}_s|\bar{p}_s) & \leq \frac{1} {8\delta_0}\bar{g}^2(1+\gamma)\int|\epsilon|^2\tilde{p}\; \,dx -\frac{1}{2}\bar{g}^2\left(\gamma - (\gamma+1)\delta_0\right) \int\left|\nabla\log\frac{\tilde{p}}{\bar{p}}\right|^2 \tilde{p}\,dx \nonumber
        \\ & \leq \frac{1}{8\delta_0}\bar{g}^2(1+\gamma)\int|\epsilon|^2\tilde{p}\; \,dx -\frac{1}{C_\text{LSI}}\bar{g}^2\left(\gamma - (\gamma+1)\delta_0\right) H(\tilde{p}|\bar{p}) \label{H_deriv pert bound 2}
    \end{align}
    by the LSI. Taking $\delta = (\gamma+1)\delta_0$ and using Grönwall's inequality, we obtain \eqref{perturbed bound 2}.
\end{enumerate}
\end{proof}

\begin{remark}\label{remark: approx bound}
 With regard to this result, we observe the following.
    \begin{enumerate}
        \item In this case, an increase in $\gamma$ is not clearly related to a decrease in the final divergence due to the dependence of $\tilde p_s$ on $\gamma$.

        \item The bound in equation \eqref{perturbed bound rev} is identical to the one in \eqref{perturbed bound 2} with $\delta=\gamma$ (that is, ignoring the decay factor), but with the expectation taken in $\bar{p}_\tau$ instead of $\tilde{p}_\tau$.
        
        \item Although equation \eqref{perturbed bound 2} is obtained using an additional inequality, since the inequality is used before the LSI, the bound there can be smaller than the one in equation \eqref{perturbed bound}, depending on the tightness of both inequalities. This can be seen in Figure \ref{fig:analyticKLbounds}, for the fully analytical example of Section \ref{sec: analytical example}.
    \end{enumerate}
\end{remark}

\begin{remark}\label{remark: opt_gamma}
In \cite{ma2024} and \cite{chen2024follmer}, an optimal choice of $\gamma(\tau)$ minimizing \eqref{perturbed bound rev} is obtained, which, in our formulation, corresponds to setting $\gamma(\tau)\equiv 1$. We can see that this also minimizes the bound \eqref{perturbed bound 2} when discarding the divergence decrease factor, i.e. with $\delta(\tau)=\gamma(\tau)$. In line with the comments in \cite[Appendix B.4]{opt_choice}, we observe that the application of Young's inequality in the proof is designed to eliminate the Fisher information term, implicitly penalizing both small and too large $\gamma$ values. Hence this particular bound erases part of the information contained in the time derivative of the KL divergence, which we analyze in the next section.
\end{remark}

\subsubsection{Minimizing the KL divergence derivative at each \texorpdfstring{$\tau$}{tau}}\label{subsec: instantaneous_opt}
Since the formula for the KL divergence time derivative obtained in Proposition \ref{prop 2} is linear in $\gamma$, we can readily obtain the values of $\tau$ for which a positive $\gamma$ can produce a decrease in the divergence. Observe that this instantaneous minimum may not correspond exactly to the perfect-foresight minimum, since, for reverse-times $\tau_1<\tau_2$, the density $\tilde{p}_{\tau_2}$ also depends on $\gamma(\tau_1)$. However, the problem of finding the perfect-foresight minimum is infinite-dimensional and typically intractable, except in very specific settings, such as the analytical example considered in Section \ref{sec: analytical example}. In this section, we explore the information provided by the instantaneous optimum, while keeping in mind its limitations.

More precisely, writing equation \eqref{H deriv pert} as
\begin{equation}
    \frac{d}{d\tau}H(\tilde{p}_\tau|\bar{p}_\tau) = \frac{1}{2}\bar{g}^2\int\epsilon\cdot \left(\nabla\log \frac{\tilde{p}_\tau}{\bar{p}_\tau}\right)\tilde{p}_\tau\; \,dx +\frac{1}{2}\gamma \bar{g}^2 \int\left( \epsilon\cdot \nabla\log \frac{\tilde{p}_\tau}{\bar{p}_\tau} - \left|\nabla\log\frac{\tilde{p}_\tau}{\bar{p}_\tau}\right|^2\right) \tilde{p}_\tau \,dx,
\end{equation}
we can see that $dH/d\tau$ is minimized at each $\tau$ by the bang-bang optimum
\begin{equation}\label{optimal_g_instantaneous}
    \gamma^*(\tau) = \begin{cases}
        \gmax, & \text{when}\quad E_{\tilde{p}_\tau}[\epsilon\cdot \nabla\log (\tilde{p}_\tau/\bar{p}_\tau) - |\nabla\log (\tilde{p}_\tau/\bar{p}_\tau)|^2]<0
        \\ 0, & \text{when}\quad E_{\tilde{p}_\tau}[\epsilon\cdot \nabla\log (\tilde{p}_\tau/\bar{p}_\tau) - |\nabla\log (\tilde{p}_\tau/\bar{p}_\tau)|^2]>0,
    \end{cases}
\end{equation}
where $\gmax$ is the stipulated maximum $\gamma$ value, limited by the discretization issues discussed in \ref{rem: gamma and discretization}. We can identify two main factors at play:
\begin{itemize}
    \item \textbf{Alignment of $\epsilon_\tau$ and $\nabla\log (\tilde{p}_\tau/\bar{p}_\tau)$:} As observed in Remark \ref{remark: error direction}, the alignment between $\epsilon_\tau(x)$ and $\nabla\log (\tilde{p}_\tau/\bar{p}_\tau)(x)$, in terms of the inner product $E_{\tilde{p}_\tau}[\epsilon_\tau\cdot\nabla\log (\tilde{p}_\tau/\bar{p}_\tau)]$, is a major factor in the accumulation of errors. This is also true here, since a low enough alignment, even with large $E[|\epsilon_\tau|^2]$, can produce a negative expectation in \eqref{optimal_g_instantaneous}, favoring the SDE. However, it is not clear how to evaluate $\epsilon_\tau\cdot \nabla\log (\tilde{p}_\tau/\bar{p}_\tau)$ in practice, since it depends on the unknown $\nabla\log \tilde{p}_\tau$.
    \item \textbf{Relative magnitudes of $\epsilon_\tau$ and $\nabla\log (\tilde{p}_\tau/\bar{p}_\tau)$:} If $|\epsilon_\tau|(x)$ is smaller on average than $|\nabla\log (\tilde{p}_\tau/\bar{p}_\tau)|(x)$, the sign of the expectation in \eqref{optimal_g_instantaneous} must be negative. More precisely, we have that
    \[
    \int\left( \epsilon_\tau\cdot \nabla\log \frac{\tilde{p}_\tau}{\bar{p}_\tau} - \left|\nabla\log\frac{\tilde{p}_\tau}{\bar{p}_\tau}\right|^2\right) \tilde{p}_\tau \,dx \leq \int\left( |\epsilon_\tau| - \left|\nabla\log\frac{\tilde{p}_\tau}{\bar{p}_\tau}\right|\right)\left|\nabla\log\frac{\tilde{p}_\tau}{\bar{p}_\tau}\right| \tilde{p}_\tau \,dx <0
    \]
    if $\int|\epsilon_\tau|\tilde \mu_\tau dx<\int|\nabla\log \tilde{p}_\tau/\bar{p}_\tau|\tilde \mu_\tau dx$, where $\tilde \mu_\tau := \tilde{p}_\tau|\nabla\log \tilde{p}_\tau/\bar{p}_\tau|$ \footnote{An analogous condition can be obtained for $H(\bar p_\tau|\tilde p_\tau)$, taking $\tilde \mu_\tau := \bar{p}_\tau(x)|\nabla\log \tilde{p}_\tau(x)/\bar{p}_\tau(x)|$.}. 
    Since $\nabla\log \tilde{p}_\tau/\bar{p}_\tau$ is the difference between the implicit score of $\tilde p_\tau$ and the true score of $\bar{p}_\tau$, it can be thought as a measure of the accumulated error. In short, \textit{if, at time $\tau$, the average magnitude of the score error is lower than a threshold given by the accumulated error, stochasticity is beneficial}.
\end{itemize}

This last point shows that the time-localization of the magnitude of the score error can be a crucial factor for the effect of stochasticity, extending the conclusion obtained in \cite{opt_choice} to the non-asymptotic case, i.e. for intermediate $\tau$ and $\gamma$ values. In particular, score errors concentrated in the beginning of sampling would make stochasticity beneficial, whereas score errors with magnitude increasing too fast near the end of sampling could make it detrimental. These aspects are investigated in the experiments of Section \ref{subsec: time_profile_experiments}.

\section{Numerical experiments}\label{sec: experiments}In this section, we explore the relation between stochasticity at sampling and model performance in several experimental scenarios. We consider two toy data sets for which the exact score functions can be computed and more complex data sets such as MNIST, CIFAR-10, and a three-dimensional porous media data set \citep{frontgeologico}. Sampling with exact score functions, on the toy data sets, is explored in Section \ref{subsection: experiments analytical scores}. Then, Section \ref{subsection: experiments approx scores} explores sampling with learned score models on both toy and complex data sets, covering the setting of Section \ref{subsection: approx scores}.

The experiments were performed using the diffusion model package 'diffsci', with the code available at \url{https://github.com/Lacadame/DiffSci/tree/main/stochasticity_paper}.

\subsection{Experimental details}

\subsubsection{Sampling parameters}

We performed experiments using the three forward processes considered in \cite{karras}: EDM, Variance Exploding (VE) and Variance Preserving (VP), which are given by the forward SDEs
\begin{align}
    dX_t & = \sqrt{2t}\,dW_t \tag{EDM}, \\
    dX_t & = \,dW_t, \tag{VE}  \\ 
    dX_t & = -\frac{1}{2}(\beta_1 t+\beta_2)X_t \,dt + \sqrt{\beta_1 t+\beta_2}\,dW_t \tag{VP}.
\end{align}
Note that the three equations satisfy the linear hypothesis in Theorems \ref{teo 2} and \ref{teo 4}. To avoid confusion, in this section we will refer only to the forward time variable $t$, and consider the densities $p_t$ and $\tilde{p}_t$ in forward time. In this way, diffusion sampling starts at an "initial time" $t=T$ and ends at $t=0$.

In order to maintain discretization errors independent of the choice of $\gamma$, we perform numerical integration (at sampling) with Euler and Euler-Maruyama methods for ODEs and SDEs, respectively. Since the main focus of this work is not discretization, as discussed in the Introduction, by default we use $1000$ discretization steps for CIFAR-10 and $500$ for the remaining data sets, following the EDM choice of steps $\{t_i\}$ in \cite[Table 1]{karras}. In these settings, we observed qualitatively similar results from the three forward processes. This agrees with findings that major differences between the processes arise only at low $n_\text{steps}$, with EDM yielding the lowest discretization error  \cite[Figure 2]{karras}. Therefore, we only show experiments with the EDM process.

\subsubsection{Datasets}

The toy data distributions consist of a mixture of Gaussians or a mixture of uniform distributions, for which exact score functions can be analytically obtained. This allows sampling with exact scores and also computing the score error function $\epsilon_t(x) = s(x, t) - \nabla\log p_t(x)$ considered in Section \ref{subsection: approx scores}. To maximize the estimation accuracy of the KL divergences, we focus on a one-dimensional setting.
More explicitly, we consider the toy data distributions $p_0$ and $q_0$ in $\mathbb{R}$, consisting of mixtures
\begin{align}
 p_0 & \sim \lambda \mathcal{N}(\mu_1, \sigma_1^2) + (1 - \lambda)\mathcal{N}(\mu_2, \sigma_2^2), \label{p0gaussianmixture}
\\ q_0 & \sim \lambda \mathcal{U}(a_1, b_1) + (1 - \lambda)\mathcal{U}(a_2, b_2), \label{q0uniformmixture}
\end{align}
where $\mathcal{N}(\mu, \sigma^2)$ is a Gaussian distribution of mean $\mu$ and variance $\sigma^2$, and $\mathcal{U}(a, b)$ is a Uniform distribution in the interval $[a, b]\subset\mathbb{R}$. For linear forward processes, the score functions $\nabla\log p_t$ can be analytically computed, and are Lipschitz continuous (see Appendix \ref{app: proofs sec4} for the mixture of Gaussians, and Corollary \ref{coro 1} for the mixture of Uniforms), satisfying the conditions on $\bar p_\tau$ in Theorems \ref{teo 2} and \ref{teo 4}; the conditions on $\tilde p_\tau$ are discussed in Remark \ref{remark: hypotheses scope}. For training score models, we consider $4000$ samples of the toy distributions. At sampling, we consider $10^5$ generated samples.

We also consider the more complex data sets MNIST \citep{mnist_lecun1998}, CIFAR-10 \citep{cifar-10_krizhevsky2009}, and a geological data set consisting of 3D volumes of micro-CT scans of porous media \citep{frontgeologico}. Due to the high dimensionality of the data, it is not possible to reliably estimate KL divergences. Instead, we evaluate performance via FID scores, for the benchmark data sets, and KL divergences of geological statistics, for the porous media data set.

\subsubsection{Estimation of KL divergences and bounds}

KL divergences are computed in two ways. The initial divergence $H(q_T|\rho_T)$, where $q$ and $\rho$ correspond either to $\tilde p$ or $\bar p$, is obtained by numerical integration of the densities, which can be calculated analytically for the toy data sets. Intermediate divergences $H(q_t|\rho_t)$ are computed from samples of $q_t$ and $\rho_t$ via the discretization of the empirical densities by a histogram\footnote{This estimator has limited precision, but it was found to be more reliable than the alternatives, such as using a KDE estimator for the densities or directly using a non-parametric estimator for the divergences.}. To reduce the dependence on the parameter number of histogram bins, we consider the mean of the values obtained with bins in the range $\{n_\text{bins}-20, ..., n_\text{bins}\}$. We found this estimator to have a precision on the order of $5\times 10^{-4}$.

The bounds of Section \ref{section: kl evolution and bounds} are computed with a different LSI constant for each toy data set. For the mixture of Uniforms, we use the constant for compactly-supported measures stated in Corollary \ref{coro 1}, which are from \cite{dimension-free}. For the mixture of Gaussians we use the constants obtained in \cite{LSImixtures}, detailed in the following.
The $\chi^2$-divergence of one-dimensional Gaussian distributions $\mu_0 \sim \mathcal{N}(m_0,\sigma_0^2)$ and $\mu_1 \sim \mathcal{N}(m_1,\sigma_1^2)$ has the closed form
\begin{equation}
    \chi^2(\mu_0|\mu_1)=\dfrac{\sigma_1^2}{\sigma_0\sqrt{2\sigma_1^2-\sigma_0^2}}\exp\left(\dfrac{(m_0-m_1)^2}{2\sigma_1^2-\sigma_0^2}\right)-1
\end{equation}
provided that $2\sigma_1^2>\sigma_0^2$. Then, from \cite[Corollary 2]{LSImixtures} we know that the LSI constant $C_p$ for the mixture of Gaussians $\mu_p = p \mu_0 + (1-p) \mu_1
$ satisfies
\begin{equation}\label{GaussianLSB}
    C_p\leq\min\{C_0, C_1\},
\end{equation}
with
\begin{align*}
    C_0 & = \max\{\sigma_0^2 (1 + (1-p)\lambda_p), \; \sigma_1^2 (1+p\lambda_p \chi_0)\},
    \\ C_1 & = \max\{\sigma_1^2 (1 + p\lambda_p), \; \sigma_0^2 (1+(1-p)\lambda_p \chi_1)\},
\end{align*}
where
\begin{align*}
    \lambda_p := \left\{
    \begin{array}{cl}
    2 &, \ p=\frac{1}{2}, \\
    \frac{\log p - \log(1-p)}{ 2p - 1} &,\ p\neq \frac{1}{2},
    \end{array}
    \right.
\end{align*}
\begin{equation*}
    \chi_0 := \chi^2(\mu_0||\mu_1) + 1 \qquad\text{and}\qquad  \chi_1 := \chi^2(\mu_1||\mu_0) + 1.
\end{equation*}
A plot of LSI constants across diffusion time $t$ can be seen in Figure \ref{fig: analytical scores}.

\subsection{Exact score functions}\label{subsection: experiments analytical scores}

First, we generate distributions with exact analytical scores, in the setting of Section \ref{subsection: exact scores}. Figure \ref{fig: analytical scores} shows generated distributions and KL divergence evolution, for sampling from a perturbed prior distribution at time $T=80$. We artificially perturb the mean and the variance of the prior distribution so that the decrease in the divergences is more perceptible and better quantifiable.

% Analytical scores
\begin{figure}[ht!]
    \centering
    \includegraphics[width=0.32\linewidth]{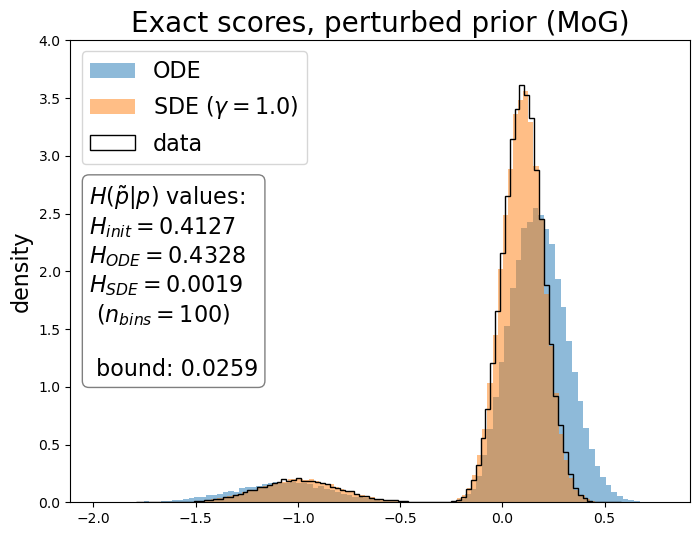}
    \includegraphics[width=0.32\linewidth]{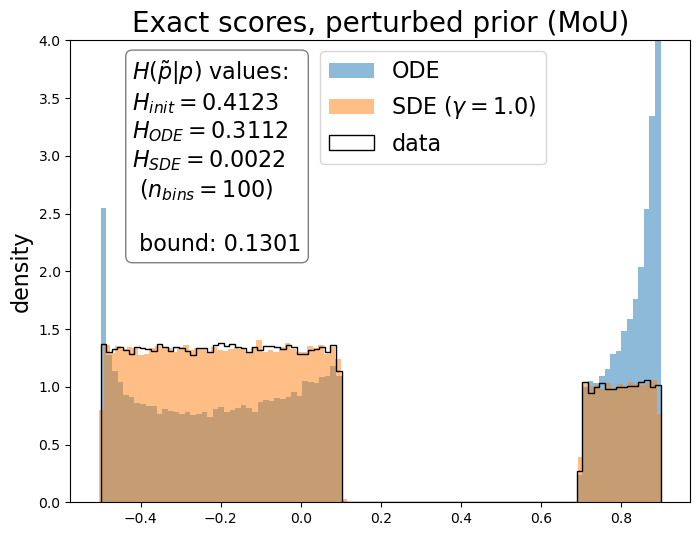} 
    \includegraphics[width=0.32\linewidth]{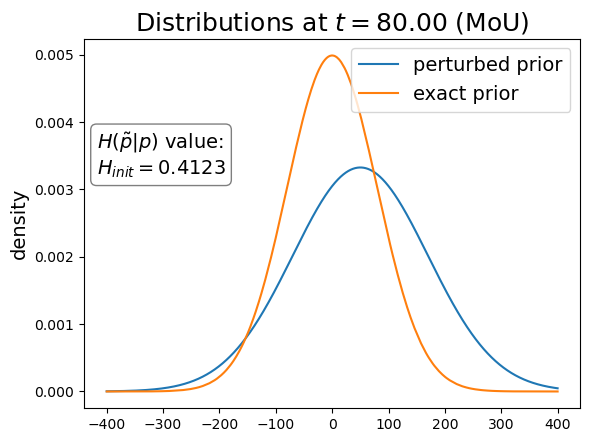}
    \hfill
    \includegraphics[width=0.32\linewidth]{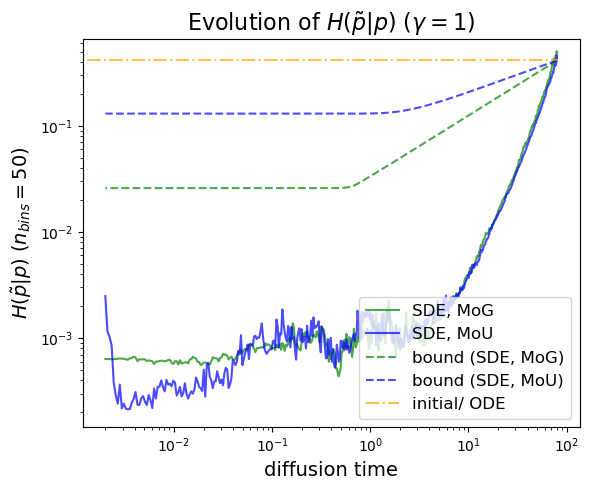}
    \includegraphics[width=0.32\linewidth]{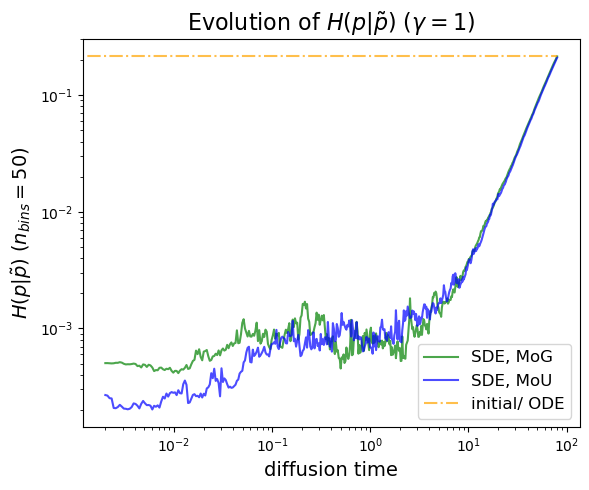}
    \includegraphics[width=0.32\linewidth]{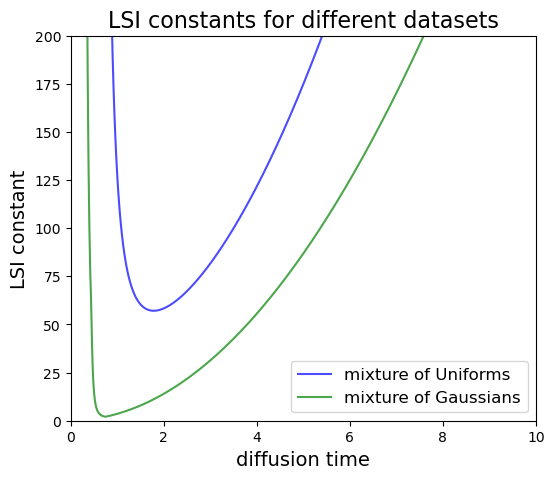}
    \hfill
    \caption{On the top left/middle, data and generated distributions with exact score functions and perturbed priors $\mathcal{N}(50, 120^2I)$ instead of the default $\mathcal{N}(0, 80^2I)$. Initial distributions for the mixture of Uniforms (top right), KL divergence evolution curves (bottom left/ middle) and LSI constants for $p_t$ with different data sets (bottom right).}
    \label{fig: analytical scores}
\end{figure}

We can see a substantial decrease in the divergences for the SDE, higher than that ensured by the bound for $H(\tilde p| p)$, and very similar across both divergences. Moreover, the bound for the mixture of Gaussians (MoG) is tighter than the one for the mixture of Uniforms (MoU), due to the different known LSI constants, whereas the decrease in the divergences is remarkably similar. This might reflect the fact that the LSI constant of $p_t$ is valid for \textit{any} $q_t$, but in the time-derivative formula it is applied to a very specific $q_t=\tilde{p}_t$, which is close to $p_t$.
Observe that the small mismatch between initial and ODE divergence values is due to estimation errors for the latter, since by Proposition \ref{prop 1} they should be the same. Observe also that the final jump in $H(\tilde{p}|p)$ for the MoU can be attributed to the sensitivity of the KL divergence estimator for $t\approx0$, since $\tilde{p}_0$ may not be absolutely continuous with respect to $p_0$ for this data set.

Figure \ref{fig: multiple gamma} illustrates the relation between $\gamma$ and the number of discretization steps, discussed in Section \ref{rem: gamma and discretization}. Both panels show the same final KL divergences with different horizontal axes: on the left, we see an optimal interval for $\gamma$ for each choice of $n_\text{steps}$, on the right, we see that the increase in errors for large $\gamma$ is a function of the "effective step size" $\gamma/n_\text{steps}\propto\gamma\Delta t$ discussed in \ref{rem: gamma and discretization}. More precisely, equation \eqref{discr rev sde} gives two effective step sizes: $\Delta \tau_i$ for the Flow ODE part and $\gamma(\tau_i)\Delta \tau_i$ for the Langevin-like part, but for large $\gamma$ the errors from the latter dominate. In the plot on the right, we see that for smaller $n_\text{steps}$ the Flow ODE effective step size still plays a significant role, with the curves not following exactly the Langevin effective step size scaling.

\begin{figure}[ht!]
    \centering
    \includegraphics[width=0.85\linewidth]{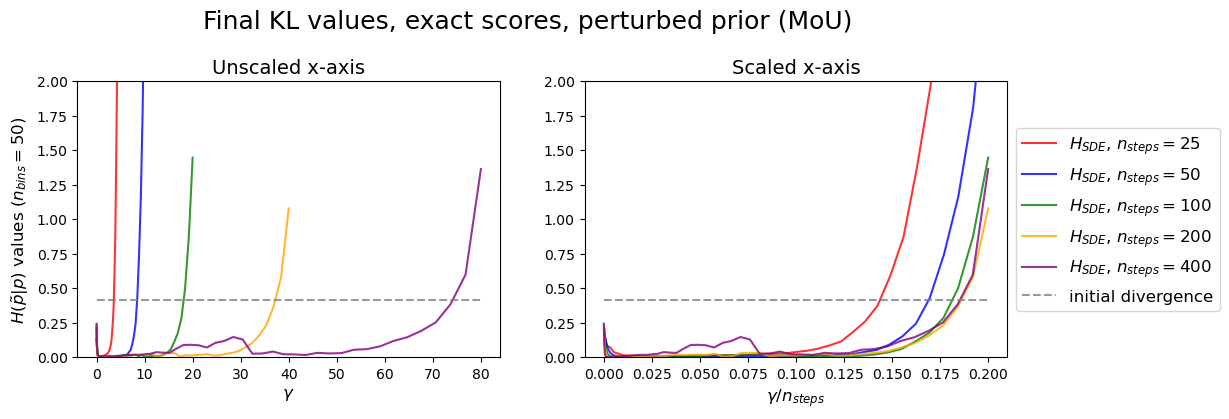}
    \caption{The effect of constant $\gamma$ and discretization step size on the final divergence $H(\tilde{p}|p)$ for the mixture of Uniforms. The horizontal axis on the right is the Langevin 'effective step size' described in \ref{rem: gamma and discretization}.}
    \label{fig: multiple gamma}
\end{figure}

\subsection{Learned score functions}\label{subsection: experiments approx scores}

Here, we consider sampling with learned score functions, in the setting of Section \ref{subsection: approx scores}, for the mixture of Gaussians and the more complex data sets. For the mixture of Gaussians, we train a multilayer perceptron (MLP) of layers $[128, 128, 128]$ with SiLU activations and sinusoidal time-embeddings, with the EDM training configurations \citep{karras}. For MNIST, we train a UNet with 45M parameters, while for CIFAR-10 we use pretrained models from \cite{karras}, and for the porous media example we use pretrained models from \cite{frontgeologico}.

The mixture of Gaussians allows a direct computation of the score error $\epsilon_t(x)$, since the exact score $\nabla\log p_t(x)$ is also available. For the more complex data sets, the true score function is unknown, but we can still compare the score matching (SM) error $E_{\bar{p}_t}[|\epsilon_t|^2]$ between different models. This can be done via the denoising score matching (DSM) errors, since SM and DSM errors differ only by a model-independent constant \citep{denoisingSM}.

\subsubsection{Prior error correction}

First, we evaluate the ability of reverse SDEs with approximate scores to correct errors from the prior distribution. As in Section \ref{subsection: experiments analytical scores}, we consider perturbed prior distributions at time $T=80$, here perturbing only its standard deviation from the default choice $\sigma(T)=T$ matching the forward process. Figure \ref{fig: gamma curves perturbed prior} shows a performance comparison across constant values of $\gamma$, for a model trained on the mixture of Gaussians and a model trained on MNIST. We can see that above a certain $\gamma$ the SDE essentially eliminates the error from the prior, with final performance independent of the prior choice. Figure \ref{fig: entropy evolution perturbed prior} shows, for the mixture of Gaussians, the evolution of the KL divergences and of the bounds from Section \ref{section: kl evolution and bounds}, in the case of $\sigma(T)=96$ and $\gamma=1$. We see an initial correction of the prior distribution error by the SDE, followed by the accumulation of score errors near the end of the sampling process, which is observed for the ODE as well. Observe that the bound for $H(\tilde{p}|p)$, from Theorem \ref{teo 4} (equation \eqref{perturbed bound 2}) tracks the initial decrease in the divergences, unlike the bound for $H(p|\tilde{p})$, from Corollary \ref{coro 2}. The first bound is computed with an instantaneous optimal $\delta=\delta(\tau)$ obtained by minimizing the bound on $dH/d\tau$ \eqref{H_deriv pert bound 2} at each time step $\tau_i$.

% Approximate prior and learned scores
\begin{figure}[ht!]
    \begin{subfigure}[b]{0.64\linewidth}
        \centering
        \includegraphics[width=0.48\linewidth]{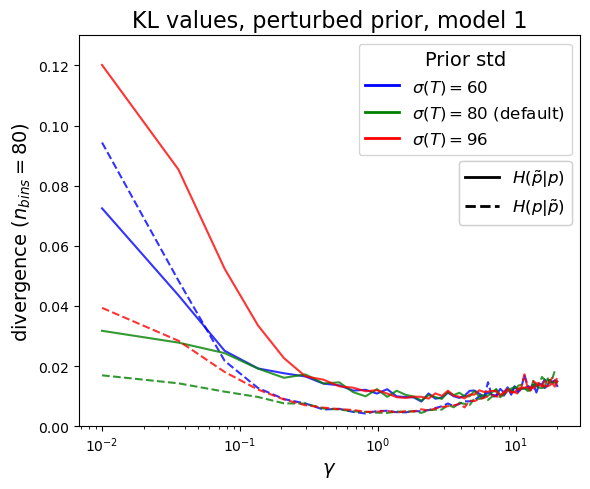}
        \includegraphics[width=0.48\linewidth]{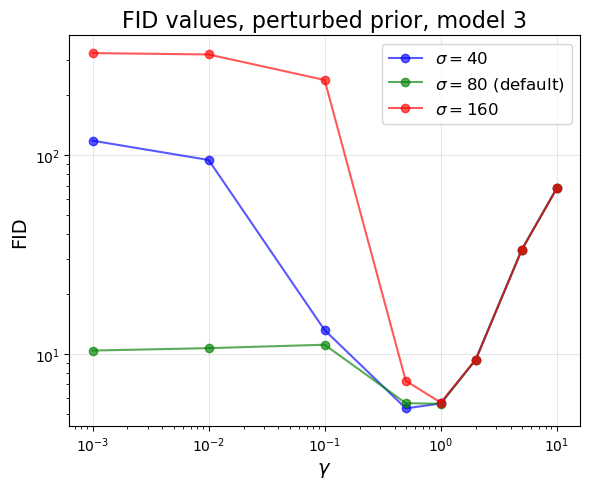}
        \caption{On the left, $\text{KL}\times\gamma$ curves for the mixture of Gaussians, using prior distributions with a perturbed variance. On the right, the equivalent plot for an MNIST model, showing $\text{FID}\times\gamma$ curves.}
        \label{fig: gamma curves perturbed prior}
    \end{subfigure}    
    \hfill
    \begin{subfigure}[b]{0.32\linewidth}
        \centering
        \includegraphics[width=\linewidth]{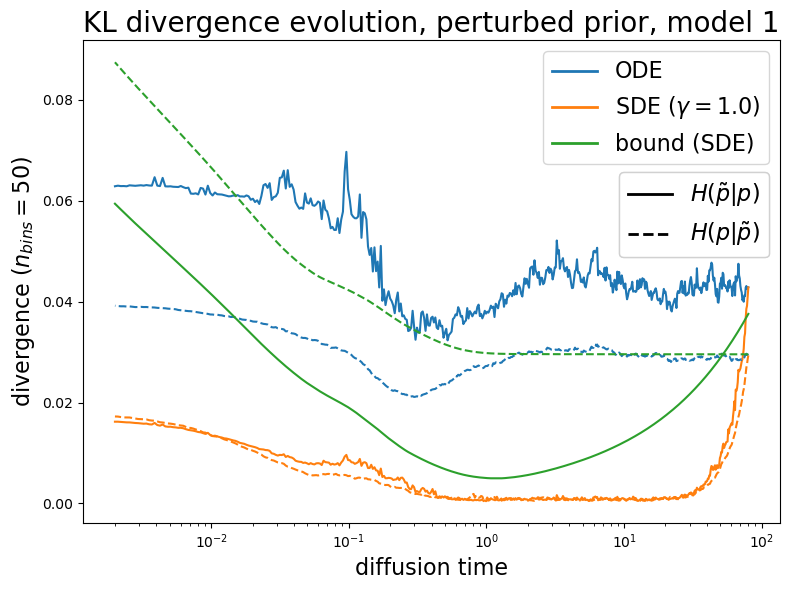}
        \caption{Evolution of KL divergences and bounds with a prior std $\sigma(T)=96$.}
        \label{fig: entropy evolution perturbed prior}
    \end{subfigure}
    \caption{Experiments with learned scores and a perturbed prior.}
    \label{fig: learned scores}
\end{figure}

\subsubsection{Time-profile of error magnitude}\label{subsec: time_profile_experiments}

We now move to a more realistic setting, where the initial error is very small, and focus on score approximation errors. Motivated by the analysis in Section \ref{subsec: instantaneous_opt}, we test how the time-profile of error magnitude interacts with the effect of $\gamma$ on final performance. Figure \ref{fig: gamma curves and error norms} shows three illustrative examples of score models with the same network architecture, trained with the same hyperparameters and having \textit{approximately the same validation loss values}. Comparing the error profiles of model 1 and 2, we see that model 1 has a higher $L^2$ norm of the error for $t>4$ and lower in the intermediate interval $[0.01, 0.5]$. Based on the reasoning from Section \ref{subsec: instantaneous_opt}, we would expect stochastic correction to be more effective for model 1, and this is indeed observed in the $\text{KL}\times\gamma$ plot. For model 3, the error near the end of sampling ($t<0.1$) is smaller than for the others, so from Section \ref{subsec: instantaneous_opt} we would expect it to produce a smaller amplification of errors for large $\gamma$, and this is also observed in the $\text{KL}\times\gamma$ plot. As an illustration of the well-known fact that deterministic sampling is more robust to the reduction of the number of steps, Figure \ref{fig: multisteps} shows the impact of $\gamma$ across several $n_\text{steps}$ values. This happens even for model 3, which benefits from stochasticity at large $n_\text{steps}$, and should be more pronounced when using a fast deterministic sampler instead of Euler's method.

% Gamma curves and error norms
\begin{figure}[ht!]
    \begin{subfigure}[b]{\linewidth}
        \centering
        \includegraphics[width=0.48\linewidth]{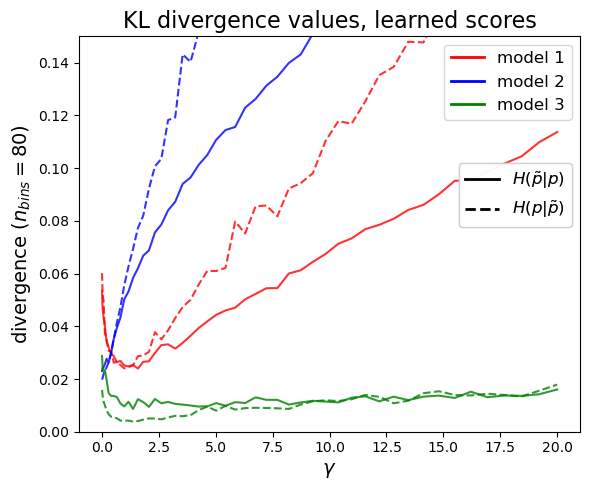}
        \hfill \includegraphics[width=0.48\linewidth]{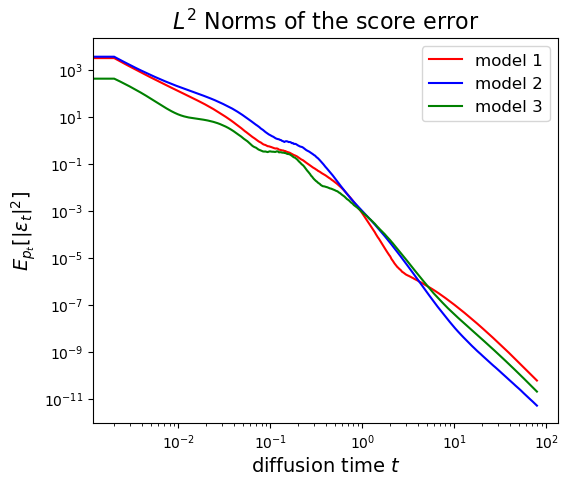}
        \caption{For three illustrative trained models, $\text{KL}\times\gamma$ curves for both divergences (left), and average squared score errors $E_{p_t}[|\epsilon_t|^2]$ (right).}
        \label{fig: gamma curves and error norms}
    \end{subfigure}
    \begin{subfigure}[b]{0.4\linewidth}
        \centering
        \includegraphics[width=\linewidth]{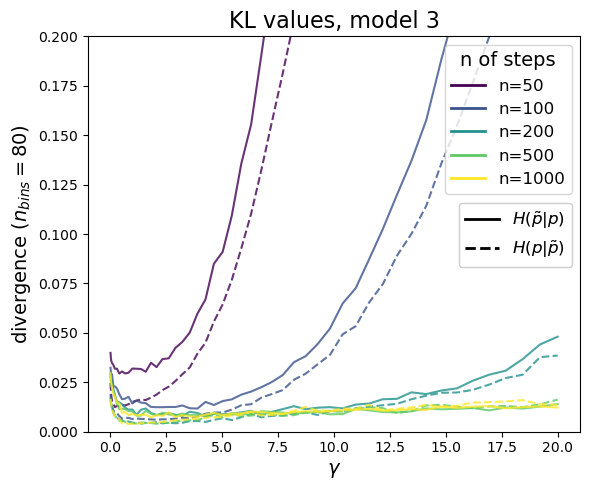}
        \caption{The effect of the number of discretization steps on $\text{KL}\times\gamma$ curves.}
        \label{fig: multisteps}
    \end{subfigure}
    \hfill
    \begin{subfigure}[b]{0.55\linewidth}
        \centering
        \includegraphics[width=\linewidth]{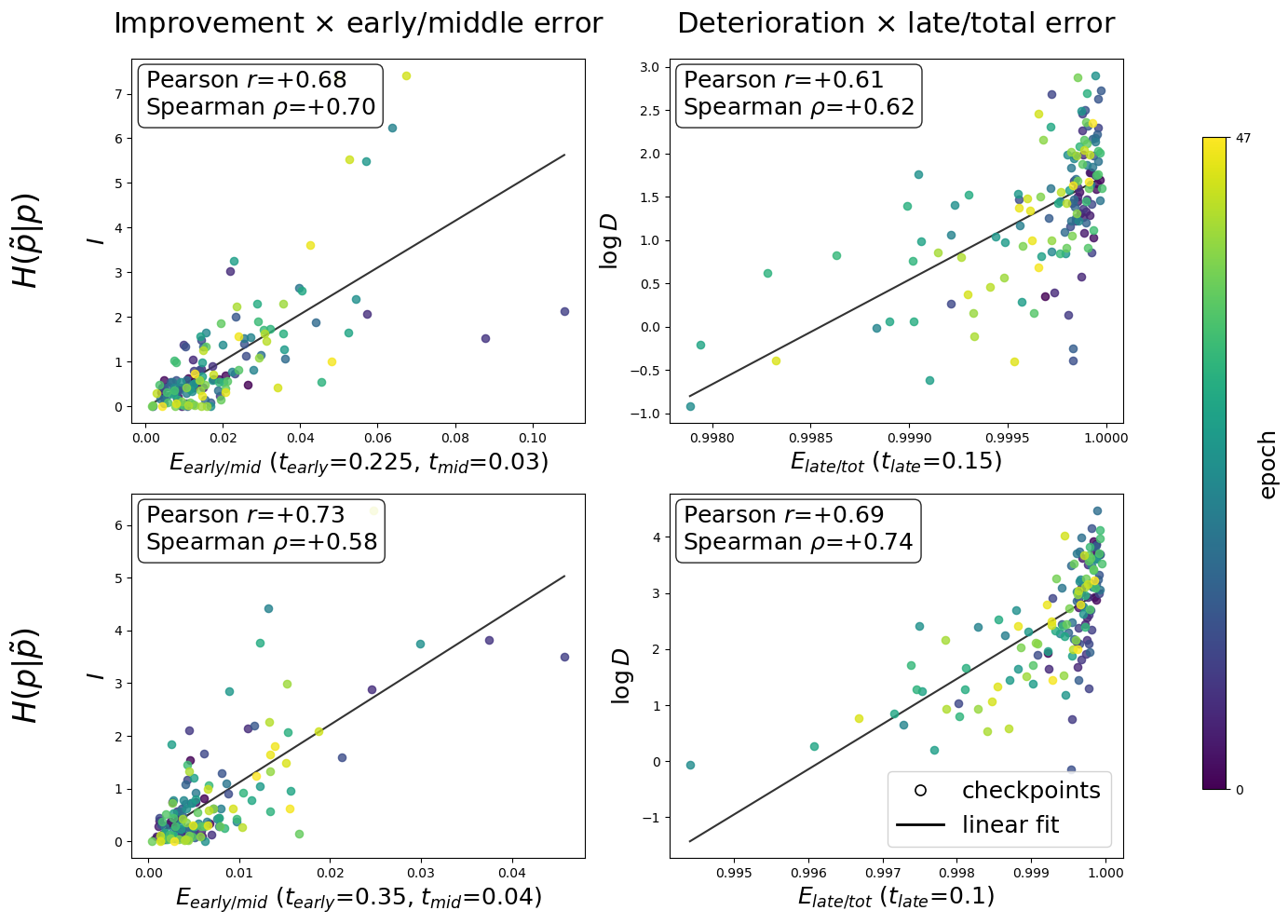}
        \caption{Correlation plots between scalar quantities across $144$ checkpoints.}
        \label{fig:correlation}
    \end{subfigure}
    \caption{Learned models on the mixture of Gaussians data set.}
    \label{fig: MoG learned scores}
\end{figure}

Complementing these examples, we investigate the relationships outlined in the previous paragraph across all the epoch checkpoints for three training runs of the same model.
We extract, from the $\text{KL}\times\gamma$ and error profile curves, scalars capable of quantifying specific aspects of these curves, and measure correlations. For the influence of stochasticity, we consider the relative improvement $I:=(H_0 - H_{\min})/ H_{\min}$ and the relative deterioration $D:=(H_{\gmax} - H_{\min})/ H_{\min}$, where $H_0$ denotes a final KL divergence at $\gamma=0$, $H_{\gmax}$ a final KL divergence at $\gamma=\gmax$, and $H_{\min}$ the lowest across all $\gamma\in[0, \gmax]$. For the error profiles, we consider the ratios of integrated errors
\begin{equation}
    E_\text{early/mid}:=\frac{\int_{t_\text{early}}^T E_{p_t}[|\epsilon_t|^2]dt}{\int_{t_\text{mid}}^{t_\text{early}}E_{p_t}[|\epsilon_t|^2]dt} \qquad \text{and}\qquad E_\text{late/tot}:=\frac{\int_0^{t_\text{late}} E_{p_t}[|\epsilon_t|^2]dt}{\int_0^T E_{p_t}[|\epsilon_t|^2]dt},
\end{equation}
with threshold parameters $t_\text{early}$, $t_\text{mid}$ and $t_\text{late}$. Figure \ref{fig:correlation} shows scatter plots of $I \times E_\text{early/mid}$ and $(\log D) \times E_\text{late/tot}$, for both KL divergences, along with values of Pearson and Spearman coefficients and the obtained regression lines. We see a moderate positive correlation, where the remaining unexplained variance could be associated with more complex interactions of the time-profile of the error. The values of $t_\text{early}$, $t_\text{mid}$ and $t_\text{late}$ were chosen to maximize the sum of Pearson and Spearman correlation coefficients on a different set of checkpoints. A plot of the sensitivity of the correlations on the threshold values, for both the main and the held-out checkpoints, can be seen in Figure \ref{fig: correlation thresholds}.
The correlation values are stable along small perturbations of the thresholds, but the most significant choices still may depend on the training hyperparameters and the data set. Thus, we regard this experiment more as a quantitative proof-of-concept than as providing general scalar metrics for any trained model. A more systematic exploration is left for future work.

Next, we consider examples of models on benchmark data sets. Since the DSM losses only give relative information on the error, we plot the relative error differences to the best model, which is taken as baseline. For the MNIST data set (Figure \ref{fig:mnist_fid_gamma_dsm}), we see the same pattern as before: model 1 has a larger error for $t\in[10, 40]$, and shows performance improvement by positive $\gamma$. For model 2, with a smaller error for large $t$, this effect is not observed, while for model 3, with a smaller error for small $t$, the improvement reappears. The same pattern is found for the CIFAR-10 data set with pretrained models differing only in their training preconditioning \citep{karras} (Figure \ref{fig:cifar10_gamma_dsm}), where the larger initial ($t>1$) error of the VP model is associated with an improvement for positive $\gamma$.

\begin{figure}[ht!]
    \centering
    % First Subfigure: MNIST
    \begin{subfigure}[b]{0.48\linewidth}
        \centering
        \includegraphics[width=0.48\linewidth]{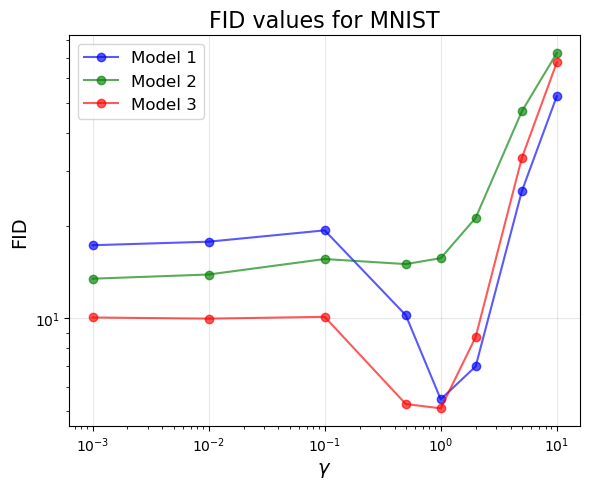}\hfill
        \includegraphics[width=0.48\linewidth]{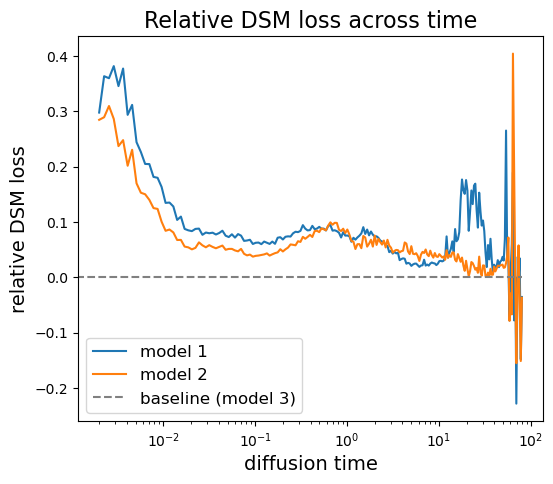}
        \caption{Comparison of three models trained on the MNIST data set, sampling with 500 steps.}
        \label{fig:mnist_fid_gamma_dsm}
    \end{subfigure}
    \hfill
     % Second Subfigure: CIFAR-10
    \begin{subfigure}[b]{0.48\linewidth}
        \centering
        \includegraphics[width=0.48\linewidth]{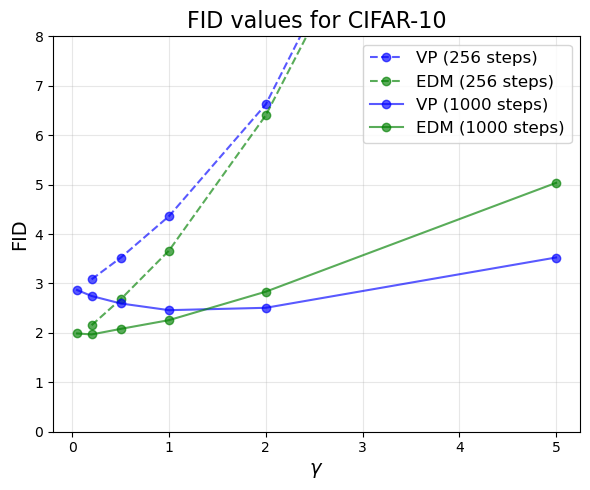}\hfill
        \includegraphics[width=0.48\linewidth]{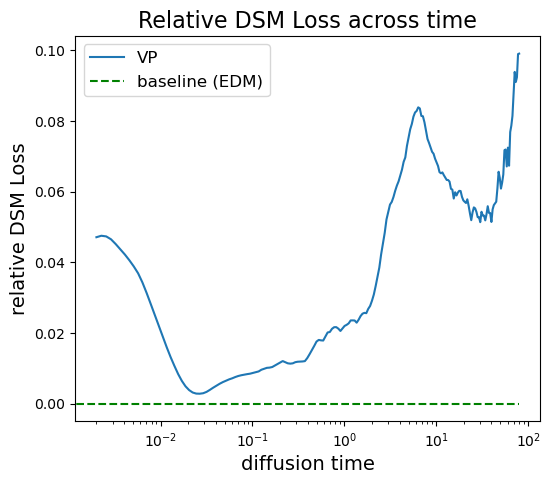}
        \caption{Comparison of pretrained models for the CIFAR-10 data set.}
        \label{fig:cifar10_gamma_dsm}
    \end{subfigure}
    
    \caption{$\text{FID}\times\gamma$ curves and relative DSM errors across time, for models trained on benchmark data sets. The relative DSM error is given by $E_\text{rel}:=\big(E_\text{DSM}(\text{model})-E_\text{DSM}(\text{best model})\big)/ E_\text{DSM}(\text{best model})$.}
    \label{fig:gamma_curves_error_norms_benchmarks}
\end{figure}

% Gamma curves and error norms for rocks
\begin{figure}[ht!]
    \begin{subfigure}[b]{0.67\linewidth}
        \centering
        \includegraphics[width=\linewidth]{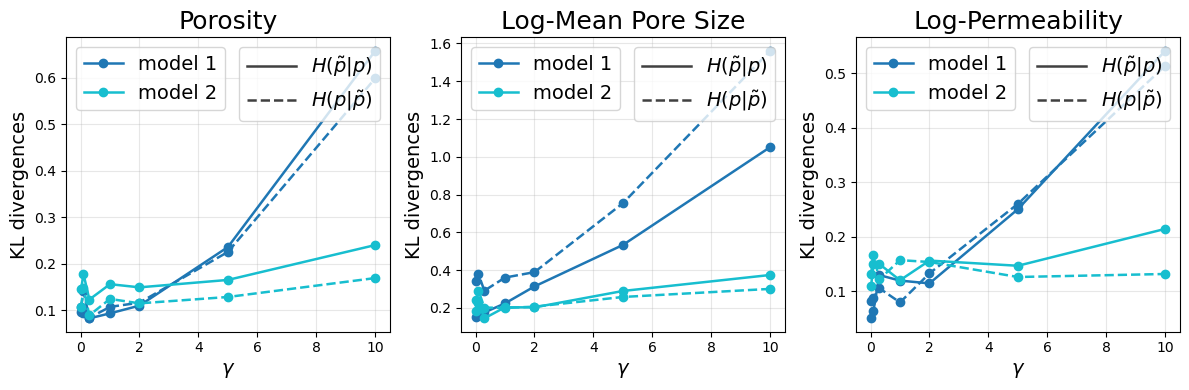}
        \caption{$\text{KL}\times\gamma$ curves for the distribution of geophysical properties.}
    \end{subfigure}
    \hfill
    \begin{subfigure}[b]{0.27\linewidth}
        \centering
        \includegraphics[width=\linewidth]{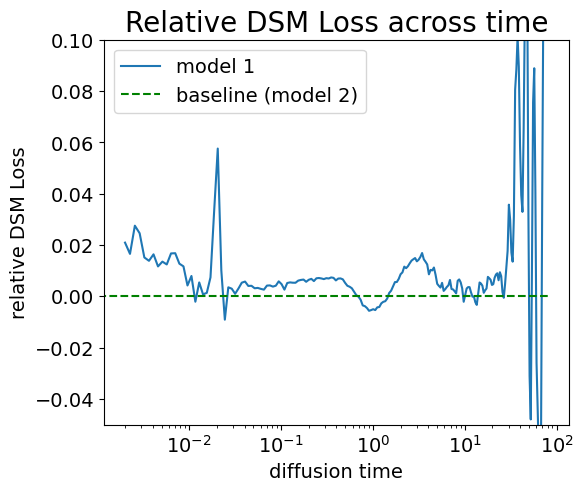}
        \caption{Relative DSM losses.}
    \end{subfigure}

    \caption{Comparison of two latent diffusion models trained on porous media volumes of Bentheimer sandstone, following \citep{frontgeologico}. Each volume comprises $64^3$ voxels at a $3$ $\mu$m resolution.}
    \label{fig: gamma curves and dsm error for rocks}
\end{figure}

Additionally, Figure \ref{fig: gamma curves and dsm error for rocks} shows examples of latent diffusion models trained on 3D volumes of porous media, in an application presented in \cite{frontgeologico}. Performance in this case is measured by the distribution of scalar geophysical properties, for which the KL divergences can be computed. This case illustrates the deterioration aspect: the larger error of model 1 for $t<1$ is associated with the error amplification for large $\gamma$ observed in all the metrics.

\subsubsection{Example: different impacts on the KL divergences}\label{sec: different KLs}

Closer inspection of the $\text{KL}\times\gamma$ curves for model 1, in Figure \ref{fig: gamma curves and error norms}, reveals that, for a certain interval of $\gamma$, the SDE can produce lower $H(\tilde p|p)$ but higher $H(p|\tilde p)$ when compared to the ODE. Figure \ref{fig: different KLs} shows the generated distributions and the evolution of the KL divergences for one such case. We can see a higher presence of outliers (samples in the low-density region of the data distribution) for the ODE, and lower mode coverage (lack of samples in some regions of the data distribution) for the SDE. This agrees with the usual interpretation of the KL divergences, mentioned in the Introduction.

% Different values for the divergences
\begin{figure}[ht!]
    \centering
    \includegraphics[width=0.48\linewidth]{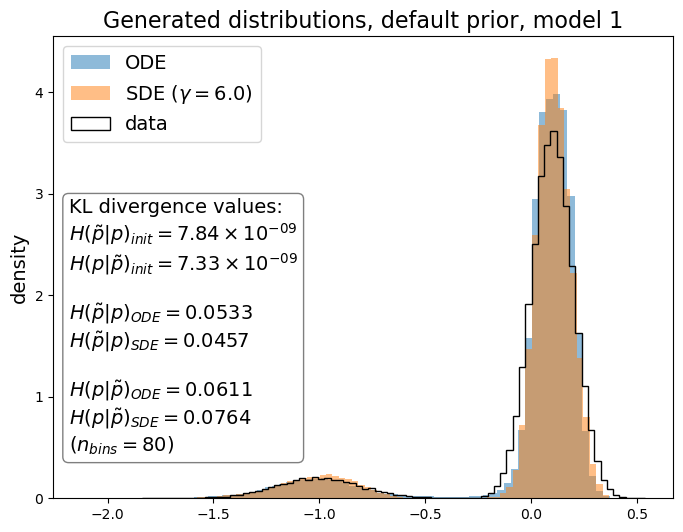}
    \includegraphics[width=0.48\linewidth]{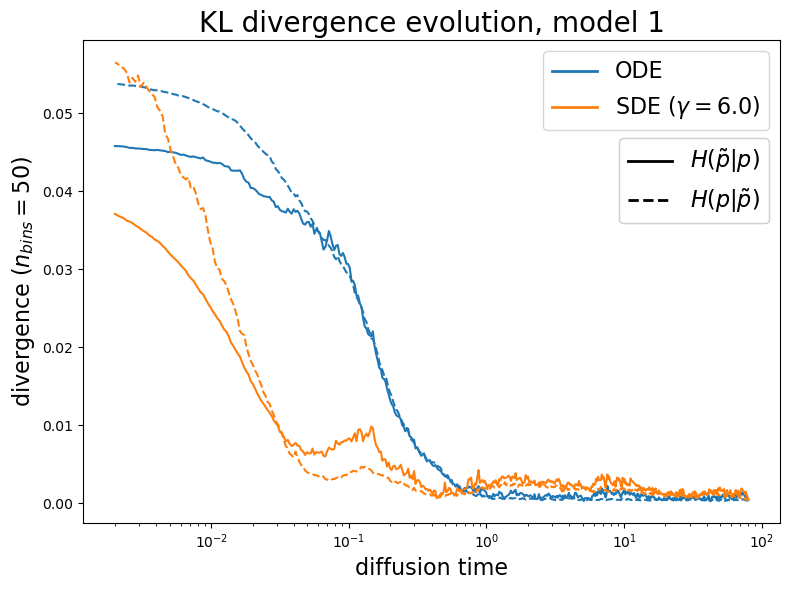}

    \caption{An example where stochasticity is beneficial in one of the divergences and detrimental in the other. Generated distributions (left) and KL divergence evolution (right), for model 1 of Figure \ref{fig: gamma curves and error norms} with the default EDM prior distribution.}
    \label{fig: different KLs}
\end{figure}

\subsubsection{Variable \texorpdfstring{$\gamma(t)$}{gamma(t)}} \label{subsec: variable_gamma}

The dependence of the relative improvement and relative deterioration on the score error within certain intervals suggests that the benefits from stochasticity can be maximized, and the detrimental effects mitigated, by setting a positive $\gamma$ only within certain intervals. Motivated by this and by the bang-bang structure of the instantaneous optimum obtained in Section \ref{subsec: instantaneous_opt}, we perform grid searches for the best interval $[S_{\min}, S_{\max}]$ to set a constant positive $\gamma$.

Figure \ref{fig: colormaps_mnist} shows FID values for the three MNIST models discussed in Section \ref{subsec: time_profile_experiments}, across several intervals $[S_{\min}, S_{\max}]$ and two choices of $\gamma$. First, notice that the minimum FID scores for models 2 and 3 are substantially lower than their minimum FID scores for constant $\gamma$ (Figure \ref{fig:mnist_fid_gamma_dsm}). We also see that stochasticity only in the beginning of sampling is always detrimental, with the largest deterioration for model 1, which has the largest initial score errors (Figure \ref{fig:mnist_fid_gamma_dsm}). This suggests that the SDE is amplifying score errors in the beginning of sampling, but it is not active in the end to correct the accumulated errors. In the case of $\gamma=5$, we also observe a deterioration when stochasticity is used only over the very end of sampling, which is also consistent with the positive correlation between the relative deterioration $D$ and the late/total error ratio $E_\text{late/tot}$ found in the previous section.

% MNIST colormaps
\begin{figure}[ht!]
    \centering
    \includegraphics[width=0.32\linewidth]{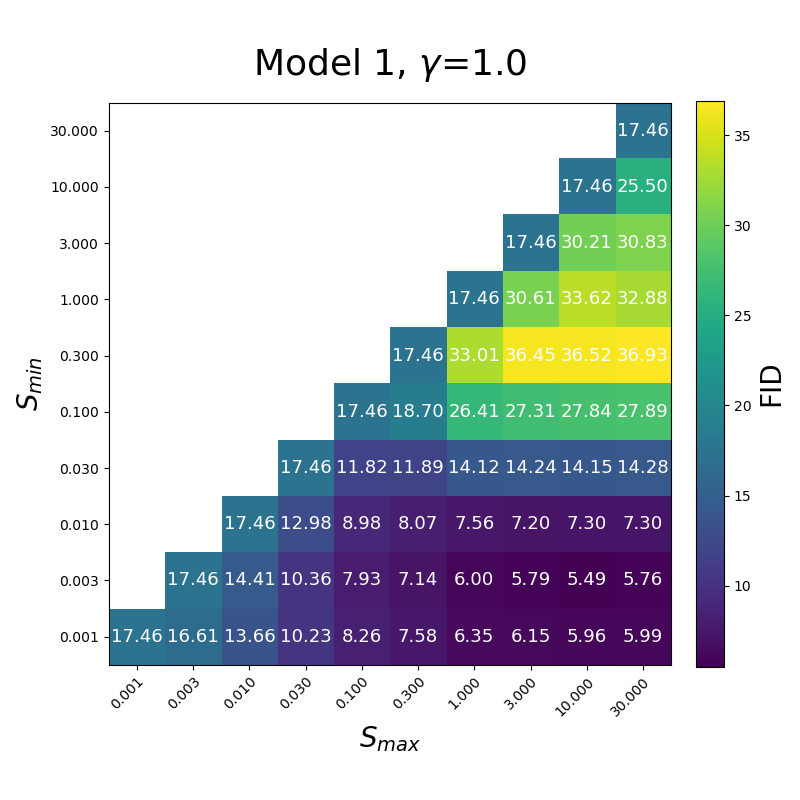}
    \includegraphics[width=0.32\linewidth]{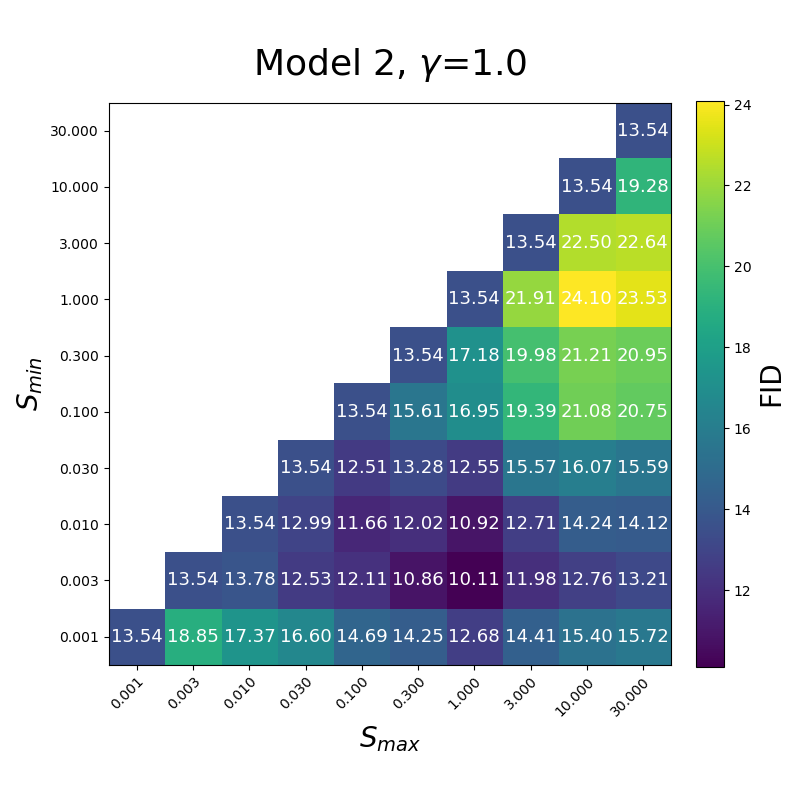}
    \includegraphics[width=0.32\linewidth]{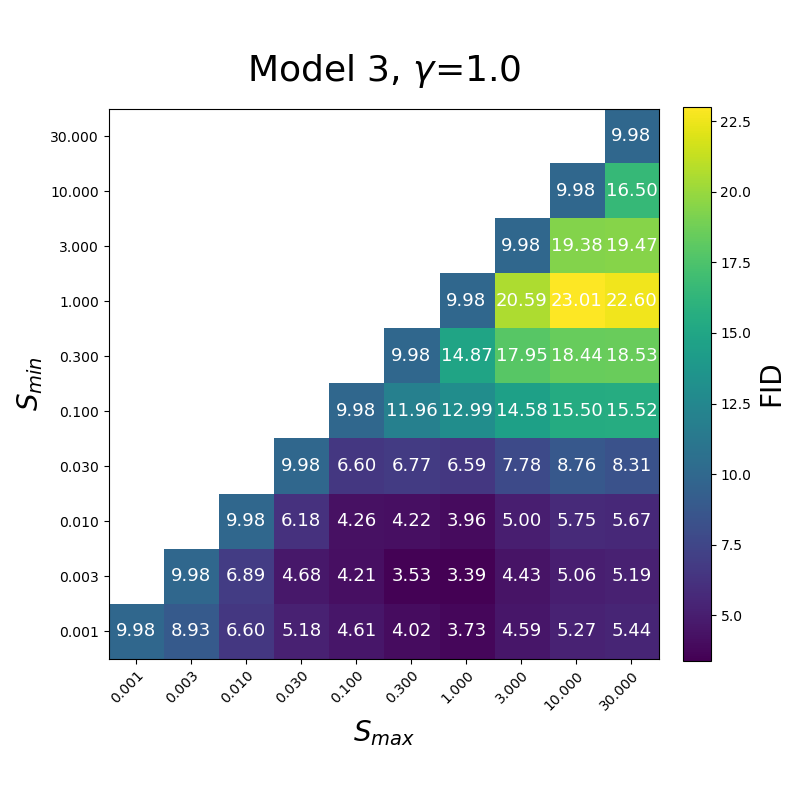}
    \includegraphics[width=0.32\linewidth]{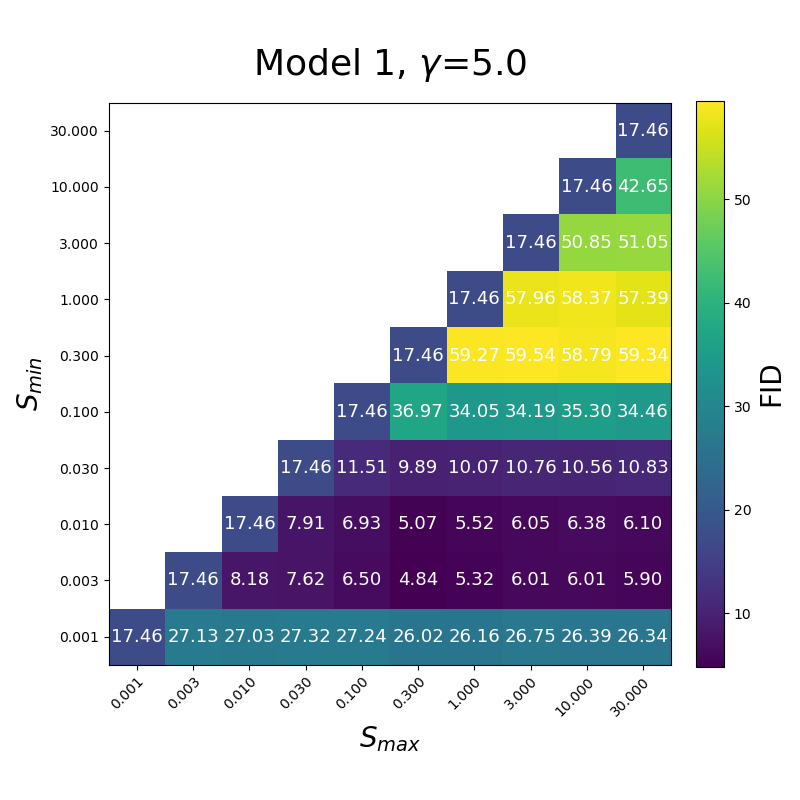}
    \includegraphics[width=0.32\linewidth]{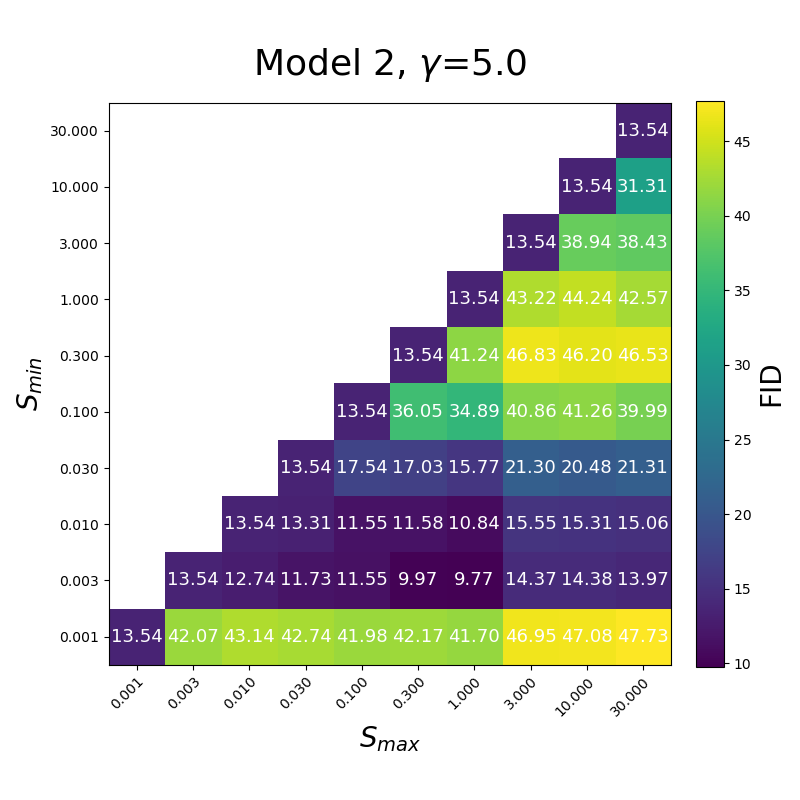}
    \includegraphics[width=0.32\linewidth]{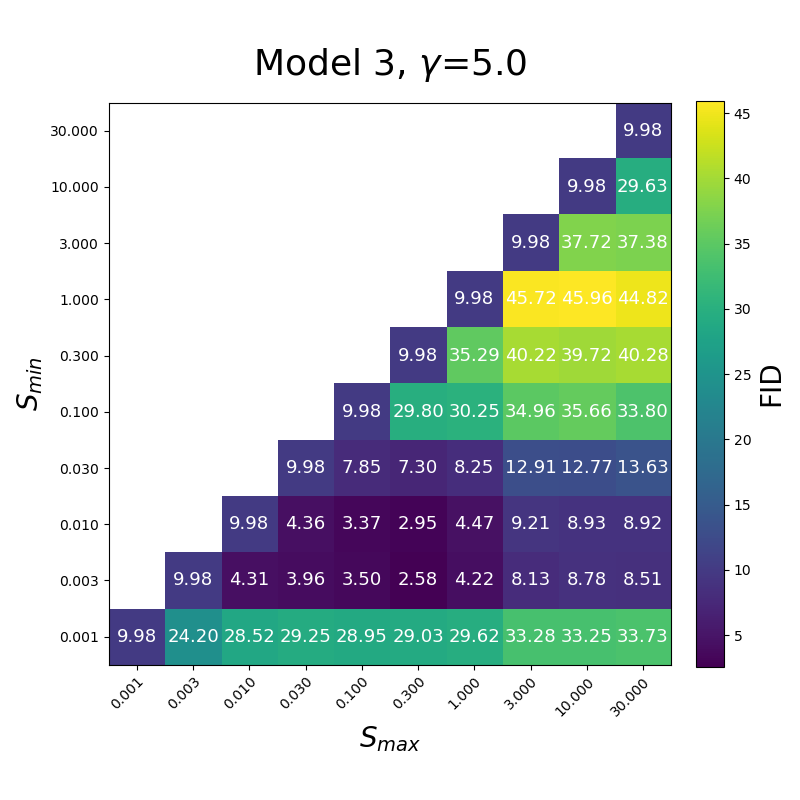}
    
    \caption{FID scores for the three MNIST models shown in Figure \ref{fig:mnist_fid_gamma_dsm}, where $\gamma(t)$ is set to $\gamma=1$, in the first row, and $\gamma=5$, in the second, within the interval $[S_{\min}, S_{\max}]$. Note that diagonal entries correspond to the pure ODE.}
    \label{fig: colormaps_mnist}
\end{figure}

Corresponding experiments for the mixture of Gaussians and CIFAR-10 data sets can be found in Figures \ref{fig: colormaps} and \ref{fig:colormaps_cifar10}. Although the worst stochasticity intervals vary significantly across models and data sets, the pattern for the most beneficial intervals appears to be more stable. For models which benefit significantly from stochasticity, we always see an optimal stochasticity interval located close to (but not at) the end of sampling, with $S_{\max}\leq1$ and $S_{\min}>0$. The idea of grid-searching for optimal stochasticity intervals is not new \citep{karras}, but the existence of an optimal stochasticity interval with $S_{\max}\leq1$ and $S_{\min}>0$ might be a more general feature of diffusion sampling, given that it also aligns with the time-profile analysis explored in Section \ref{subsec: instantaneous_opt}. If true, it would allow grid-searching for an optimal interval across a limited range of intervals, given a trained model. However, the experimental evidence provided here is still preliminary.

\begin{remark}[Conservativity of learned scores]
    In \cite[Section 4]{karras}, it is conjectured that the possible non-conservativity of the learned score could cause the observed detrimental effects of stochasticity.
    The experiments in this section show that even in a 1D setting, where every continuous function is conservative, some choices of $\gamma$ and of the interval $[S_{\min}, S_{\max}]$ make stochasticity detrimental, for certain trained models.
\end{remark}

\section{A fully analytical example}\label{sec: analytical example}

In this section, we consider an analytical example for which we can explicitly compute the sampling distributions, KL divergences, the first bound of Theorem \ref{teo 4}, and all other relevant quantities. This toy setting allows a direct analysis of the effect of the parameter $\gamma$, including an optimal control solution for $\gamma(\tau)$ at the end of the section. We will consider linear perturbations in the score and the prior distribution, and analyze how these errors interact with each other. Although linear perturbations do not accurately reflect the neural network errors found in practice, this is the only setting for which the reverse-time SDE can be analytically solved. In the following analysis, observe that the error attributed to the prior can also be seen as the error arising from inaccurate scores at the early phase of sampling.

\subsection{One-dimensional case}

We start with the one-dimensional case, assuming that the ``data'' distribution is a Gaussian $X_0 \sim \mathcal{N}(\mu_0, \sigma_0^2)$ and considering a forward equation of the form
\begin{equation}
    dX_t = g(t)\,dW_t,
\end{equation}
on a time interval $0\leq t \leq T$. The forward process is a Gauss-Markov process $X_t \sim \mathcal{N}(\mu_0, \sigma(t)^2),$ with score function
\begin{equation}
    \label{scorefullyanalytic}
    \nabla \log p(x, t) = -\frac{x - \mu_0}{\sigma(t)^2}, \qquad \sigma(t)^2 = \sigma_0^2 + \sigma_{\text{diff}}(t)^2, \quad \sigma_{\text{diff}}(t)^2 = \int_0^t g(s)^2 ds.
\end{equation} 
We consider an approximate score of the form
\[ s_\theta(x, t) = -\frac{x - \mu_\theta}{\sigma_\theta(t)^2}, \qquad \sigma_\theta(t)^2 = \alpha_\theta \sigma(t)^2,
\]
with constant parameters $\mu_\theta\in\mathbb{R}$ and $\alpha_\theta > 0.$ The error to the exact score, introduced in Section \ref{subsection: approx scores}, becomes
\begin{equation} 
    \epsilon_t(x) = s_\theta(x, t) - \nabla \log p(x, t) = \frac{\mu_\theta - \mu_0}{\sigma_\theta(t)^2} + \left(1 - \frac{\sigma(t)^2}{\sigma_\theta(t)^2} \right)\frac{x - \mu_0}{\sigma(t)^2}.
\end{equation}

This approximate score yields the family of linear approximate reverse equations
\begin{equation}\label{AE:pertubedequation}
    d\tilde{X}_\tau = -\frac{1}{2}\bar{g}^2(\tau)(1+\gamma(\tau))\left( \frac{\tilde{X}_\tau - \bar\mu_\theta(\tau)}{\bar\sigma_\theta(\tau)^2}\right) \,d\tau + \sqrt{\gamma(\tau)} \bar{g}(\tau) \,dW_\tau,
\end{equation}
in the interval $0\leq\tau\leq T$. Since this is a linear equation, it can be solved explicitly \cite[Section 5.6]{karatzas}, yielding the solution
\begin{equation}
    \tilde X_\tau = e^{-A(\tau)} \tilde X_0 + e^{-A(\tau)} \int_0^\tau \bar\mu_\theta(s) a(s) e^{A(s)} \;\mathrm{d}s + e^{-A(\tau)} \int_0^\tau \sqrt{\gamma(s)}\bar g(s) e^{A(s)} \;\mathrm{d}W_s,
\end{equation}
where
\[ A(\tau) = \int_0^\tau a(s)\;\mathrm{d}s, \qquad a(\tau) = \frac{1}{2} \frac{\bar g(\tau)^2 (1 + \gamma(\tau))}{\bar\sigma_\theta(\tau)^2}.
\]
Finally, we consider a prior, for the reverse process, which is also Gaussian, $\tilde{X}_0 \sim \mathcal{N}(\mu_T, \sigma_T^2),$ with $\sigma_T^2 = \beta_T \sigma(T)^2,$ for a given parameter $\beta_T > 0$. Thus, we obtain
\begin{equation}
    \tilde\mu_\theta(\tau) = e^{-A(\tau)}\mu_T + e^{-A(\tau)} \int_0^\tau \bar\mu_\theta(s) a(s) e^{A(s)} \;\mathrm{d}s
\end{equation}
and
\begin{equation}
    \tilde\sigma_\theta(\tau)^2 = e^{-2A(\tau)}\sigma_T^2 + e^{-2A(\tau)} \int_0^\tau \gamma(s) \bar g(s)^2 e^{2A(s)}\;\mathrm{d}s.
\end{equation}

\subsubsection{Constant \texorpdfstring{$\gamma(t)=\gamma$}{gamma(t)=gamma}}

For more explicit calculations, here we assume
\begin{equation}
    \label{parametersanalytic}
    \gamma(t) = \gamma, \quad \forall t \geq 0,
\end{equation}
with a given constant $\gamma \geq 0.$ Thus, the free (constant) parameters for the approximate reverse process are the shape-changing parameters $\alpha_\theta, \beta_T > 0,$ the translation parameters $\mu_\theta, \mu_T \in \mathbb{R},$ the stochasticity parameter $\gamma \geq 0,$ and the starting time $T > 0.$ Besides these, we also have the parameters $\mu_0, \sigma_0$ defining the data distribution.

Then, the mean-square error in the score becomes
\begin{equation}
        \mathbb{E}_{\bar p_\tau}[\tilde\epsilon_\tau^2] = \frac{(\mu_\theta - \mu_0)^2}{\alpha_\theta^2\bar\sigma(\tau)^4} + \left(1 - \frac{1}{\alpha_\theta} \right)^2 \frac{1}{\bar\sigma(\tau)^2}.
\end{equation}
Since the variance of the forward diffusion increases with time, the error in the score decreases, which coincides with the behavior in the numerical example, where the score is approximated by a neural network trained with denoising score-matching (see e.g. Figure \ref{fig: learned scores}).

With constant $\gamma$, the mean and the variance of the Gauss-Markov approximate reverse process become
\begin{equation}
    \label{meananalytic}
    \tilde\mu_\theta(\tau) = \mu_\theta + \left(\frac{\bar\sigma(\tau)^2}{\bar\sigma(0)^2}\right)^{\frac{1}{2} \frac{(1 + \gamma)}{\alpha_\theta}}(\mu_T - \mu_\theta)
\end{equation}
and
\begin{align}
    \label{varianceanalytic}
    \tilde\sigma_\theta(\tau)^2 & = \begin{cases}
        \bar\sigma(\tau)^2 \left( \frac{\gamma\alpha_\theta}{(1 + \gamma) - \alpha_\theta} + \left(\beta_T - \frac{\gamma\alpha_\theta}{(1 + \gamma) - \alpha_\theta}\right)\left(\frac{\bar\sigma(\tau)^2}{\bar\sigma(0)^2}\right)^{\frac{(1 + \gamma)}{\alpha_\theta} - 1} \right), & 1 + \gamma \neq \alpha_\theta, \\
        \bar\sigma(\tau)^2\left( \beta_T +\gamma\ln\left(\frac{\bar\sigma(0)^2}{\bar\sigma(\tau)^2}\right)\right), & 1 + \gamma = \alpha_\theta,
    \end{cases}
\end{align}
where the apparent singularity as $\gamma \rightarrow \alpha_\theta - 1$ in the first expression is a removable singularity. 

With that, we compute both KL divergences
\begin{align}
    \label{KLQPdivergencescalaranalyticexample}
    H(\tilde{p}_\tau|\bar{p}_\tau) & = \frac{1}{2}\left(\ln {\frac {\bar\sigma(\tau)^2}{\tilde\sigma_\theta(\tau)^2}}+{\frac {\tilde\sigma_\theta(\tau)^{2}+(\mu_{0}-\tilde\mu_\theta(\tau))^{2}}{\bar\sigma(\tau)^{2}}}-1 \right), \\
    \label{KLPQdivergencescalaranalyticexample}
    H(\bar{p}_\tau|\tilde{p}_\tau) & = \frac{1}{2}\left(\ln {\frac{\tilde\sigma_\theta(\tau)^2}{\bar\sigma(\tau)^2}}+{\frac {\bar\sigma(\tau)^{2}+(\mu_{0}-\tilde\mu_\theta(\tau))^{2}}{\tilde\sigma_\theta(\tau)^{2}}}-1 \right).
\end{align}

In Figure \ref{fig:analyticKLbounds}, we can see the evolution of $H(\tilde{p}_\tau|\bar{p}_\tau)$ and of the two bounds of Theorem \ref{teo 4}, for a particular choice of prior and score errors, for different values of $\gamma$. Note that the second bound can, in some cases, be tighter than the first one, even though its derivation uses one further inequality. This is possible since the extra inequality is applied before the LSI, as discussed in Remark \ref{remark: approx bound}.

\begin{figure}[ht!]
    \centering
    \includegraphics[width=0.32\linewidth]{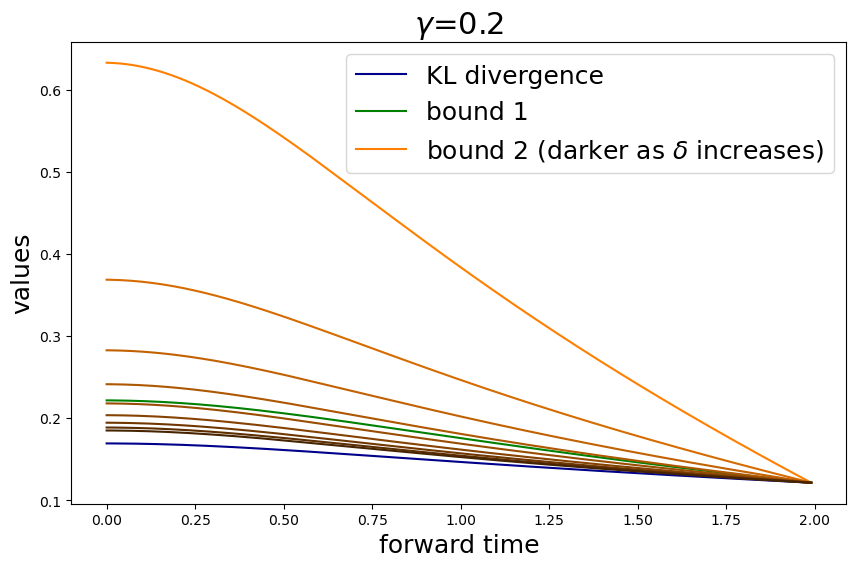}
    \includegraphics[width=0.32\linewidth]{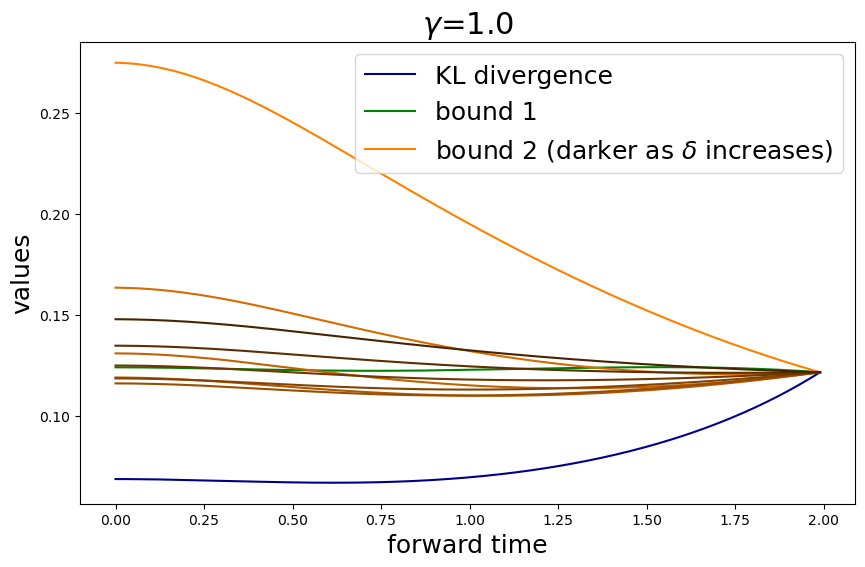}
    \includegraphics[width=0.32\linewidth]{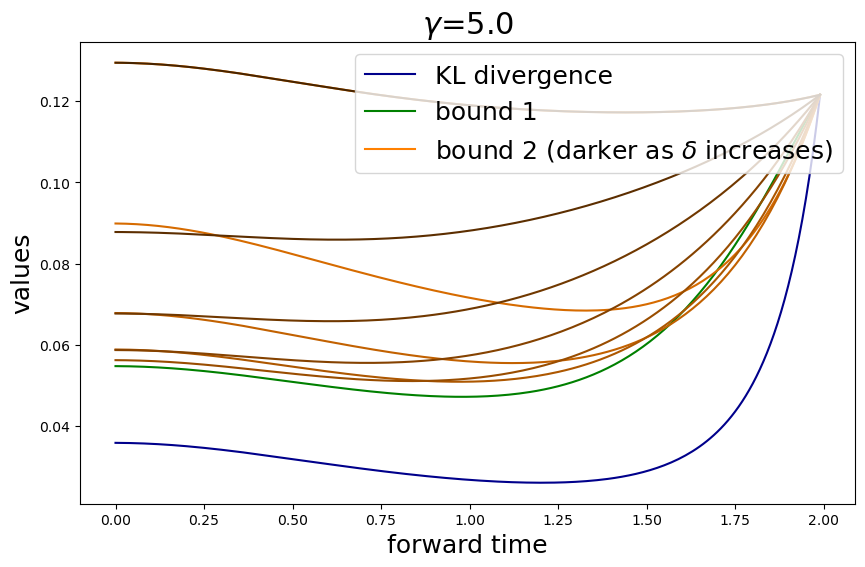}
    \hfill
    \caption{Evolution of $H(\tilde p|p)$ and its bounds for some choices of constant $\gamma$, with the sampling parameters $\mu_\theta=0.2$, $\alpha_\theta=0.8$, $\mu_T=0.5$, $\beta_T=0.5$ and $T=2$.}
    \label{fig:analyticKLbounds}
\end{figure}

We now use formulas \eqref{KLQPdivergencescalaranalyticexample} and \eqref{KLPQdivergencescalaranalyticexample} to better understand the effect of $\gamma$ on the KL divergences. We summarize, below, the findings discussed in more detail in Appendix \ref{app: proofs sec5}. Assuming, for simplicity, that the mean is preserved, i.e. $\tilde\mu_\theta(\tau) = \mu_0$, stemming from the choices $\mu_T = \mu_\theta = \mu_0$, the KL divergences reduce, after sampling (corresponding to $t = 0$ or, equivalently, $\tau = T$), to
\[
    H(\tilde p_T | p_0) = \frac{v(\gamma) - \ln v(\gamma) - 1}{2} \qquad \textrm{and} \qquad H(p_0 | \tilde p_T) = \frac{\frac{1}{v(\gamma)} + \ln v(\gamma) - 1}{2}
\]
where
\[
    v(\gamma) = \frac{\tilde\sigma_\theta(T)^2}{\bar\sigma(T)^2} = \begin{cases}
         \frac{\gamma\alpha_\theta}{(1 + \gamma) - \alpha_\theta} + \left(\beta_T - \frac{\gamma\alpha_\theta}{(1 + \gamma) - \alpha_\theta}\right)\zeta_T^{\frac{(1 + \gamma)}{\alpha_\theta} - 1}, & 1 + \gamma \neq \alpha_\theta, \\
        \beta_T  - \gamma\ln\zeta_T, & 1 + \gamma = \alpha_\theta,
    \end{cases}
\]
with $\zeta_T = \bar\sigma(T)^2/\bar\sigma(0)^2 = \sigma_0^2 / ( \sigma_0^2 + \sigma_{\text{diff}}(T)^2 ) < 1.$

If the score is exact and only the prior is off, meaning $\alpha_\theta = 1$ and $\beta_T \neq 1,$ we have $v(\gamma) = 1 + (\beta_T - 1)\zeta_T^{\gamma} \neq 1$, and $v(\gamma)$ converges monotonically and exponentially towards $1$ as $\gamma \rightarrow \infty$. In this case, both KL divergences decay exponentially to zero, as $\gamma$ increases, as expected. This means stochastic sampling is always better than ODE sampling when the score is exact, as discussed before.

With $\alpha_\theta \neq 1,$ the behavior is more involved. For $\gamma = 0,$ we simply have $v_0:=v(0) = \beta_T \zeta_T^{(1 - \alpha_\theta)/\alpha_\theta},$ while, asymptotically in $\gamma,$ we have $v_\infty := \lim_{\gamma\rightarrow\infty} v(\gamma) = \alpha_\theta.$ Looking at the ratio $v_0/v_\infty = (\beta_T/\alpha_\theta)\zeta_T^{(1 - \alpha_\theta)/\alpha_\theta}$, we see that $v_0 < v_\infty$ if, and only if, either $\zeta_T < (\alpha_\theta/\beta_T)^{\alpha_\theta / (1 - \alpha_\theta)},$ when $\alpha_\theta < 1$, or $\zeta_T > (\beta_T/\alpha_\theta)^{\alpha_\theta / (\alpha_\theta - 1)},$ when $\alpha_\theta > 1$. Taking into account that $\zeta_T < 1$ and looking at $v'(\gamma)$, we see that, when $\beta_T \leq \alpha_\theta < 1,$ the variance ratio $v(\gamma)$ increases monotonically from $v_0$ to $v_\infty,$ approaching $v_\infty=\alpha_\theta$ from below, so both KL divergences decrease monotonically with $\gamma.$ Similarly, when $\beta_T \geq \alpha_\theta > 1,$ the variance ratio $v(\gamma)$ decreases monotonically from $v_0$ toward $v_\infty = \alpha_\theta,$ so both KL divergences decrease monotonically toward $h_\infty.$ In both cases, increasing $\gamma$ is beneficial, and $\gamma \rightarrow \infty$ is optimal. These two regimes occur when the training and the prior errors are aligned, i.e. either both underestimate or both overestimate the exact variances, provided also that the error in the prior is relatively more pronounced. 

Analogous behaviors occur when considering only errors in the mean, and they are also aligned, i.e. when either $\mu_T \leq\mu_\theta\leq \mu_0$ or $\mu_T \geq\mu_\theta\geq \mu_0$. Then, increasing $\gamma$ is beneficial, and $\gamma \rightarrow \infty$ is optimal (see Appendix \ref{app: analytical_exact_variances}). In more general scenarios, the behavior is more nuanced, and the KL divergences may not be monotonic, and may achieve local or global maxima or minima at finite values of the diffusion parameter. In some cases, the error in the prior may compensate the error in the score, and sampling with the reverse ODE may be better. In other cases, the best sampling occurs at an intermediate value of the stochasticity parameter. Figure \ref{fig:analyticKLongammaregimes} illustrates a few different scenarios, where the 2 leftmost plots of the first row are the aligned cases mentioned above.
Interestingly, we can also see that the performance of the SDEs seems to be more stable to variations of the parameters, as they converge quickly in $\gamma$ to the value of $\alpha_\theta$, independently of the other parameters.

With the same scenarios, Figure \ref{fig:analyticmusigmaevolution} shows the evolution of the mean $\tilde\mu_\theta(t)$ and the standard deviation $\tilde\sigma_\theta(t)$ in the phase-space plane $(\mu, \sigma)$ of the analytical solution, given by \eqref{meananalytic} and \eqref{varianceanalytic}. In the first two columns, variance error from the prior and the score are aligned, either underestimating or overestimating the true value, and the error accumulation makes stochastic sampling better. In the remaining plots, the two errors are in different directions, yielding non-monotonic curves in $\gamma$, but only in the last one the prior error and score errors balance out in a way that the reverse ODE is always better.

Interestingly, the alignment of the prior error with the score error is related to the alignment of $\epsilon$ and $\nabla\log\tilde{p}_\tau/\bar{p}_\tau$, discussed in Remark \ref{remark: error direction} and also Section \ref{subsec: instantaneous_opt}. For $\mu_0=0$, at the initial time $\tau=0$ we have
\begin{equation}
     E_{\tilde{p}_0}\big[\epsilon_0\cdot\nabla\log(\tilde{p}_0/\bar{p}_0)\big] = \frac{1}{\alpha_\theta\sigma(T)^4}\left(\mu_\theta \mu_T + \mu_T^2(\alpha_\theta-1) + (\alpha_\theta-1)(\beta_T-1)\sigma(T)^2\right).
\end{equation}
The first term is negative if the means are in opposite directions, i.e. $\mu_\theta\mu_T<0$, and the third term is negative if the variances are in opposite directions, i.e. $(\alpha_\theta-1)(\beta_T-1)<0$. The second term gives the more subtle interaction between score variance and prior mean, and it can become large when both the error in the prior mean and the error in the score variance are large, which intuitively corresponds to the under-correction of the prior error on the mean by too wide score fields. Nevertheless, for typical values of $\sigma(T)$ and $\mu_T$ the third term becomes dominant, making the variance alignment the dominant factor. This is exemplified in the insets of Figure \ref{fig:analyticmusigmaevolution} showing the sign of $\epsilon\nabla\log\tilde{p}_\tau/\bar{p}_\tau$. When the variances of the errors are misaligned, we see that the score error initially reduces the KL (negative $\epsilon\nabla\log\tilde{p}_\tau/\bar{p}_\tau$), until the accumulated error $\nabla\log\tilde{p}_\tau/\bar{p}_\tau$ switches to the direction of the score error.

\begin{figure}[ht!]
    \centering
    \includegraphics[width=0.96\linewidth]{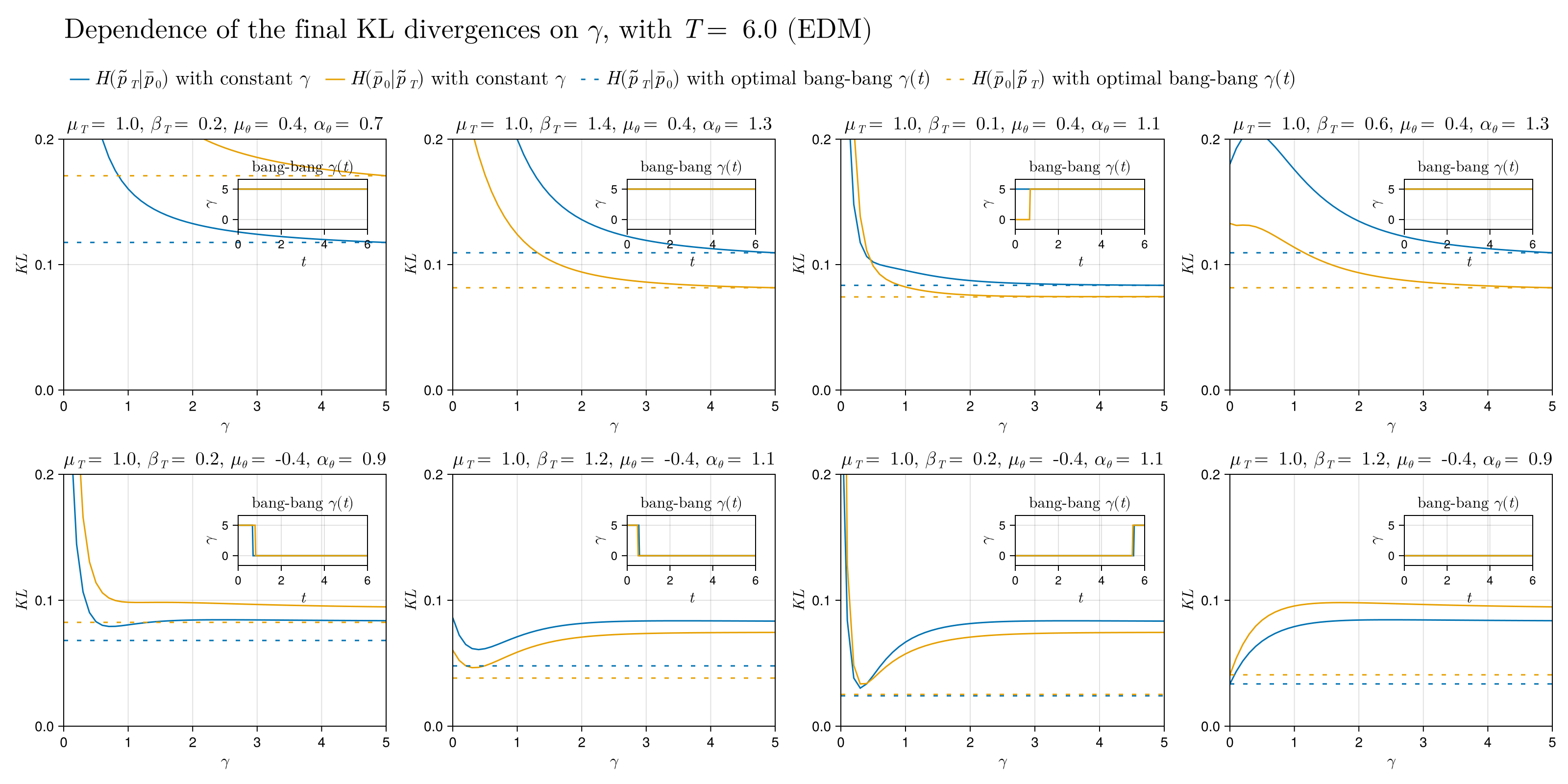}
    \caption{Final KL divergences $H(\tilde p_T | p_0)$ and $H(p_0 | \tilde p_T)$ as functions of $\gamma$ (solid lines); the final divergences with optimal bang-bang diffusion with maximum stochasticity $\gamma_{\textrm{max}} = 5.0$ (dashed lines); for $\mu_T=1.0$ and various choices of the sampling parameters $\mu_\theta,$ $\alpha_\theta,$ and $\beta_T,$ with data parameters $\mu_0=0,$ $\sigma_0=1.0,$ and final time $T=6.0.$ In the insets, we show the $\gamma(\tau)$ profiles for the optimal bang-bang diffusions.}
    \label{fig:analyticKLongammaregimes}
\end{figure}

\begin{figure}[ht!]
    \centering
    \includegraphics[width=0.96\linewidth]{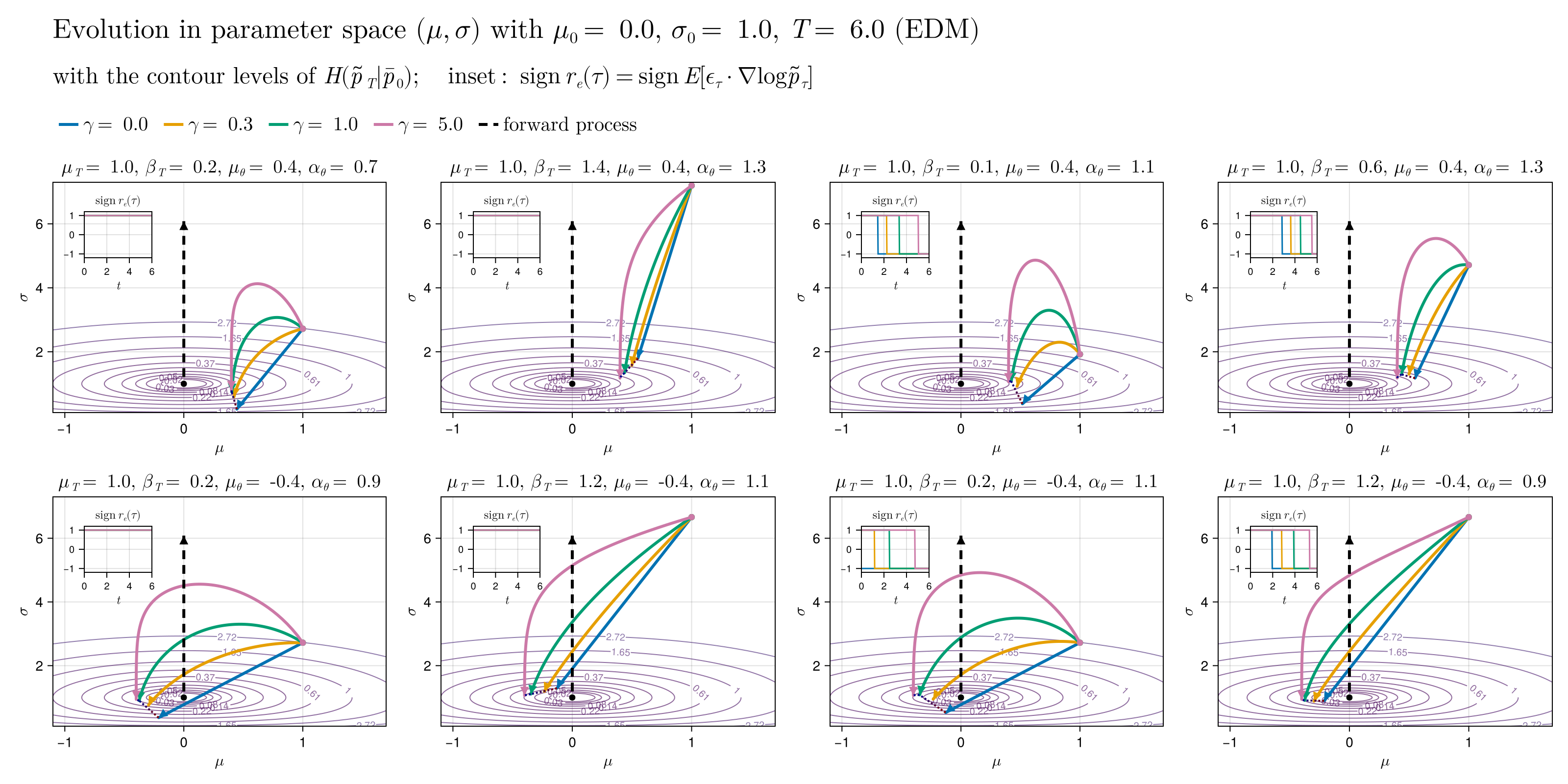}
    \caption{Evolution of the mean and standard deviation $\tau \mapsto (\tilde\mu_\theta(\tau),\tilde\sigma_\theta(\tau))$ of the analytical solution, for selected values of the stochasticity parameter ($\gamma=0.0, 0.3, 1.0, 5.0$), with the same choices of parameters as in Figure \ref{fig:analyticKLongammaregimes}. The background curves are the level curves of the final sample-quality KL divergence $H(\tilde p_T | p_0)$, and the insets show the dynamic alignment measure $\operatorname{sgn}(\epsilon\nabla\log\tilde{p}_\tau/\bar{p}_\tau)$ along sampling.}
    \label{fig:analyticmusigmaevolution}
\end{figure}

\subsubsection{Variable \texorpdfstring{$\gamma(t)$}{gamma(t)}}\label{sec: analytical_opt_control}

We may also consider a time-varying $\gamma = \gamma(t)$ and look for the function that minimizes each terminal KL divergence. This can be formulated as an optimal control problem with terminal cost, imposing an upper bound $\gamma(t)\in [0, \gamma_{\textrm{max}}]$ on the control variable. As detailed in Appendix \ref{app: optimal control}, the control problem in this specific analytical example turns out to be linear in $\gamma$, and applying Pontryagin's maximum principle reveals that the optimal solution is a ``bang-bang'' control: the optimal $\gamma(t)$ takes only two possible values, alternating between the probability flow ODE and stochastic sampling with maximal stochasticity parameter. Moreover, the solution admits \textit{at most one switch}. Remarkably, this structure is independent of the specific terminal cost considered, being a consequence of the underlying linear dynamics of the state space arising from the linear perturbation on the score.

Figure \ref{fig:analyticKLongammaregimes} also shows, in the dashed lines, KL divergences for the optimal $\gamma^*(t)$, where the switching time $\tau^*$ is obtained by numerically searching in the interval $[0, T]$. We can see that the optimal $\gamma^*(t)$ yields lower KL divergences than any constant $\gamma$ when the mean errors are in different directions but the variance errors are not. This global bang-bang optimum can also be compared with the instantaneous optimum discussed in Section 3.2.1, which results from minimizing the KL divergence time derivative at each instant $\tau$. This is done in Figure \ref{fig:instant_global_optima} for the previous choices of parameters, where we see that in some cases the final divergences of the instantaneous optimum agree with those of the global optimum, but that is not always the case. This indicates some limitations of the reasoning developed in Section \ref{subsec: instantaneous_opt}, at least in this simplified setting.

\begin{figure}[ht!]
    \centering
    \includegraphics[width=0.96\linewidth]{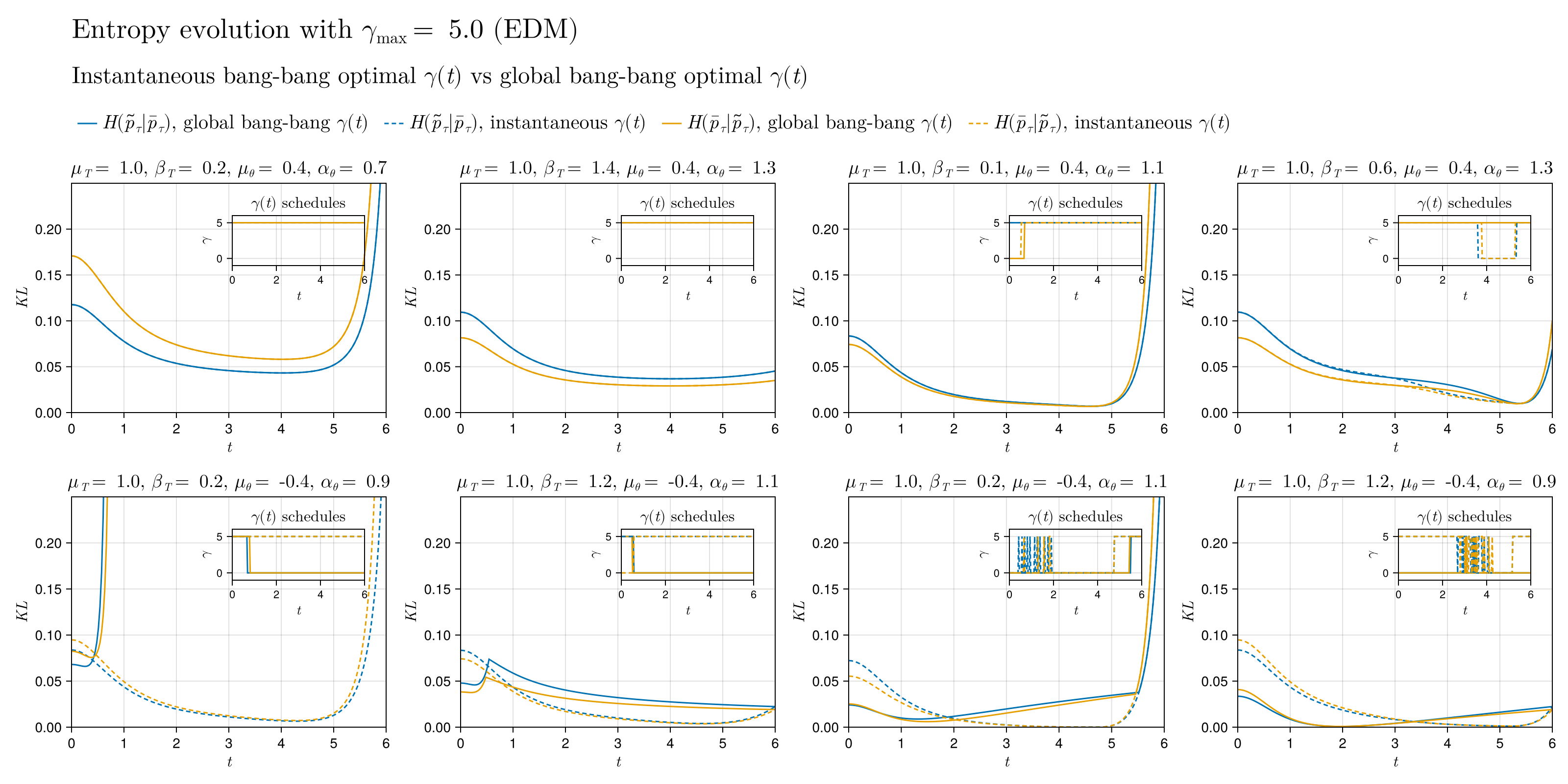}
    \caption{Evolution of KL divergences for both the instantaneous and the global optimal $\gamma(t)$, with the same choices of parameters as in Figure \ref{fig:analyticKLongammaregimes}. In the insets, we show the $\gamma(t)$ profiles for the optima.}
    \label{fig:instant_global_optima}
\end{figure}

\subsection{Multidimensional case}\label{subsec:AEmulti}

We now extend this example to the multivariate case, with the ``data'' distribution being a Gaussian $X_0 \sim \mathcal{N}(\mu_0, \Sigma_0)$ in $\mathbb{R}^n$. To simulate a $d$-dimensional ``data'' manifold, we let the covariance be block-diagonal:
\begin{equation}
    \Sigma_0 = \begin{pmatrix} {\sigma_0'}^2 I_d & 0 \\ 0 & {\sigma_0''}^2 I_{n-d} \end{pmatrix},
\end{equation}
with the variances satisfying ${\sigma_0''}^2 \ll {\sigma_0'}^2$. The first $d$ components represent the data manifold, embedded on a high-dimensional ambient space with dimension $n \gg d.$

The forward equation is essentially the same, $dX_t = g(t)\,dW_t,$ so the forward process remains Gaussian with covariance $\Sigma(t) = \Sigma_0 + \sigma_{\text{diff}}(t)^2 I_n$, where $\sigma_{\text{diff}}(t)^2$ is as given in \eqref{scorefullyanalytic}.

The exact score function is linear and splits into two independent blocks:
\[ \nabla \log p(x, t) = -\Sigma(t)^{-1}(x - \mu_0) = \begin{pmatrix} -\frac{x' - \mu'_0}{{\sigma'}(t)^2} \\ -\frac{x'' - \mu''_0}{{\sigma''}(t)^2} \end{pmatrix}, \]
where ${\sigma'}(t)^2 = {\sigma_0'}^2 + \sigma_{\text{diff}}(t)^2$ and ${\sigma''}(t)^2 = {\sigma_0''}^2 + \sigma_{\text{diff}}(t)^2$.

We consider an approximate score $s_\theta(x, t)$ modeled as a linear perturbation of the exact score. We introduce a diagonal error matrix $\Lambda_\theta = \text{diag}(\alpha' I_d, \alpha'' I_{n-d})$ and approximate parameters $\mu_\theta$, yielding
\begin{equation}
    s_\theta(x, t) = - \Lambda_\theta^{-1} \Sigma(t)^{-1} (x - \mu_\theta)  = \begin{pmatrix} -\frac{x' - \mu'_\theta}{\alpha' {\sigma'}(t)^2} \\ -\frac{x'' - \mu''_\theta}{\alpha'' {\sigma''}(t)^2} \end{pmatrix}.
\end{equation}
This formulation allows for different score accuracies on the manifold versus the ambient space, depending on $\mu_\theta', \alpha'$ and $\mu_\theta'', \alpha''$. This modeling allows for simulating and investigating the situation in which learned scores first align to the manifold and later learn finer details of the distribution, as suggested in a recent work \citep{shen2026manifold}.

The form of the perturbed score yields a linear reverse SDE. The prior is taken without knowledge of the manifold and is an isotropic Gaussian $\tilde{X}_0 \sim \mathcal{N}(\mu_T, \sigma_T^2 I_n)$, and we write $\sigma_T^2 = \beta'_T\sigma'(T)^2,$ for a parameter $\beta'_T$.

Since the matrices $\Sigma_0$ and $\Lambda_\theta$ commute (they are both diagonal), the system decouples into two independent blocks: the manifold block $\tilde{X}'$ and the ambient block $\tilde{X}''$. The corresponding variances evolve according to the scalar formulas derived for the 1D case, shown in equation \eqref{varianceanalytic}. For the manifold block (assuming $1+\gamma \neq \alpha'$), we have
\begin{equation}
    \tilde{\sigma}'_\theta(\tau)^2 = \bar{\sigma}'(\tau)^2 \left( \frac{\gamma\alpha'}{(1 + \gamma) - \alpha'} + \left(\beta'_T - \frac{\gamma\alpha'}{(1 + \gamma) - \alpha'}\right)\left(\frac{\bar{\sigma}'(\tau)^2}{\bar{\sigma}'(0)^2}\right)^{\frac{(1 + \gamma)}{\alpha'} - 1} \right).
\end{equation}
An analogous formula holds for the ambient variance $\tilde{\sigma}''_\theta(\tau)^2$ using parameter $\alpha''$ and the dependent parameter $\beta''_T:=\beta'_T\sigma'(T)^2/\sigma''(T)^2$.

Since the marginals are independent, the total KL divergence is the sum of the divergences of the two orthogonal subspaces, which are then multiples of the scalar components:
\begin{equation}
    H(\tilde{p}_\tau|\bar{p}_\tau) = d h'(\tau) + (n-d) h''(\tau),
\end{equation}
where $h'(\tau)$ and $h''(\tau)$ are the corresponding scalar Gaussian KL divergence contribution, given by \eqref{KLQPdivergencescalaranalyticexample}, with the appropriate parameters.

\begin{figure}
    \centering
    \includegraphics[width=\linewidth]{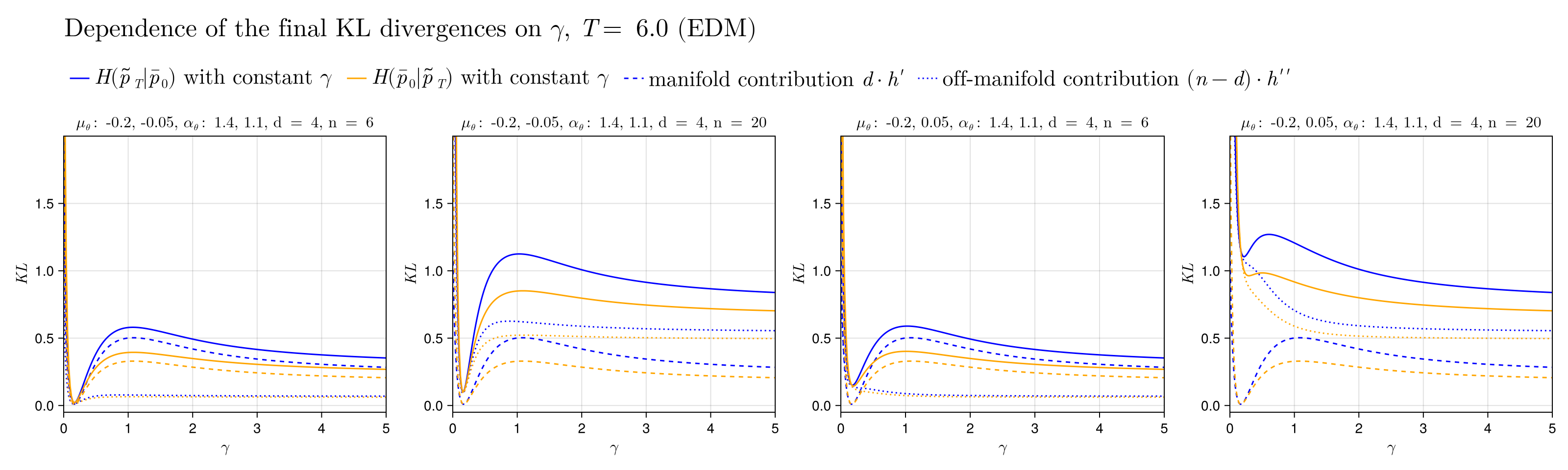}
    \caption{Final KL divergences $H(\tilde p_T | p_0)$ and $H(p_0 | \tilde p_T)$ as functions of $\gamma$ (solid lines) for the multidimensional case; for $\mu_T=1.0, \beta'_T=0.1$ and some choices of the score parameters $\mu_\theta,$ $\alpha_\theta,$ with data parameters $\mu_0=0,$ $\sigma_0=1.0,$ and final time $T=6.0.$}
    \label{fig:multidim}
\end{figure}

Since $n-d \gg d$ in typical generative modeling scenarios, the global KL divergence is heavily weighted on the ambient term, but that can be counteracted by a relatively smaller error $h''(\tau)$, on the direction normal to the manifold. Nevertheless, this reveals a fundamental trade-off: tune $\gamma$ to $\alpha'$ to improve the distribution within the manifold, as characterized here by $\sigma_0'$, or tune $\gamma$ to $\alpha''$ to improve the location of the manifold, as characterized here by $\sigma_0''$. Figure \ref{fig:multidim} exemplifies this trade-off: in the first two plots, the two optimal constant-$\gamma$ values are close, whereas in the last two plots they are not.

As for the optimal time-dependent stochasticity parameter $\gamma=\gamma(t)$, the control problem has the same essential structure, with the same bang-bang solution. However, since the switching function is now a sum of two terms, one for each block, with different exponential rates, it can change sign at most twice, yielding an optimal schedule with at most two switches, e.g., deterministic, then stochastic, then deterministic again, reflecting the competing requirements of the two groups. If we extend the analysis allowing the approximate score to have $d$ different score error factors $\alpha_i$, the same structure persists, with at most $d$ switches, as detailed in Appendix \ref{app: optimal control multi}.

\section{Conclusions}

In this work, we investigated the effect of stochasticity on the generated distribution through the lens of KL divergence evolution along sampling. At each instant, the time derivative of the divergences in stochastic sampling (Proposition \ref{prop 2}) reveals a fundamental trade-off between the correction of accumulated errors and the amplification of current score errors. From this general property, more specific conclusions are derived. When sampling with exact scores, only the error correction part is present, and log-Sobolev inequalities can be used to bound the divergence decay from above (Theorem \ref{teo 2}). The same technique produces bounds in the case of approximate scores (Theorem \ref{teo 4}), but the additional term associated with the score error prevents the derivation of an explicit divergence decay.

Minimizing the time derivative of the divergences at each instant with respect to $\gamma$ might not correspond to a minimum for the final KL divergence but provides a complementary perspective to the bounds. In particular, we show that stochasticity reduces the divergence at a given instant when the norm of the score error at that instant is smaller than a threshold depending on the magnitude of the accumulated errors (Section \ref{subsec: instantaneous_opt}). We further investigate this point in numerical experiments, and evaluate correlations between the time-profile of the model's score error and the impact of stochasticity on its performance (Section \ref{subsec: time_profile_experiments}). Moreover, grid searches over intervals that incorporate stochasticity, with models trained on different data sets, indicate that a family of intervals near the end of sampling produces the best results (Section \ref{subsec: variable_gamma}). This is consistent with the analysis in Section \ref{subsec: instantaneous_opt}, and we conjecture that it might be a more general property of diffusion models, although the exact optimal interval could depend on the data set and the architecture considered.

% or


\section*{Acknowledgment}

The authors acknowledge the support of ExxonMobil Exploração Brasil Ltda. and Agência Nacional de Petróleo, Gás Natural e Biocombustíveis (ANP) through the grant no. 23789-1.

\bibliographystyle{abbrv}
\bibliography{references}

@inproceedings{karras,
author = {Karras, Tero and Aittala, Miika and Laine, Samuli and Aila, Timo},
title = {Elucidating the design space of diffusion-based generative models},
year = {2022},
isbn = {9781713871088},
publisher = {Curran Associates Inc.},
address = {Red Hook, NY, USA},
booktitle = {Proceedings of the 36th International Conference on Neural Information Processing Systems},
articleno = {1926},
numpages = {13},
location = {New Orleans, LA, USA},
series = {NIPS '22}
}

@article{villani,
     author={Peter A. Markowich and C{\'e}dric Villani},
    title = {On the trend to equilibrium for the {Fokker-Planck} equation: an interplay between physics and functional analysis},
    journal = {Mat. Contemp.},
    volume={19},
    pages={1-29},
    year = {2000}
}

@article{anderson,
    author = {Brian DO Anderson},
    title = {Reverse-time diffusion equation models},
    journal = {Stochastic Process. Appl., 12(3): 313–326},
    year = {1982}
}

@article{LSImixtures,
   title={Poincaré and {Log–Sobolev} Inequalities for Mixtures},
   volume={21},
   ISSN={1099-4300},
   DOI={10.3390/e21010089},
   number={1},
   journal={Entropy},
   publisher={MDPI AG},
   author={Schlichting, André},
   year={2019},
   month=jan, pages={89} }

@article{dimension-free,
title = {Dimension-free {log-Sobolev} inequalities for mixture distributions},
journal = {Journal of Functional Analysis},
volume = {281},
number = {11},
pages = {109236},
year = {2021},
issn = {0022-1236},
doi = {https://doi.org/10.1016/j.jfa.2021.109236},
author = {Hong-Bin Chen and Sinho Chewi and Jonathan Niles-Weed},
}

@article{song2021scorebasedmodeling,
    author = {Yang Song and Jascha Sohl-Dickstein and Diederik P. Kingma and Abhishek Kumar and Stefano Ermon and Ben Poole},
    title = {SCORE-BASED GENERATIVE MODELING THROUGH
STOCHASTIC DIFFERENTIAL EQUATIONS},
    journal = {In International Conference on Learning Representations (ICLR)},
    year = {2021},
}

@article{denoisingSM,
    author = {Pascal Vincent},
    title = {A connection between score matching and denoising autoencoders},
    journal = {Neural computation, 23(7):1661–1674},
    year = {2011}
}

@article{RT1996,
    author = {Gareth O. Roberts and Richard L. Tweedie},
    title = {Exponential Convergence of {Langevin} Distributions and Their Discrete Approximations},
    journal = {Bernoulli, Vol. 2, No. 4. (Dec., 1996), pp. 341-363},
    year = {1996}
}

@article{MENOZZI2021,
title = {Density and gradient estimates for non degenerate {Brownian SDEs} with unbounded measurable drift},
journal = {Journal of Differential Equations},
volume = {272},
pages = {330-369},
year = {2021},
issn = {0022-0396},
doi = {https://doi.org/10.1016/j.jde.2020.09.004},
author = {S. Menozzi and A. Pesce and X. Zhang},
}

@article{uniqueness_figalli2008,
title = {Existence and uniqueness of martingale solutions for {SDEs} with rough or degenerate coefficients},
journal = {Journal of Functional Analysis},
volume = {254},
number = {1},
pages = {109-153},
year = {2008},
issn = {0022-1236},
doi = {https://doi.org/10.1016/j.jfa.2007.09.020},
author = {Alessio Figalli}
}

@misc{ramesh2022hierarchicaltextconditionalimagegeneration,
      title={Hierarchical Text-Conditional Image Generation with CLIP Latents}, 
      author={Aditya Ramesh and Prafulla Dhariwal and Alex Nichol and Casey Chu and Mark Chen},
      year={2022},
      eprint={2204.06125},
      archivePrefix={arXiv},
      primaryClass={cs.CV},
}

@misc{videoworldsimulators2024,
  title={Video generation models as world simulators},
  author={Tim Brooks and Bill Peebles and Connor Holmes and Will DePue and Yufei Guo and Li Jing and David Schnurr and Joe Taylor and Troy Luhman and Eric Luhman and Clarence Ng and Ricky Wang and Aditya Ramesh},
  year={2024},
  url={https://openai.com/research/video-generation-models-as-world-simulators}
}

@misc{zou2023surveydiffusionmodelsnatural,
      title={A Survey of Diffusion Models in Natural Language Processing}, 
      author={Hao Zou and Zae Myung Kim and Dongyeop Kang},
      year={2023},
      eprint={2305.14671},
      archivePrefix={arXiv},
      primaryClass={cs.CL},
}

@misc{zhang2023surveyaudiodiffusionmodels,
      title={A Survey on Audio Diffusion Models: Text To Speech Synthesis and Enhancement in Generative AI}, 
      author={Chenshuang Zhang and Chaoning Zhang and Sheng Zheng and Mengchun Zhang and Maryam Qamar and Sung-Ho Bae and In So Kweon},
      year={2023},
      eprint={2303.13336},
      archivePrefix={arXiv},
      primaryClass={cs.SD},
}

@article{kazerouni2023diffusionmodelsmedicalimage,
title = {Diffusion models in medical imaging: A comprehensive survey},
journal = {Medical Image Analysis},
volume = {88},
pages = {102846},
year = {2023},
issn = {1361-8415},
doi = {https://doi.org/10.1016/j.media.2023.102846},
author = {Amirhossein Kazerouni and Ehsan Khodapanah Aghdam and Moein Heidari and Reza Azad and Mohsen Fayyaz and Ilker Hacihaliloglu and Dorit Merhof},
}

@article{lin2023diffusionmodelstimeseries,
title = {Empowering Time Series Analysis with Foundation Models: A Comprehensive Survey},
journal = {Information Fusion},
pages = {104601},
year = {2026},
issn = {1566-2535},
doi = {https://doi.org/10.1016/j.inffus.2026.104601},
author = {Jiexia Ye and Yongzi Yu and Weiqi Zhang and Le Wang and Jia Li and Fugee Tsung},
}

@article{bioinformatics,
	author = {Guo, Zhiye and Liu, Jian and Wang, Yanli and Chen, Mengrui and Wang, Duolin and Xu, Dong and Cheng, Jianlin},
	date = {2024/02/01},
	doi = {10.1038/s44222-023-00114-9},
	id = {Guo2024},
	isbn = {2731-6092},
	journal = {Nature Reviews Bioengineering},
	number = {2},
	pages = {136--154},
	title = {Diffusion models in bioinformatics and computational biology},
	volume = {2},
	year = {2024},
}

@misc{yang2024diffusionmodelscomprehensivesurvey,
      title={Diffusion Models: A Comprehensive Survey of Methods and Applications}, 
      author={Ling Yang and Zhilong Zhang and Yang Song and Shenda Hong and Runsheng Xu and Yue Zhao and Wentao Zhang and Bin Cui and Ming-Hsuan Yang},
      year={2024},
      eprint={2209.00796},
      archivePrefix={arXiv},
      primaryClass={cs.LG},
}

@article{gencast,
  author  = {Price, Ilan and Sanchez-Gonzalez, Alvaro and Alet, Ferran and Andersson, Tom R. and El-Kadi, Andrew and Masters, Dominic and Ewalds, Timo and Stott, Jacklynn and Mohamed, Shakir and Battaglia, Peter and Lam, Remi and Willson, Matthew},
  title   = {Probabilistic weather forecasting with machine learning},
  journal = {Nature},
  year    = {2025},
  volume  = {637},
  number  = {8044},
  pages   = {84--90},
  doi     = {10.1038/s41586-024-08252-9},
}

@article{frontgeologico,
title = {Controlled latent diffusion models for {3D} porous media reconstruction},
journal = {Computers \& Geosciences},
volume = {206},
pages = {106038},
year = {2026},
issn = {0098-3004},
doi = {https://doi.org/10.1016/j.cageo.2025.106038},
author = {Danilo Naiff and Bernardo P. Schaeffer and Gustavo Pires and Dragan Stojkovic and Thomas Rapstine and Fabio Ramos},
}

@article{alphafold_3,
author = {Abramson, Josh and Adler, Jonas and Dunger, Jack and Evans, Richard and Green, Tim and Pritzel, Alexander and Ronneberger, Olaf and Willmore, Lindsay and Ballard, Andrew and Bambrick, Joshua and Bodenstein, Sebastian and Evans, David and Hung, Chia-Chun and O’Neill, Michael and Reiman, David and Tunyasuvunakool, Kathryn and Wu, Cervantes and Žemgulytė, Akvilė and Arvaniti, Eirini and Jumper, John},
year = {2024},
month = {05},
pages = {493-500},
title = {Accurate structure prediction of biomolecular interactions with {AlphaFold} 3},
volume = {630},
journal = {Nature},
doi = {10.1038/s41586-024-07487-w}
}

@misc{li2025thermodynamicsproteindesigndiffusion,
      title={From thermodynamics to protein design: Diffusion models for biomolecule generation towards autonomous protein engineering}, 
      author={{Wen\-ran} Li and Xavier F. Cadet and David Medina-Ortiz and Mehdi D. Davari and Ramanathan Sowdhamini and Cedric Damour and Yu Li and Alain Miranville and Frederic Cadet},
      year={2025},
      eprint={2501.02680},
      archivePrefix={arXiv},
      primaryClass={q-bio.QM},
}

@article{conforti2025klconvergence,
author = {Conforti, Giovanni and Durmus, Alain and Silveri, Marta Gentiloni},
title = {KL Convergence Guarantees for Score Diffusion Models under Minimal Data Assumptions},
journal = {SIAM Journal on Mathematics of Data Science},
volume = {7},
number = {1},
pages = {86-109},
year = {2025},
doi = {10.1137/23M1613670},
eprint = {https://doi.org/10.1137/},
}

@inproceedings{
benton2024nearly,
title={Nearly \$d\$-Linear Convergence Bounds for Diffusion Models via Stochastic Localization},
author={Joe Benton and Valentin De Bortoli and Arnaud Doucet and George Deligiannidis},
booktitle={The Twelfth International Conference on Learning Representations},
year={2024},
url={https://openreview.net/forum?id=r5njV3BsuD}
}

@inproceedings{Chen2022Improved,
  title={Improved Analysis of Score-based Generative Modeling: User-Friendly Bounds under Minimal Smoothness Assumptions},
  author={Hongrui Chen and Holden Lee and Jianfeng Lu},
  booktitle={International Conference on Machine Learning},
  year={2022},
}

@article{stoch_interp,
  author  = {Michael Albergo and Nicholas M. Boffi and Eric Vanden-Eijnden},
  title   = {Stochastic Interpolants: A Unifying Framework for Flows and Diffusions},
  journal = {Journal of Machine Learning Research},
  year    = {2025},
  volume  = {26},
  number  = {209},
  pages   = {1--80},
}

@misc{song2021maxlikelihood,
      title={Maximum Likelihood Training of Score-Based Diffusion Models}, 
      author={Yang Song and Conor Durkan and Iain Murray and Stefano Ermon},
      year={2021},
      eprint={2101.09258},
      archivePrefix={arXiv},
      primaryClass={stat.ML},
}

@inproceedings{
song2021implicitmodels,
title={Denoising Diffusion Implicit Models},
author={Jiaming Song and Chenlin Meng and Stefano Ermon},
booktitle={International Conference on Learning Representations},
year={2021}
}

@book{karatzas,
    author = {Ioannis Karatzas and Steven E. Shreve},
    title = {Brownian Motion and Stochastic Calculus},
    publisher = {Springer-Verlag New York},
    year = {1988}
}

@inproceedings{precision/recall_2018,
 author = {Sajjadi, Mehdi S. M. and Bachem, Olivier and Lucic, Mario and Bousquet, Olivier and Gelly, Sylvain},
 booktitle = {Advances in Neural Information Processing Systems},
 editor = {S. Bengio and H. Wallach and H. Larochelle and K. Grauman and N. Cesa-Bianchi and R. Garnett},
 pages = {},
 publisher = {Curran Associates, Inc.},
 title = {Assessing Generative Models via Precision and Recall},
 volume = {31},
 year = {2018}
}

@inproceedings{feature_lik_divergence_2023,
author = {Jiralerspong, Marco and Bose, Avishek (Joey) and Gemp, Ian and Qin, Chongli and Bachrach, Yoram and Gidel, Gauthier},
title = {Feature likelihood divergence: evaluating the generalization of generative models using samples},
year = {2023},
publisher = {Curran Associates Inc.},
address = {Red Hook, NY, USA},
booktitle = {Proceedings of the 37th International Conference on Neural Information Processing Systems},
articleno = {1436},
numpages = {25},
location = {New Orleans, LA, USA},
series = {NIPS '23}
}

@misc{steph2025regularity,
      title={Regularity of the score function in generative models}, 
      author={Arthur Stéphanovitch},
      year={2025},
      eprint={2506.19559},
      archivePrefix={arXiv},
      primaryClass={math.ST},
      url={https://arxiv.org/abs/2506.19559}, 
}

@article{BOBKOV1999,
title = {Exponential Integrability and Transportation Cost Related to Logarithmic {Sobolev} Inequalities},
journal = {Journal of Functional Analysis},
volume = {163},
number = {1},
pages = {1-28},
year = {1999},
issn = {0022-1236},
doi = {https://doi.org/10.1006/jfan.1998.3326},
author = {S.G Bobkov and F Götze}
}

@InProceedings{bakry-emmery,
author="Bakry, D.
and {\'E}mery, M.",
editor="Az{\'e}ma, Jacques
and Yor, Marc",
title="Diffusions hypercontractives",
booktitle="S{\'e}minaire de Probabilit{\'e}s XIX 1983/84",
year="1985",
publisher="Springer Berlin Heidelberg",
address="Berlin, Heidelberg",
pages="177--206",
isbn="978-3-540-39397-9"
}

@article{holley-stroock,
    author = {Holley, Richard and Stroock, Daniel},
    title = {Logarithmic {Sobolev} inequalities and stochastic {Ising} models},
    journal = {Journal of Statistical Physics},
    volume = {46},
    pages = {1159-1194},
    year = {1987},
    doi = {https://doi.org/10.1007/BF01011161}
}

@book{Murphy2012MachineL,
  title={Machine learning - a probabilistic perspective},
  author={Kevin P. Murphy},
  publisher={MIT Press},
  year={2012}
}

@book{bishop,
  title={Pattern recognition and machine learning},
  author={Christopher M. Bishop},
  publisher={Springer},
  year={2006}
}

@inproceedings{aithal,
title={Understanding Hallucinations in Diffusion Models through Mode Interpolation},
author={Sumukh K Aithal and Pratyush Maini and Zachary Chase Lipton and J Zico Kolter},
booktitle={The Thirty-eighth Annual Conference on Neural Information Processing Systems},
year={2024}
}

@inproceedings{
SDEDrag,
title={The Blessing of Randomness: {SDE} Beats {ODE} in General Diffusion-based Image Editing},
author={Shen Nie and Hanzhong Allan Guo and Cheng Lu and Yuhao Zhou and Chenyu Zheng and Chongxuan Li},
booktitle={The Twelfth International Conference on Learning Representations},
year={2024}
}

@inproceedings{
opt_choice,
title={Exploring the Optimal Choice for Generative Processes in Diffusion Models: Ordinary vs Stochastic Differential Equations},
author={Yu Cao and Jingrun Chen and Yixin Luo and Xiang Zhou},
booktitle={Thirty-seventh Conference on Neural Information Processing Systems},
year={2023}
}

@article{boffi-vanden,
    author = {Boffi, Nicholas M. and
Vanden-Eijnden, Eric},
    title = {Probability flow solution of the {Fokker–Planck} equation},
    journal = {Machine Learning: Science and
Technology},
    volume = {4},
    issue = {3},
    pages = {035012},
    year = {2023},
    doi = {https://doi.org/10.1088/2632-2153/ace2aa}
}

@inproceedings{
shen2026manifold,
title={Manifold Generalization Provably Proceeds Memorization in Diffusion Models},
author={Zebang Shen and Ya-Ping Hsieh and Niao He},
booktitle={ICLR 2026 Workshop on Geometry-grounded Representation Learning and Generative Modeling},
year={2026}
}

@misc{jolicoeurmartineau2021gottafastgeneratingdata,
      title={Gotta Go Fast When Generating Data with Score-Based Models}, 
      author={Alexia Jolicoeur-Martineau and Ke Li and Rémi Piché-Taillefer and Tal Kachman and Ioannis Mitliagkas},
      year={2021},
      eprint={2105.14080},
      archivePrefix={arXiv},
      primaryClass={cs.LG},
}

@inproceedings{
bao2022analyticdpm,
title={Analytic-{DPM}: an Analytic Estimate of the Optimal Reverse Variance in Diffusion Probabilistic Models},
author={Fan Bao and Chongxuan Li and Jun Zhu and Bo Zhang},
booktitle={International Conference on Learning Representations},
year={2022}
}

@inproceedings{
xu2023restart,
title={Restart Sampling for Improving Generative Processes},
author={Yilun Xu and Mingyang Deng and Xiang Cheng and Yonglong Tian and Ziming Liu and Tommi S. Jaakkola},
booktitle={Thirty-seventh Conference on Neural Information Processing Systems},
year={2023}
}

@inproceedings{
xue2023sasolver,
title={{SA}-Solver: Stochastic Adams Solver for Fast Sampling of Diffusion Models},
author={Shuchen Xue and Mingyang Yi and Weijian Luo and Shifeng Zhang and Jiacheng Sun and Zhenguo Li and Zhi-Ming Ma},
booktitle={Thirty-seventh Conference on Neural Information Processing Systems},
year={2023}
}

@inproceedings{
ma2024,
title={{SiT}: Exploring flow and diffusion-based generative models with scalable inter-
polant transformers},
author={Nanye Ma and Mark Goldstein and Michael S. Albergo and Nicholas M. Boffi and Eric Vanden-Eijnden and Saining Xie},  
booktitle={European Conference on Computer Vision (ECCV), Springer}, 
year={2024}
}

@inproceedings{
chen2024follmer,
title={Probabilistic forecasting with stochastic interpolants and {F\"ollmer} processes},
author={Yifan Chen and Mark Goldstein and Mengjian Hua and Michael S. Albergo and Nicholas M. Boffi and Eric
Vanden-Eijnden}, 
booktitle={Proceedings of the 41st International Conference on Machine Learning (ICML)},
volume={235},
pages={6728–6756}, 
year={2024}
}

@inproceedings{Srivastava2017Veegan,
author = {Srivastava, Akash and Valkov, Lazar and Russell, Chris and Gutmann, Michael U. and Sutton, Charles},
title = {VEEGAN: reducing mode collapse in {GANs} using implicit variational learning},
year = {2017},
isbn = {9781510860964},
publisher = {Curran Associates Inc.},
address = {Red Hook, NY, USA},
booktitle = {Proceedings of the 31st International Conference on Neural Information Processing Systems},
pages = {3310–3320},
numpages = {11},
location = {Long Beach, California, USA},
series = {NIPS'17}
}

@article{Shumailov2024,
  author = {Shumailov, Ilia and Shumaylov, Zakhar and Zhao, Yiren and Papernot, Nicolas and Anderson, Ross and Gal, Yarin},
  year = {2024},
  title = {{AI models collapse when trained on recursively generated data}},
  journal = {Nature},
  volume = {631},
  number = {8022},
  pages = {755--759},
  doi = {10.1038/s41586-024-07566-y}
}

@article{mnist_lecun1998,
  title={Gradient-based learning applied to document recognition},
  author={LeCun, Yann and Bottou, L{\'e}on and Bengio, Yoshua and Haffner, Patrick},
  journal={Proceedings of the IEEE},
  volume={86},
  number={11},
  pages={2278--2324},
  year={1998},
  publisher={Ieee}
}

@techreport{cifar-10_krizhevsky2009,
  title={Learning multiple layers of features from tiny images},
  author={Krizhevsky, Alex},
  year={2009},
  institution={University of Toronto}
}

\appendix

\section{Additional proofs}

We include, in this appendix, the details of a number of results mentioned in the main part of the paper.

\subsection{Section 2}\label{app: proofs sec2}

We start with the proofs of the results mentioned in Section 2, which are stated in slightly different forms than what is found in the literature.

\begin{proof}[of Lemma \ref{lemma rev time}]
    We will show that $\bar{p}_\tau$ and $\tilde{p}^\gamma_\tau$ solve the same Fokker-Planck equation, which admits a unique solution for each initial condition. The FP equation of the forward SDE is given by
    \begin{align}
        \frac{\partial p}{\partial t} &= -\nabla\cdot(fp) +\frac{1}{2}g^2\Delta p \nonumber
        \\ & = -\nabla\cdot(fp) +\frac{1}{2}g^2((1+\gamma)\Delta p -\gamma\Delta p) \nonumber
        \\ & = -\nabla\cdot(fp) +\frac{1}{2}g^2((1+\gamma)(\nabla\cdot\nabla p) -\gamma\Delta p) \nonumber
        \\ & = -\nabla\cdot(fp) +\frac{1}{2}g^2((1+\gamma)(\nabla\cdot(p\nabla\log p) -\gamma\Delta p) \nonumber \qquad\qquad\text{since $\nabla\log p = \frac{1}{p}\nabla p$}
        \\ & = \nabla\cdot\left(\left(-f+\frac{1}{2}g^2(1+\gamma)\nabla\log p\right)p\right)  -\frac{1}{2}g^2\gamma\Delta p,\label{rev fp}
    \end{align}
    for an arbitrary function $\gamma=\gamma(t)$. Under a change of variables $\tau=T-t$, this is equivalent to
    \begin{equation}
        \frac{\partial p}{\partial \tau}= -\nabla\cdot\left(\left(-f+\frac{1}{2}g^2(1+\gamma)\nabla\log p\right)p\right)  +\frac{1}{2}g^2\gamma\Delta p,
    \end{equation}
    which is the FP equation for the SDE
    \begin{equation}
        dY_\tau = \left(-\bar{f}(Y_\tau, \tau) + \frac{1}{2}\bar{g}^2(\tau)(1+\gamma(\tau))\nabla\log \bar{p}_\tau(Y_\tau)\right)\,d\tau + \sqrt{\gamma(\tau)} \bar{g}(\tau) \,dW_\tau
    \end{equation}
    in reverse-time $\tau$.
    Since $\bar{p}$ solves the forward FP equation \eqref{fp}, the computation above shows that it also solves this reverse-time FP equation; and $\tilde{p}^\gamma$ solves it by definition. By the hypothesis on $\nabla\log\bar{p}$ and the Lipschitz continuity of $f$, the drift of \eqref{rev sde 2} is Lipschitz in the space variable, so this FP equation admits a unique solution for each initial condition \citep{uniqueness_figalli2008}. Therefore, from $\tilde{p}^\gamma_0 = \bar{p}_0$ we conclude that $\tilde{p}^\gamma_\tau = \bar{p}_\tau$ for all $\tau\in[0,T)$.
\end{proof}

\bigskip

\begin{proof}[of Lemma \ref{lemma LSI}]
    Since $\tilde{p}\ll p$, we can write $\tilde{p}=f^2\;p$. Then, we have
\begin{align*}
    \text{Ent}_p(f^2) = \text{Ent}_p\left(\frac{\tilde{p}}{p}\right) & = \int \frac{\tilde{p}}{p}\log \frac{\tilde{p}}{p}\;p\,dx - \left(\int\frac{\tilde{p}}{p}\;p\,dx\right)\log \left(\int\frac{\tilde{p}}{p}\;p\,dx\right)
    \\ & = \int \log\frac{\tilde{p}}{p}\tilde{p} - 1\log 1= H(\tilde{p}|p)
\end{align*}
and
\begin{align*}
    \int \left\|\nabla\log\frac{\tilde{p}}{p}\right\|^2\tilde{p}\,dx = \int \left\|\nabla\log f^2\right\|^2f^2\;p\,dx & = \int \left\|\frac{1}{f^2}2f \nabla f\right\|^2f^2\;p\,dx 
    \\ & = 4\int \left\|\nabla f\right\|^2\;p\,dx.
\end{align*}
Therefore, if the LSI holds for $p$, the second inequality also holds for all such $\tilde{p}$.
\end{proof}

\subsection{Section 3} \label{app: proofs sec3}

The proofs of Propositions \ref{prop 1} and \ref{prop 2} use the following lemma, whose formula appears in
\cite[Lemma C.1]{Chen2022Improved}. We restate it at the level of solutions
of the Fokker-Planck equations, with explicit conditions under which the
formula holds.

\begin{lemma}[Derivative of the KL divergence]\label{lemma H_deriv}
    Let $b \in C((0,T))$ and let $a_1, a_2 : \mathbb{R}^n \times (0,T) \to \mathbb{R}^n$ be continuous and $C^1$ in $x$. Let $p, q \in C^{2,1}(\mathbb{R}^n \times (0,T))$ be positive probability densities solving
    \[
     \partial_t p = -\nabla\cdot(a_1\, p) + \tfrac12 b^2 \Delta p, \qquad \partial_t q = -\nabla\cdot(a_2\, q) + \tfrac12 b^2 \Delta q .
    \]
    If the densities $p_t, q_t$ are bounded above and below by Gaussians, locally uniformly in $t \in (0,T)$, and $a_1$, $a_2$, $\nabla \log p_t$, $\nabla \log q_t$ have at most polynomial growth in $x$, then, for every $t \in (0,T)$, the evolution of the KL divergence $H(p_t|q_t)$ is given by
    \begin{equation}
        \frac{d}{dt}H(p_t|q_t)
        = -\frac{1}{2}b^2(t)
          \int\left|\nabla\log\frac{p_t}{q_t}\right|^2 p_t\,dx
        + \int(a_1 - a_2)\cdot
          \left(\nabla\log\frac{p_t}{q_t}\right)p_t\,dx.
    \end{equation}
\end{lemma}

\bigskip

\begin{proof}
    Fix $[t_1, t_2] \subset (0,T)$. By the hypotheses, there are constants $C_0, c_0, \tilde C_0, \tilde c_0, C_1 > 0$ and $m \geq 1$ such that, on $\mathbb{R}^n \times [t_1,t_2]$,
    \begin{equation}\label{lemma2 standing bounds}
        c_0\, e^{-\tilde C_0|x|^2} \leq p,\, q \leq C_0\, e^{-\tilde c_0|x|^2}, \qquad |a_1| + |a_2| + |\nabla\log p| + |\nabla\log q| \leq C_1 (1+|x|)^m.
    \end{equation}
    In particular, $|\log(p/q)| \leq C(1+|x|^2)$, and, writing the fluxes $J_p := a_1 p - \frac12 b^2 \nabla p$ and $J_q := a_2 q - \frac12 b^2 \nabla q$, so that the equations read $\partial_t p = -\nabla\cdot J_p$ and $\partial_t q = -\nabla\cdot J_q$, we have, using $\nabla p = p\,\nabla\log p$ and $\nabla q = q\,\nabla\log q$,
    \begin{equation}\label{lemma2 flux bounds}
        |J_p| + \frac{p}{q}\,|J_q| \leq C (1+|x|)^m e^{-c_0|x|^2}.
    \end{equation}
    Hence all the integrands appearing below are bounded, uniformly in $t \in [t_1,t_2]$, by the integrable function $\Phi(x) := C(1+|x|)^{2m+2} e^{-c_0|x|^2}$.
 
    Let $\chi \in C_c^\infty(\mathbb{R}^n)$ with $0 \leq \chi \leq 1$, $\chi \equiv 1$ on $B_1$, $\supp\chi \subset B_2$ and $|\nabla\chi| \leq 2$, where $B_R:=\{x\in\mathbb{R}^n; |x|< R\}$ is the ball of radius $R$ centered in the origin. Set $\chi_R := \chi(\cdot/R)$, and consider the localized KL divergence
    \[
        H_R(t) := \int \chi_R\, p \log\frac{p}{q}\,dx.
    \]
    Since $p, q \in C^{2,1}$ are positive, the integrand is continuously differentiable in $t$ and supported in $B_{2R}$, so we may differentiate under the integral sign:
    \[
        H_R'(t) = \int \chi_R \left[\partial_t p \left(\log\frac{p}{q} + 1\right) - \frac{p}{q}\,\partial_t q\right] dx.
    \]
    Substituting the equations and integrating by parts, there are no boundary terms since $\chi_R$ is compactly supported. Then, using $\nabla(p/q) = (p/q)\nabla\log(p/q)$ together with
    \[
        J_p - \frac{p}{q}J_q = (a_1 - a_2)p - \frac12 b^2\left(\nabla p - p\,\nabla\log q\right) = (a_1 - a_2)p - \frac12 b^2\, p\,\nabla\log\frac{p}{q},
    \]
    we obtain
    \begin{equation}\label{lemma2 HR prime}
        H_R'(t) = \int \chi_R\, (a_1 - a_2)\cdot\nabla\log\frac{p}{q}\;p\,dx - \frac{b^2}{2}\int \chi_R \left|\nabla\log\frac{p}{q}\right|^2 p\,dx + E_R(t),
    \end{equation}
    with the cutoff error
    \[
        E_R(t) := \int \nabla\chi_R \cdot \left[\left(\log\frac{p}{q}+1\right)J_p - \frac{p}{q}\,J_q\right] dx.
    \]
    Since $|\nabla\chi_R| \leq 2/R \leq 2$ is supported in $\{R \leq |x| \leq 2R\}$ and the bracket is bounded by $\Phi$ by \eqref{lemma2 standing bounds}--\eqref{lemma2 flux bounds}, we have $|E_R(t)| \leq 2\int_{|x|\geq R}\Phi\,dx \to 0$ uniformly in $t$. Integrating \eqref{lemma2 HR prime} over $[t_1, t_2]$ and letting $R \to \infty$, with dominated convergence in $x$ (as $\chi_R \uparrow 1$) and in $t$ (all terms bounded by $\int\Phi$), and using $H_R(t_i) \to H(t_i)$ (the KL divergence integrand is dominated by $C(1+|x|^2)e^{-c_0|x|^2}$), we obtain
    \[
        H(p_{t_2}|q_{t_2}) - H(p_{t_1}|q_{t_1}) = \int_{t_1}^{t_2}\left(-\frac{b^2}{2}\int\left|\nabla\log\frac{p}{q}\right|^2 p\,dx + \int (a_1-a_2)\cdot\nabla\log\frac{p}{q}\;p\,dx\right) dt.
    \]
    Finally, the time integrand is continuous on $(0,T)$ (by dominated convergence, using the continuity of $p, \nabla p, q, \nabla q$ and of $a_1, a_2, b$ in $t$, with the domination by $\Phi$ on compact subintervals), so the fundamental theorem of calculus yields the derivative formula at every $t \in (0,T)$.
\end{proof}

The verification of the hypotheses of Propositions \ref{prop 1} and \ref{prop 2} for the density of the (approximate) reverse-time SDE is discussed in Remark \ref{remark: hypotheses scope} and deferred to future work. In the remainder of this section, we collect the verifications used for $\bar p$ in the linear-drift case.

\bigskip

\begin{proof}[of the two-sided Gaussian bounds for $\bar p$ in Theorem \ref{teo 2}]
    Since $p_t$ can be written as the Gaussian convolution $p_t=p_0^t\ast \mathcal{N}(0, s(t)^2\sigma(t)^2I)$, with $p_0^t(x):=s(t)^{-n}p_0(x/s(t))$, we have
    \begin{align*}
        p_t(x) & = C(t)\int p_0^t(y) e^{-c(t)|x-y|^2} dy
        \\ &  \geq C(t)\int p_0^t(y) e^{-2c(t)(|x|^2+|y|^2)} dy 
        \\ & = C(t)e^{-2c(t)|x|^2}\int p_0^t(y)e^{-2c(t)|y|^2} dy = C_2(t)e^{-2c(t)|x|^2}
    \end{align*}
    for $c(t), C(t)$ constants in $x$ and $C_2(t)=C(t)\int p_0^t(y)e^{-2c(t)|y|^2} dy<\infty$ since $p_0^t(x)$ is a probability density and $e^{-2c(t)|y|^2}$ is bounded. For the upper bound, since $p_0$ has sub-Gaussian tails, we have that
    \begin{align*}
        p_t(x) & = C(t)\int p_0^t(y) e^{-c(t)|x-y|^2} dy 
        \\ & \le C(t)\int C_3(t)e^{-c_3(t)|y|^2} e^{-c(t)|x-y|^2} dy = C_4(t)\int e^{-c_3(t)|y|^2-c(t)|x-y|^2} dy.
    \end{align*}
    For any $0<\epsilon<1$,
    \begin{align*}
        c_3|y|^2 +c|x-y|^2 & \geq \epsilon c (|x|^2 - 2\langle x, y\rangle + |y|^2) +c_3|y|^2
        \\ & \geq \epsilon c\left(|x|^2 - \frac{1}{2}|x|^2 -2|y|^2 + |y|^2\right) +c_3|y|^2
        \\ & = \epsilon c\left(\frac{1}{2}|x|^2 -|y|^2\right) +c_3|y|^2 = \epsilon c\frac{1}{2}|x|^2 +(c_3-\epsilon c)|y|^2.
    \end{align*}
    Taking $\epsilon<\min\left\{1, c_3/c\right\}$, we have $c_6:=c_3-\epsilon c>0$, and thus
    \begin{align*}
        p_t(x) \leq C_4 \int e^{-c_5|x|^2 -c_6 |y|^2}dy
        = C e^{-c|x|^2}
    \end{align*}
    for some constants $C, c>0$.
\end{proof}

\begin{proof}[of Corollary \ref{coro 1}]
    To show that $\nabla\log p_t$ is Lipschitz in this case, denote the Gaussian pdf by $$g_t(x):=\frac{1}{(2\pi t^2)^{\frac{n}{2}}}e^{-\frac{|x|^2}{2t^2}},$$ where $x\in\mathbb{R}^n$, and note that
    \begin{equation*}
        \nabla\log p_t(x) = \frac{\nabla (p_0 \ast g_t)(x)}{p_t(x)} = \frac{p_0 \ast\nabla g_t(x)}{p_t(x)} = \frac{\int p_0(y)\frac{y-x}{t^2}g_t(x-y)\,dy}{p_t(x)} = \frac{1}{t^2}\left(-x + \frac{m_t(x)}{p_t(x)}\right),
    \end{equation*}
    where $m_t(x):= \int p_0(y)yg_t(x-y)\,dy$. 

    To see that $m_t(x)/p_t(x)$ is Lipschitz, note that $\nabla\left(m_t(x)/p_t(x)\right)$ is bounded, since
    \begin{align*}
        \nabla\left(\frac{m_t(x)}{p_t(x)}\right) & = \frac{\nabla m_t(x)}{p_t(x)} - m_t(x)\frac{\nabla p_t(x)}{p_t(x)^2}
        \\ & = \frac{\frac{1}{t^2}\int p_0(y)y(y-x)g_t(x-y)\,dy}{p_t(x)} - m_t(x)\frac{\frac{1}{t^2}\int p_0(y)(y-x)g_t(x-y)\,dy}{p_t(x)^2}
        \\ & = \frac{1}{t^2}\left(\frac{\int p_0(y)y^2g_t(x-y)\,dy}{p_t(x)} -\frac{x m_t(x)}{p_t(x)} -\frac{m_t(x)^2}{p_t(x)^2} +m_t(x)\frac{x}{p_t(x)}\right)
        \\ & = \frac{1}{t^2}\left(\frac{\int p_0(y)y^2g_t(x-y)\,dy}{p_t(x)} -\frac{m_t(x)^2}{p_t(x)^2}\right)
        \\ & \leq \frac{1}{t^2}\left(\frac{\int p_0(y)R^2g_t(x-y)\,dy}{p_t(x)} +\left(\frac{\int p_0(y)Rg_t(x-y)\,dy}{p_t(x)}\right)^2\right) = \frac{2R^2}{t^2},
    \end{align*}
    using that $\supp{p_0}\subset B_R(0)$.
\end{proof}

\subsection{Section 4} \label{app: proofs sec4}

As discussed in the last section (proof of Theorem \ref{teo 2}), $p_t$ can be written as the Gaussian convolution $p_t=p_0^t\ast \mathcal{N}(0, s(t)^2\sigma(t)^2I)$, where $p_0^t(x):=s(t)^{-n}p_0(x/s(t))$. If $p_0$ is the mixture of Gaussians $p_0=\sum_{l=1}^n w_l\mathcal{N}(\mu_l, \sigma_l^2I)$, we have $p_0^t=\sum_{l=1}^n w_l\mathcal{N}(s(t)\mu_l, s(t)^2\sigma_l^2I)$ and thus
\begin{align*}
    p_t & = p_0^t\ast \mathcal{N}(0, s(t)^2\sigma(t)^2I)
    \\ & = \sum_{l=1}^n w_l\mathcal{N}(s(t)\mu_l, s(t)^2\sigma_l^2I) \ast \mathcal{N}(0, s(t)^2\sigma(t)^2I)
    \\ & = \sum_{l=1}^n w_l\mathcal{N}(s(t)\mu_l, s(t)^2(\sigma_l^2+ \sigma(t)^2)I),
\end{align*}
is also a mixture of spherical Gaussians. Therefore, the following lemma ensures that $\nabla\log p_t$ is Lipschitz for all $t\in[0,T]$.

\begin{lemma}
    The score function of a mixture of finitely many spherical Gaussians is globally Lipschitz.
\end{lemma}

\begin{proof}
    Let $p$ be a mixture of $M$ spherical Gaussians in $\mathbb{R}^n$ given by
    \begin{equation}
        p(x)=\sum_{l=1}^M w_l\mathcal{N}(x; \mu_l, \sigma_l^2I),
    \end{equation}
    where $\mathcal{N}(x; \mu, \sigma^2I)$ is the Gaussian pdf with mean $\mu$ and covariance matrix $\sigma^2I$, and $\sum w_l=1$ with $w_l>0$. Then the score function is given by
    \begin{equation}
        \nabla\log p(x)= \frac{\nabla p(x)}{p(x)} = \frac{\sum_{l=1}^M w_l \left(\frac{\mu_l-x}{\sigma_l^2}\right)\mathcal{N}(x; \mu_l, \sigma_l^2I)}{\sum_{l=1}^M w_l\mathcal{N}(x; \mu_l, \sigma_l^2I)},
    \end{equation}
    and, writing $\mathcal{N}_l := \mathcal{N}(x; \mu_l,\sigma_l^2I)$, we have
    \begin{align*}
        \partial^2_{x_i x_j}(\log p(x)) = & \partial_{x_j}\left(\frac{\partial_{x_i} p(x)}{p(x)}\right)= \frac{p(x)\partial^2_{x_i x_j} p(x) - \partial_{x_i} p(x)\partial_{x_j}p(x)}{p(x)^2}
        \\ = &\frac{\left(\sum_{l=1}^M w_l \left(\frac{(\mu_l^{(i)}-x_i)(\mu_l^{(j)}-x_j)}{\sigma_l^4}-\frac{\delta_{ij}}{\sigma_l^2}\right)\mathcal{N}_l\right)\sum_{l=1}^M w_l\mathcal{N}_l}{\left(\sum_{l=1}^M w_l\mathcal{N}_l\right)^2}
        \\ & - \frac{\left(\sum_{l=1}^M w_l\left(\frac{\mu_l^{(i)}-x_i}{\sigma_l^2}\right)\mathcal{N}_l\right)\left(\sum_{l=1}^M w_l\left(\frac{\mu_l^{(j)}-x_j}{\sigma_l^2}\right)\mathcal{N}_l\right)}{\left(\sum_{l=1}^M w_l\mathcal{N}_l\right)^2}
        \\ = & \frac{\sum_{l,k} w_l w_k \mathcal{N}_l \mathcal{N}_k \left(\frac{(\mu_l^{(i)}-x_i)(\mu_l^{(j)}-x_j)}{\sigma_l^4}-\frac{\delta_{ij}}{\sigma_l^2}\right)}{\sum_{l,k} w_l w_k \mathcal{N}_l \mathcal{N}_k}
        \\ & - \frac{\sum_{l,k} w_l w_k \mathcal{N}_l \mathcal{N}_k\left( \frac{(\mu_l^{(i)}-x_i)(\mu_k^{(j)}-x_j)}{\sigma_l^2 \sigma_k^2}\right)}{\sum_{l,k} w_l w_k \mathcal{N}_l \mathcal{N}_k}.
    \end{align*}
    Thus
    \begin{align*}
        \partial^2_{x_i x_j}(\log p(x)) &= \sum_{l,k}\frac{ w_l w_k \mathcal{N}_l \mathcal{N}_k }{\sum_{l',k'} w_{l'} w_{k'} \mathcal{N}_{l'} \mathcal{N}_{k'}}\left(\frac{(\mu_l^{(i)}-x_i)(\mu_l^{(j)}-x_j)}{\sigma_l^4} - \frac{(\mu_l^{(i)}-x_i)(\mu_k^{(j)}-x_j)}{\sigma_l^2\sigma_k^2} - \frac{\delta_{ij}}{\sigma_l^2}\right) 
        \\ &= : \sum_{l,k} A_{lk} (x)
    \end{align*}
    To show that this function is bounded, consider the indices $l_m$ such that $\sigma_{l_m}=\sigma_*$ is the largest variance. For all $m$, we have
    \begin{equation}\label{A_lk}
        |A_{l, k}| \leq \frac{ w_l w_k \mathcal{N}_l \mathcal{N}_k }{w_{l_m}^2\mathcal{N}_{l_m}^2}\left|\frac{(\mu_l^{(i)}-x_i)(\mu_l^{(j)}-x_j)}{\sigma_l^4} - \frac{(\mu_l^{(i)}-x_i)(\mu_k^{(j)}-x_j)}{\sigma_l^2\sigma_k^2} - \frac{\delta_{ij}}{\sigma_l^2}\right|,
    \end{equation}
    which vanishes with $|x|\to\infty$ for all pair $(l,k)$ not satisfying $\sigma_l=\sigma_k= \sigma_*$.

    If $\sigma_l=\sigma_k= \sigma_*$, we consider two cases. If $l=k$, we have that
    \begin{align*}
          |A_{l l}| & =  \frac{w_l^2\mathcal{N}_l^2}{\sum_{l,k} w_l w_k \mathcal{N}_l \mathcal{N}_k}\left|\frac{(\mu_l^{(i)}-x_i)(\mu_l^{(j)}-x_j) - (\mu_l^{(i)}-x_i)(\mu_l^{(j)}-x_j)}{\sigma_*^4} - \frac{1}{\sigma_*^2}\right|
          \\ & = \frac{w_l^2\mathcal{N}_l^2}{\sum_{l,k} w_l w_k \mathcal{N}_l \mathcal{N}_k}\left(\frac{1}{\sigma_*^2}\right)\leq \frac{1}{\sigma_*^2}.
    \end{align*}
    
    If $l\neq k$, denoting $a=(\mu_l^{(i)}-x_i)$, $b=(\mu_l^{(j)}-x_j)$, $c=(\mu_k^{(i)}-x_i)$ and $d=(\mu_k^{(j)}-x_j)$, we have that
    \begin{align*}
          A_{lk}+A_{kl} & = \frac{ w_l w_k \mathcal{N}_l \mathcal{N}_k }{\sum_{l',k'} w_{l'} w_{k'} \mathcal{N}_{l'} \mathcal{N}_{k'}}\left(\frac{ab - ad + cd - cb}{\sigma_*^4} - \frac{2\delta_{ij}}{\sigma_*^2}\right).
    \end{align*}
    Using that $ab - ad + cd - cb = (a-c)(b-d)$, we have
    \begin{align*}
          |A_{lk}+A_{kl}| & = \frac{ w_l w_k \mathcal{N}_l \mathcal{N}_k }{\sum_{l',k'} w_{l'} w_{k'} \mathcal{N}_{l'} \mathcal{N}_{k'}}\left|\frac{(\mu_l^{(i)} - \mu_k^{(i)})(\mu_l^{(j)} - \mu_k^{(j)})}{\sigma_*^4} - \frac{2\delta_{ij}}{\sigma_*^2}\right|
          \\ & \leq \left|\frac{(\mu_l^{(i)} - \mu_k^{(i)})(\mu_l^{(j)} - \mu_k^{(j)})}{\sigma_*^4} - \frac{2\delta_{ij}}{\sigma_*^2}\right|,
    \end{align*}
    which is constant  in $x$. Therefore
    \begin{equation*}
        \left|\partial^2_{x_i x_j}(\log p(x))\right| = \left|\sum_{l,k} A_{lk} (x)\right| \leq \sum_{l=k} \left|A_{lk} (x)\right| + \sum_{l<k} \left|A_{lk} (x) + A_{kl} (x)\right|
    \end{equation*}
    is bounded for all $i,j$, and thus $\nabla\log p(x)$ is Lipschitz.
\end{proof}

\subsection{Section 5: constant \texorpdfstring{$\gamma$}{gamma}} \label{app: proofs sec5}

The sample quality and sample coverage KL divergences \eqref{KLQPdivergencescalaranalyticexample} and \eqref{KLPQdivergencescalaranalyticexample} can be written more succinctly as
\[
    h_q(\gamma) = H(\tilde p_T | p_0) = \frac{1}{2} \left( v(\gamma) - \ln v(\gamma) - 1  +  m(\gamma)^2 \right),
\] 
and
\[
    h_c(\gamma) = H(p_0 | \tilde p_T) = \frac{1}{2} \left(\frac{1}{v(\gamma)} + \ln v(\gamma) - 1 + \frac{m(\gamma)^2}{v(\gamma)}\right),
\]
where
\[
    v(\gamma) = \frac{\tilde\sigma_\theta(T)^2}{\bar\sigma(T)^2} = \begin{cases}
         \frac{\gamma\alpha_\theta}{1 + \gamma - \alpha_\theta} + \left(\beta_T - \frac{\gamma\alpha_\theta}{1 + \gamma - \alpha_\theta}\right)\zeta_T^{\frac{1 + \gamma}{\alpha_\theta} - 1}, & 1 + \gamma \neq \alpha_\theta, \\
        \beta_T  - \gamma\ln\zeta_T, & 1 + \gamma = \alpha_\theta,
    \end{cases}
\]
and
\[
    m(\gamma) := \frac{\mu_0 - \tilde\mu_\theta(T)}{\bar\sigma(T)} = \frac{\mu_0 - \mu_\theta - (\mu_T - \mu_\theta)\zeta_T^{\frac{1}{2}\frac{1+\gamma}{\alpha_\theta}}}{\sigma_0}
\]
with $\zeta_T = \bar\sigma(T)^2/\bar\sigma(0)^2 = \sigma_0^2 / ( \sigma_0^2 + \sigma_{\text{diff}}(T)^2 ) < 1.$

\subsubsection{Exact score}

When the score for the approximate reverse problem is exact, i.e. when $\alpha_\theta = 1$ and $\mu_\theta = \mu_0,$ we have
\[
    v(\gamma) = 1 + \left(\beta_T - 1\right)\zeta_T^\gamma, \qquad
    m(\gamma) = \frac{\mu_0 - \mu_T}{\sigma_0}\zeta_T^{\frac{1 + \gamma}{2}},
\]
for any $\gamma \geq 0.$ Since $\zeta_T < 1,$ they both converge monotonically to $1$ and $0$, respectively. Moreover, the ratio $m(\gamma)^2/v(\gamma)$ is either zero for all $\gamma\geq 0$, which happens when $\mu_T = \mu_0$, or decreases monotonically to $0,$ as $\gamma \rightarrow \infty$ (indeed, with $\mu_T \neq \mu_0$, just note that the reciprocal is of the form $v(\gamma)/m(\gamma)^2 = A\zeta^{-\gamma} + B$, with $A$ and $B$ independent of $\gamma$ and $A$ strictly positive). Thus, both KL divergences $h_q(\gamma)$ and $h_c(\gamma)$ are monotonically decreasing to $0$ as $\gamma \rightarrow \infty.$ This means, as expected, that, when the score is exact, stochastic sampling is better than reverse ODE sampling.

\subsubsection{Exact mean with aligned variances}

When $\alpha_\theta \neq 1$ but $\mu_T = \mu_\theta = \mu_0,$ the mean of the approximate reverse problem is exact, $\tilde\mu(\tau) = \mu_0,$ so that $m(\gamma) = 0,$ for all $\gamma > 0.$ In this case, $h_q(\gamma)$ and $h_c(\gamma)$ depend only on $v(\gamma),$ and their critical points occur when either $v'(\gamma) = 0$ or $v(\gamma) = 1.$ For the analysis of the critical points, we change the variable to $u = (1 - \alpha_\theta)/\alpha_\theta + \gamma/\alpha_\theta,$ which increases linearly with $\gamma = \gamma(u) = \alpha_\theta u + \alpha_\theta - 1$, and write
\begin{align*}
   \tilde v(u) = v(\gamma(u)) & = \frac{\gamma(u)}{u} + (\beta_T - \frac{\gamma(u)}{u}) \zeta_T^u \\ 
   & = \alpha_\theta  + \frac{\alpha_\theta - 1}{u} + (\beta_T - \alpha_\theta - \frac{\alpha_\theta - 1}{u}) \zeta_T^u \\
   & = \alpha_\theta + (\beta_T - \alpha_\theta)\zeta_T^u + (\alpha_\theta - 1)\frac{1 - \zeta_T^u}{u} \\
   & = \alpha_\theta + (\beta_T - \alpha_\theta)\zeta_T^u + (\alpha_\theta - 1)\psi(u),
\end{align*}
where
\[ \psi(u) = \frac{1 - \zeta_T^u}{u}.
\]
Notice that both $\zeta_T^u$ and $\psi(u)$ are positive, decreasing, and convex in $u$.

When $0 < \beta_T \leq \alpha_\theta < 1,$ the second term is non-positive and the third term is negative, so that $v(\gamma) = \tilde v(u(\gamma)) < \alpha_\theta < 1,$ for all $\gamma \geq 0,$ and increasing in $\gamma.$ Thus, both $h_q(\gamma)$ and $h_c(\gamma)$ are decreasing in $\gamma,$ and we have
\begin{equation}\label{kl_limits_analytical}
    \begin{aligned}
        h_q(\gamma)\to & \left( \alpha_\theta - \ln \alpha_\theta - 1\right)/2 > 0,
    \\ h_c(\gamma)\to & \left( \alpha_\theta^{-1} + \ln \alpha_\theta - 1\right)/2 > 0, 
    \end{aligned}
\end{equation}
as $\gamma \rightarrow \infty.$

If $\beta_T \geq \alpha_\theta > 1,$ the second term is nonnegative and the third term is positive, so that $v(\gamma) = \tilde v(u(\gamma)) > \alpha_\theta > 1,$ for all $\gamma \geq 0,$ and decreasing in $\gamma.$ Thus, both $h_q(\gamma)$ and $h_c(\gamma)$ are decreasing in $\gamma$ and have the same limits \eqref{kl_limits_analytical}.

\subsubsection{Exact variances}\label{app: analytical_exact_variances}

When $\beta_T = \alpha_\theta = 1$, we have $v(\gamma) = 1,$ for all $\gamma \geq 0,$ and thus $h_q(\gamma) = h_c(\gamma) = m(\gamma)^2/2.$ In this case, we have a trichotomy depending on the values of $A:=\mu_0-\mu_\theta$ and $B:=\mu_T-\mu_\theta$. If $AB\leq0$, corresponding to the orderings $\mu_T\leq\mu_\theta\leq\mu_0$ or $\mu_T\geq\mu_\theta\geq\mu_0$, then $m(\gamma)^2$ is decreasing, hence both $h_q(\gamma)$ and $h_c(\gamma)$ decrease monotonically to $(\mu_0 - \mu_\theta)^2/2\sigma_0^2$, in which case SDE sampling is always better than ODE sampling and improves with increased $\gamma$. If $0<A/B<\zeta_T^{1/2}$, then $m(\gamma)^2$ has a unique zero at a finite nonzero value of $\gamma,$ and thus perfect generation (zero values for both KL divergences) is achieved at this value of $\gamma.$ Stochastic sampling with the stochasticity parameter near this optimal value outperforms both ODE sampling and highly stochastic sampling, while some values of $\gamma$ away from the optimum may perform worse than ODE sampling. Finally, if $\zeta_T^{1/2}\leq A/B$, then $m(\gamma)^2$ is increasing, in which case ODE sampling is always better than SDE sampling.

\subsubsection{Aligned variances and means}

When we have $0 < \beta_T \leq \alpha_\theta < 1$ together with $\mu_T \leq \mu_\theta  \leq \mu_0$, then $v(\gamma)$ increases monotonically to $1$ while $\mu(\gamma)$ decreases monotonically to $0$ and then $m(\gamma)^2/v(\gamma)$ also decreases monotonically to zero. In this case, both KL divergences decrease monotonically as $\gamma \rightarrow \infty$, meaning that stochastic sampling is always better than reverse ODE sampling, improving with increased $\gamma.$ If we have different combinations of either $0 < \beta_T \leq \alpha_\theta < 1$ or $\beta_T \geq \alpha_\theta > 1,$ together with either $\mu_T \leq \mu_\theta  \leq \mu_0$ or $\mu_T \geq \mu_\theta \geq \mu_0,$ then we do not have control of the term $m(\gamma)^2/v(\gamma)$, and we can only guarantee that the KL divergence $h_q(\gamma)$ decreases monotonically as $\gamma \rightarrow \infty.$

\subsubsection{Other cases}

In other cases, however, the different terms have different and competing behaviors and the KL divergences may have local or global minima or maxima at finite nonzero values of $\gamma,$ as illustrated in Figure \ref{fig:analyticKLongammaregimes}.

\subsection{Section 5: optimal control}\label{app: optimal control}

Following the setup of Section \ref{sec: analytical example}, we now consider
a stochasticity parameter $\gamma(\tau) \in [0, \gmax]$. From equations \eqref{KLQPdivergencescalaranalyticexample} and \eqref{KLPQdivergencescalaranalyticexample}, the final KL divergences depend on $\gamma$ through the terminal mean $\tilde{\mu}_\theta(T)$ and variance $\sapprox(T)^2$. To formulate an optimal control problem with terminal cost, we consider the state space $(m,u)$, where $m(\tau):=\tilde{\mu}_\theta(\tau)$ and $u(\tau):=\sapprox(\tau)^2$.

\subsubsection{Reduction to clock time}

Since the approximate reverse SDE \eqref{AE:pertubedequation} is linear, the evolution of the mean and variance of its solution $\tilde{p}_\tau$ is given by the ODEs
\begin{align}
 \dot{m}(\tau) &= -\frac{\bar g(\tau)^2(1+\gamma(\tau))}{2\alpha_\theta\smean(\tau)^2}\left(m - \mu_\theta\right), &m(0) &= \mu_T,\\
  \dot u(\tau) &= -\frac{\bar g(\tau)^2(1+\gamma(\tau))}{\alpha_\theta\smean(\tau)^2}\,u(\tau) + \gamma(\tau)\,\bar g(\tau)^2, \qquad &u(0) &= \beta_T\,\smean(0)^2.
\end{align}

We consider the clock time
\[
s(\tau):=\smean(0)^2 -\smean(\tau)^2,
\]
and the variance (in the new variable) $\phi(s) := \smean(0)^2 - s$. Since $ds/d\tau = -d(\smean(\tau)^2)/d\tau = \bar g(\tau)^2$, in clock time the dynamics become
\begin{align}
  m'(s) &= -\frac{1+\gamma(s)}{2\alpha_\theta\,\phi(s)}\left(m - \mu_\theta\right), &m(0) &= \mu_T, \label{eq:delta_s} \\
  u'(s) &= -\frac{(1+\gamma(s))\,u(s)}{\alpha_\theta\,\phi(s)} + \gamma(s), \qquad &u(0) &= \beta_T\,\phi(0), \label{eq:var_s}   
\end{align}
making it \textit{independent of the specific choice of $g$}.

It is convenient to work with the normalized variance ratio
\begin{equation}\label{eq:vdef}
  v(s) := \frac{u(s)}{\phi(s)},
\end{equation}
which satisfies
\begin{equation}\label{eq:v_s}
  v'(s) = \frac{1}{\phi(s)}\left[\,v(s)\left(1 - \frac{1+\gamma(s)}{\alpha_\theta}\right) + \gamma(s)\,\right].
\end{equation}

\subsubsection{Optimal control formulation}

Let $\Phi(m(s_T), u(s_T))$ denote either of the terminal KL divergences. We formulate the problem of minimizing the terminal cost as a standard optimal control problem for $s\in[0, s_T]$, with states $m$ and $u$, and control $\gamma \in [0, \gmax]$.

The control Hamiltonian is given by
\begin{equation}\label{eq:ham}
  H(s, m, u, \gamma, p_m, p_u) = - p_m\frac{1+\gamma}{2\alpha_\theta\phi(s)}\left(m - \mu_\theta\right) + p_u\left(-\frac{(1+\gamma)u}{\alpha_\theta\phi(s)} + \gamma\right),
\end{equation}
with costates $p_m(s)$ and $p_u(s)$ satisfying
\begin{equation}\label{eq:costate}
  p_m'(s) = \frac{(1+\gamma(s))\,p_m(s)}{2\alpha_\theta\phi(s)}, \qquad p_u'(s) = \frac{(1+\gamma(s))\,p_u(s)}{\alpha_\theta\phi(s)},
\end{equation}
and terminal transversality conditions
\begin{equation}\label{eq:trans}
  p_m(s_T) = \partial_m \Phi(m(s_T), u(s_T)), \qquad p_u(s_T) = \partial_u \Phi(m(s_T), u(s_T)).
\end{equation}
Observe that, by \eqref{eq:delta_s} and \eqref{eq:costate}, we have that
\begin{equation}\label{eq:conserved_C}
  \frac{d}{ds}\left[p_m(s)\big(m(s) - \mu_\theta\big)\right] = p_m'(s)\big(m(s) - \mu_\theta\big) + p_m(s)m'(s) = 0.
\end{equation}
We therefore denote $C := p_m(s)\big(m(s) - \mu_\theta\big)$.

The Hamiltonian is linear in $\gamma$, so by Pontryagin's minimum principle the optimal control is bang-bang with switching function
\begin{equation}\label{eq:switch}
  \psi(s) := p_u(s)\left(1 - \frac{u(s)}{\alpha_\theta\phi(s)}\right) - \frac{p_m(s)\big(m(s) - \mu_\theta\big)}{2\alpha_\theta\phi(s)} = p_u(s)\left(1 - \frac{v(s)}{\alpha_\theta}\right) - \frac{C}{2\alpha_\theta\phi(s)},
\end{equation}
and optimal control
\[
  \gamma^*(s) = \begin{cases}
    0 & \text{if } \psi(s) > 0, \\
    \gmax & \text{if } \psi(s) < 0, \\
    \text{singular} & \text{if } \psi(s) = 0 \text{ on a nontrivial interval}.
  \end{cases}
\]

Observe that the costate ODE \eqref{eq:costate} is linear and homogeneous in $p_u$, so $p_u$ never vanishes unless it is identically zero. In particular, it never changes sign.

\subsubsection{At most one switch}

Interestingly, this property is not restricted to KL divergences as a terminal cost.

\begin{theorem}\label{thm:main}
  Let $\alpha_\theta \neq 1$. For any $\mathcal{C}^1$ terminal cost $\Phi(m(s_T), u(s_T))$ such that $$\partial_u\Phi(m(s_T), u(s_T))\neq0$$ at optimality, the optimal control $\gamma^* : [0, s_T] \to [0, \gmax]$ is bang-bang with at most one switch between $\gamma = 0$ and $\gamma = \gmax$. When a switch occurs at $s^* \in (0, s_T)$, it takes place at the variance threshold
  \[
    v(s^*) = \alpha_\theta - \frac{C}{2p_u(s^*)\phi(s^*)}.
  \]
\end{theorem}

\begin{proof}
    \textbf{1) Derivative of the switching function $\psi$}

    Differentiating \eqref{eq:switch} and using that $\phi'(s) = -1$, we have
  \[
    \psi'(s) = p_u'(s)\left(1 - \frac{v(s)}{\alpha_\theta}\right) - \frac{p_u(s)v'(s)}{\alpha_\theta} - \frac{C}{2\alpha_\theta\phi(s)^2}.
  \]
  Substituting $p_u'$ from \eqref{eq:costate} and $v'$ from \eqref{eq:v_s}, the terms involving $\gamma(s)$ cancel, yielding
  \begin{equation}\label{eq:psi_prime_general}
    \psi'(s) = \frac{p_u(s)(1 - v(s))}{\alpha_\theta\phi(s)} - \frac{C}{2\alpha_\theta\phi(s)^2}.
  \end{equation}

  \textbf{2) Derivative of $\psi$ at its zeros}
  
  When $\psi(s)=0$, from \eqref{eq:switch} we have
  \[
    v(s) = \alpha_\theta - \frac{C}{2p_u(s)\phi(s)},
  \]
  which, substituted in \eqref{eq:psi_prime_general}, yields
  \begin{equation}\label{eq:psi_prime_zero}
      \psi'(s) = \frac{p_u(s)}{\alpha_\theta\phi(s)}\left(1 - \alpha_\theta + \frac{C}{2p_u(s)\phi(s)}\right) - \frac{C}{2\alpha_\theta\phi(s)^2} = \frac{p_u(s)(1 - \alpha_\theta)}{\alpha_\theta\phi(s)}.
  \end{equation}
  Since $\phi(s) > 0$ and $\mathrm{sign}(p_u(s))$ is constant, $\psi'$ has the same sign at any zero of $\psi$.

  \textbf{3) At most one switch}

  If $\alpha_\theta\neq1$, $\psi'$ is either strictly positive or strictly negative at any zero of $\psi$. Therefore there is at most one zero, and $\gamma*$ has at most one switch.
\end{proof}

\begin{remark}
    The sign of $\psi'(s^*)$ determines the direction of the switch: if $\psi'(s^*) > 0$, $\gamma^*$ jumps from $\gmax$ to $0$ (stochastic then deterministic), and vice-versa. By equation \eqref{eq:psi_prime_zero}, it is determined by $\mathrm{sign}\big(p_u(s)(1-\alpha_\theta)\big)$. Since the sign of $p_u$ is constant, it can be obtained from the terminal conditions \eqref{eq:trans}.
\end{remark}

\subsubsection{Multidimensional case (at most \texorpdfstring{$N$}{N} switches)}\label{app: optimal control multi}

In a general $N$-dimensional model, the state space expands to $(\mathbf{m}, \mathbf{u}) = (m_1, \dots, m_N, u_1, \dots, u_N)$ with decoupled dynamics across dimensions, and the control Hamiltonian is given by
\begin{equation}
    H(s, \mathbf{m}, \mathbf{u}, \gamma, \mathbf{p_m}, \mathbf{p_u}) = \sum_{i=1}^N \Bigg[ - p_{m_i}\frac{1+\gamma}{2\alpha_{\theta,i}\phi(s)}\left(m_i - \mu_{\theta,i}\right) + p_{u_i}\left(-\frac{(1+\gamma)u_i}{\alpha_{\theta,i}\phi(s)} + \gamma\right) \Bigg].
\end{equation}
By the same costate dynamics as in the unidimensional case, the quantities $C_i = p_{m_i}(s)\big(m_i(s) - \mu_{\theta,i}\big)$ are conserved for each dimension $i \in \{1, \dots, N\}$. The switching function therefore becomes a sum of the unidimensional switching components:
\begin{equation}\label{eq:switch_multi_N}
  \psi(s) = \sum_{i=1}^N \underbrace{\left[ p_{u_i}(s)\left(1 - \frac{v_i(s)}{\alpha_{\theta, i}}\right) - \frac{C_i}{2\alpha_{\theta, i}\phi(s)} \right]}_{:= \psi_i(s)}.
\end{equation}

Then we have the following extension of the unidimensional one-switch result.

\begin{theorem}\label{thm:multi_N}
  Let $\alpha_{\theta,i}\neq1$ for $1\leq i\leq N$, and the terminal cost satisfy the condition in Theorem \ref{thm:main}. Then the optimal control $\gamma^* : [0, s_T] \to [0, \gmax]$ for the $N$-dimensional problem is bang-bang with at most $m\leq N$ switches, where $m$ is the number of distinct score scaling errors $\alpha_{\theta,1},\dots,\alpha_{\theta,N}$.
\end{theorem}

\begin{proof}
Differentiating $\psi_i(s)$ and substituting the state and costate ODEs, we obtain an analogous formula as \eqref{eq:psi_prime_general} for the derivative of each component
  \begin{equation}\label{eq:psi_prime__N}
      \psi_i'(s) = \frac{p_{u_i}(s)(1-v_i(s))}{\alpha_{\theta,i}\phi(s)} - \frac{C_i}{2\alpha_{\theta,i}\phi(s)^2}.
  \end{equation}
  Summing and subtracting $p_{u_i}(1- 1/\alpha_{\theta,i})$, this can be rewritten as
  \begin{equation}\label{eq:psi_i_identity_N}
      \psi_i'(s) = \frac{1}{\phi(s)}\psi_i(s) + \frac{p_{u_i}(s)(1-\alpha_{\theta,i})}{\alpha_{\theta,i}\phi(s)},
  \end{equation}
  and summing over $1\leq i\leq N$ yields
  \begin{equation}
      \psi'(s) = \frac{1}{\phi(s)}\psi(s) + \frac{1}{\phi(s)}\underbrace{\sum_{i=1}^N \frac{p_{u_i}(s)(1-\alpha_{\theta,i})}{\alpha_{\theta,i}}}_{:= \Gamma(s)}.
  \end{equation}
  Therefore, at any zero of $\psi$ we have
  \begin{equation}
      \psi'(s^*)\Big|_{\psi=0} = \frac{1}{\phi(s^*)}\Gamma(s^*).
  \end{equation}
  Since the linear costate ODEs for $p_{u_i}$ \eqref{eq:costate} have the closed form solution
  \[
  p_{u_i}(s) = p_{u_i}(s_T) \exp\left(-\int_s^{s_T} \frac{1+\gamma(s')}{\alpha_{\theta,i}\phi(s')} ds'\right),
  \]
  the function $\Gamma(s)$ is a linear combination of $m\leq N$ real exponentials with monotonic arguments, where $m$ is the number of distinct score scaling errors $\alpha_{\theta,1},\dots,\alpha_{\theta,N}$. By properties of exponential polynomials, $\Gamma(s)$ can change sign at most $m-1$ times, allowing at most $m$ sign changes (switches) for $\psi$.
\end{proof}

\section{Scalar counterexample to the LSI}\label{app: counterexample_lsi}

For a probability measure on $\mathbb{R}$, the LSI is equivalent to the following integrability characterization, given by \cite{BOBKOV1999}.

\begin{theorem}[Bobkov \& Götze]
    A measure $\mu$ in $\mathbb{R}$ with cumulative distribution function $F$, density $p$ (of its absolutely continuous part) and median $m$ satisfies a log-Sobolev inequality if and only if both
    \begin{align}
        D_0 &:= \sup_{x<m} \left(F(x)\log\frac{1}{F(x)}\right)\int_x^m\frac{1}{p(t)}dt
        \\ D_1 &:= \sup_{x>m} \left((1-F(x)) \log\frac{1}{1-F(x)}\right)\int_m^x\frac{1}{p(t)}dt
    \end{align}
    are finite.
\end{theorem}

\begin{example}[A density in $\mathbb{R}$ bounded by Gaussians which does not satisfy an LSI]\label{ex: counterexample}
    Let
    \begin{equation}
        p(x) := \begin{cases}
            c_1e^{-c_2 x^2}, & \exists n\in\mathbb{N}, \; 2n\leq |x| < 2n+1
            \\ c_3e^{-c_4 x^2}, & \exists n\in\mathbb{N}, \; 2n+1\leq |x| <2n+2
        \end{cases}
    \end{equation}
    where $c_2<c_4$ and $c_1>c_3$ are taken such that $p$ is a normalized probability density. A plot can be seen in Figure \ref{fig: counterexample}.
    \begin{figure}
        \centering
        \includegraphics[width=0.8\linewidth]{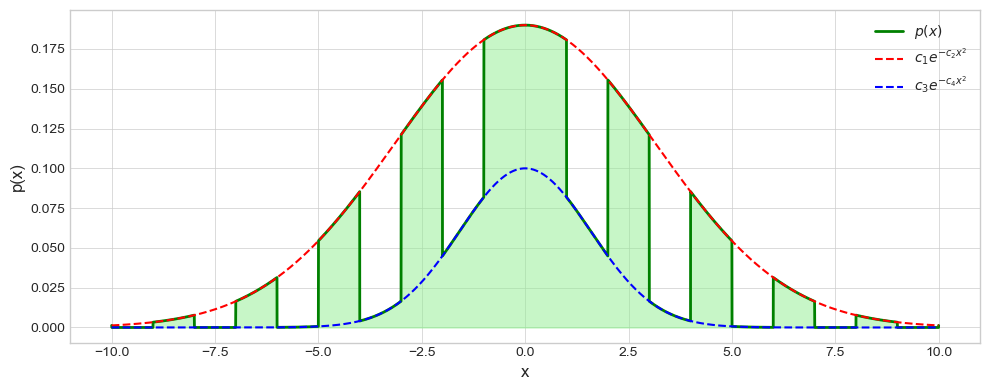}
        \caption{A plot of the density of Example \ref{ex: counterexample}, for specific choices of constants.}
        \label{fig: counterexample}
    \end{figure}

    For a sequence $t_n = 2n$, we have, since the integrand is decreasing,
    \begin{align*}
        \int_{t_n}^\infty p(s)ds &= \sum_{k=n}^{\infty} \left(\int_{2k}^{2k+1} c_1e^{-c_2 s^2}ds + \int_{2k+1}^{2k+2} c_3e^{-c_4 s^2}ds\right)
        \\ & \geq \sum_{k=n}^{\infty} \left(\int_{2k}^{2k+1} c_1e^{-c_2 s^2}ds\right) \geq \frac{1}{2}\int_{t_n}^\infty c_1e^{-c_2 s^2}ds,
    \end{align*}
    and also that
    \begin{equation*}
        \int_{t_n}^\infty p(s)ds \leq \int_{t_n}^\infty c_1e^{-c_2 s^2}ds.
    \end{equation*}
    Conversely, since the integrand is increasing,
    \begin{align*}
        \int_0^{t_n} \frac{1}{p(s)}ds &= \sum_{k=1}^{n} \left(\int_{2k-2}^{2k-1} c_1^{-1}e^{c_2 s^2}ds + \int_{2k-1}^{2k} c_3^{-1}e^{c_4 s^2}ds\right)
        \\ & \geq \frac{1}{2} \int_0^{t_n} c_3^{-1}e^{c_4 s^2}ds.
    \end{align*}
    Then, from the well known bounds for large $t$
    \begin{equation}
        \frac{t}{2ct^2+1} e^{-ct^2}\leq \int_t^\infty e^{-cs^2}ds\leq \frac{1}{2ct}e^{-ct^2}, \qquad \frac{e^{ct^2}}{2ct} \leq \int_0^te^{cs^2}ds \leq \frac{e^{ct^2}}{ct},
    \end{equation}
    we have, since $m=0$,
    \begin{align*}
        (1-F(t_n)) \log\frac{1}{1-F(t_n)}\int_m^{t_n}\frac{1}{p(s)}ds & = \left(\int_{t_n}^\infty p(s)ds\right)\log\frac{1}{\int_{t_n}^\infty p(s)ds} \int_0^{t_n}\frac{1}{p(s)}ds
        \\ & \geq \frac{1}{4}\left(\int_{t_n}^\infty c_1e^{-c_2 s^2}ds\right)\log\frac{1}{\int_{t_n}^\infty c_1e^{-c_2 s^2}ds}\int_0^{t_n} c_3^{-1}e^{c_4 s^2}ds
        \\  & \geq \frac{1}{4}\left( \frac{c_1 t_n e^{-c_2 t_n^2}}{2c_2{t_n}^2+1} \right) \left(\frac{c_3^{-1} e^{c_4t_n^2}}{2c_4t_n}\right) \log\frac{2c_2t_n}{c_1 e^{-c_2t_n^2}} \to \infty,
    \end{align*}
    since $c_2<c_4$, and thus $D_1=\infty$.
\end{example}

Finally, note that a mollification of this $p(x)$ would have the same properties while being smooth (as a solution of the reverse-time SDE should be).

\section{Additional figures}\label{app:additional_figures}

We conclude with additional figures mentioned in the main text.

% Correlation thresholds
\begin{figure}[ht!]
    \begin{subfigure}[b]{\linewidth}
        \centering
        \includegraphics[width=0.9\linewidth]{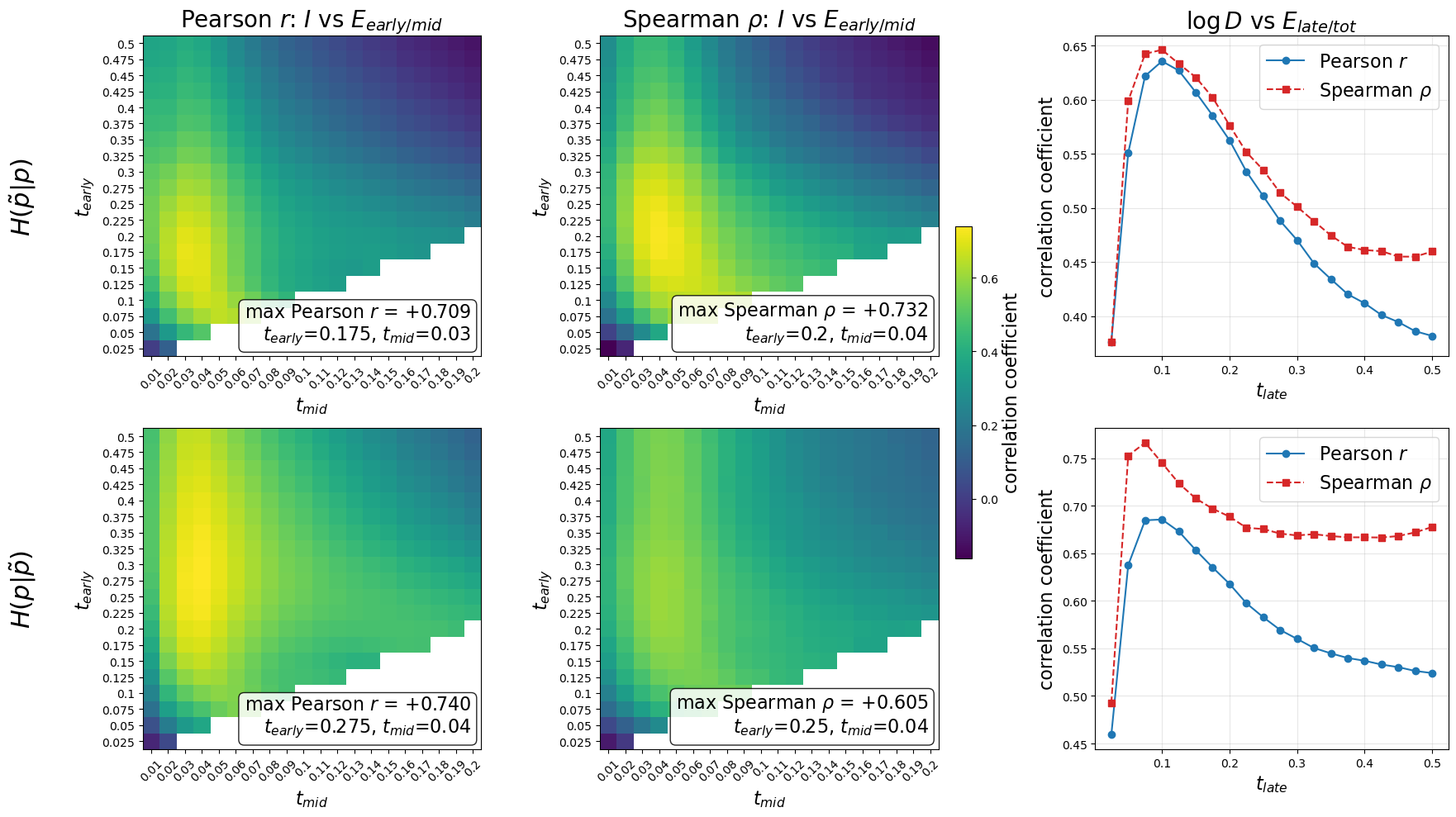}
        \caption{Correlations on the test set of checkpoints, used by Figure \ref{fig:correlation}.}
    \end{subfigure}
    \begin{subfigure}[b]{\linewidth}
        \centering
        \includegraphics[width=0.9\linewidth]{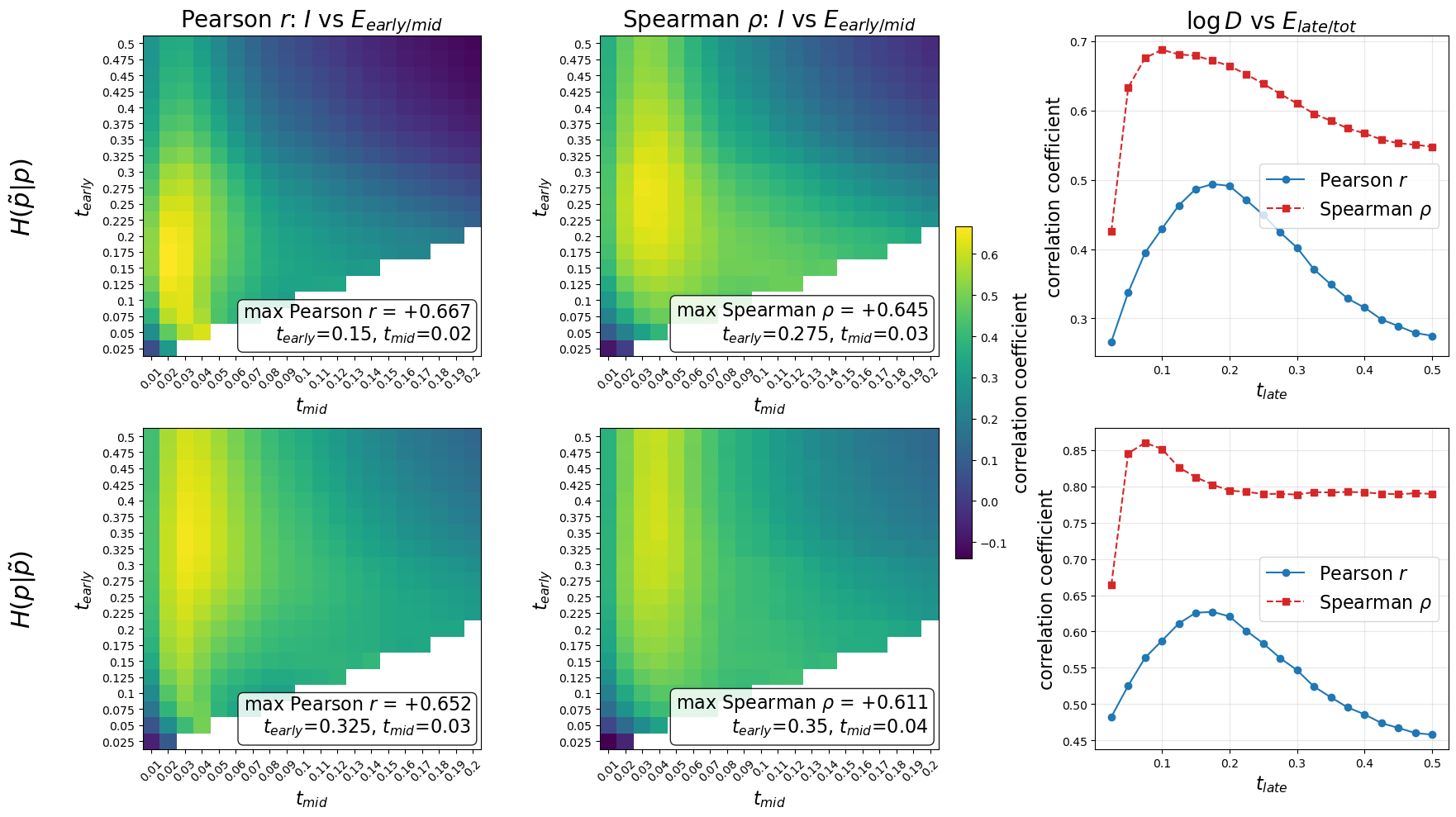}
        \caption{Correlations on the held-out set of checkpoints.}
    \end{subfigure}
    \caption{Plots for the sensitivity of the correlations shown in Figure \ref{fig:correlation} on the choices of the score error thresholds $t_\text{early}$, $t_\text{mid}$ and $t_\text{late}$. Each of the sets consist of $144$ epoch checkpoints from $3$ independent training runs.}
    \label{fig: correlation thresholds}
\end{figure}

% Mixture of Gaussians colormaps
\begin{figure}[ht!]
    \begin{subfigure}[b]{\linewidth}
        \centering
        \includegraphics[width=0.32\linewidth]{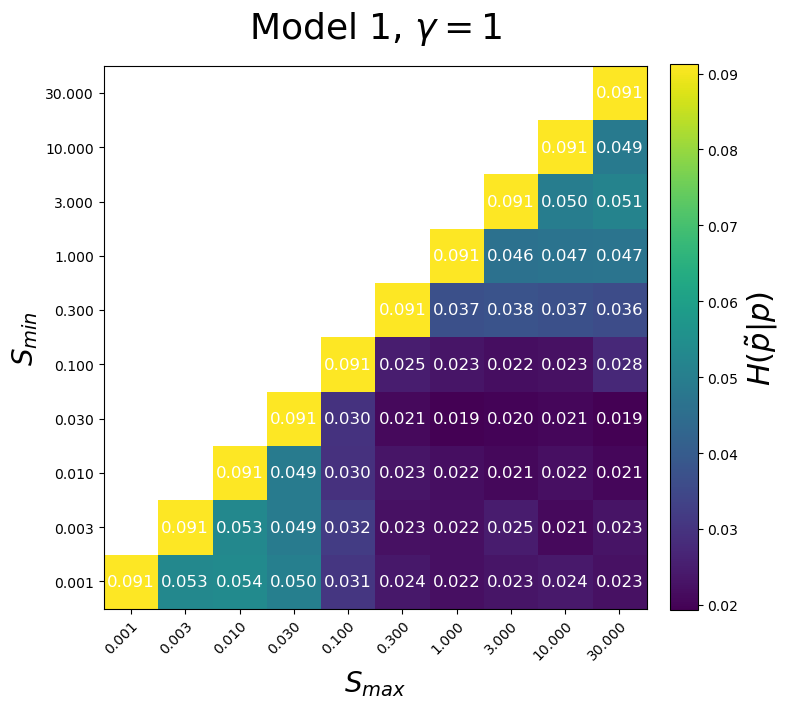}
        \includegraphics[width=0.32\linewidth]{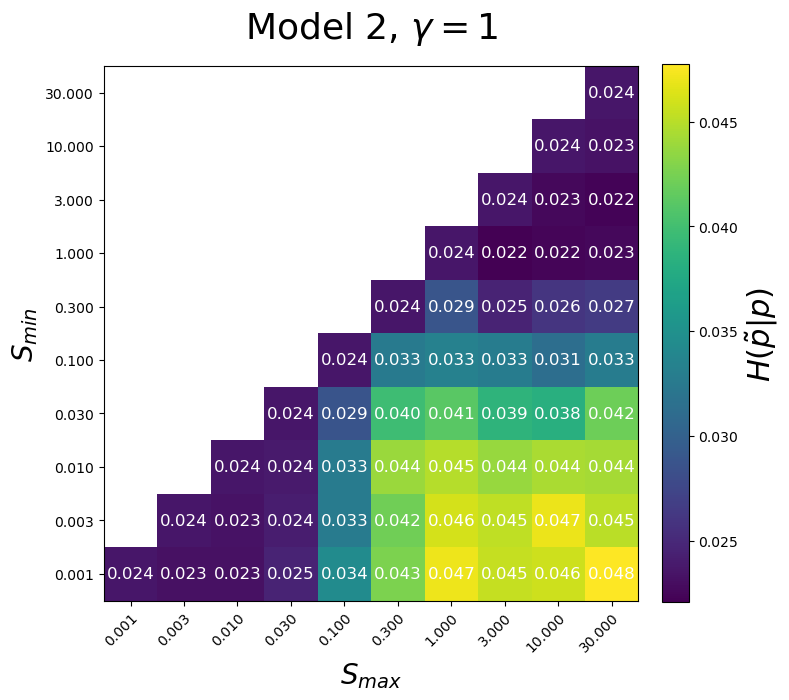}
        \includegraphics[width=0.32\linewidth]{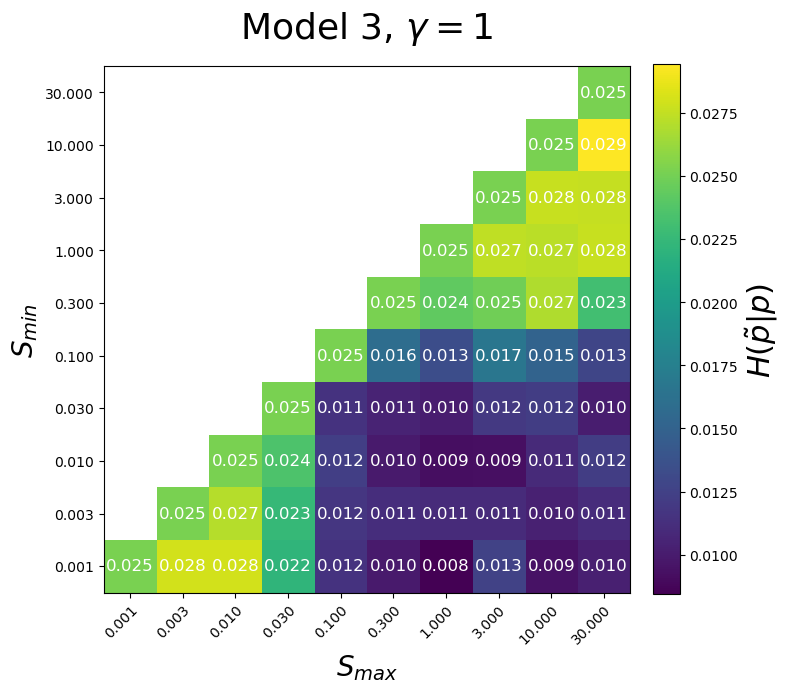}
        \includegraphics[width=0.32\linewidth]{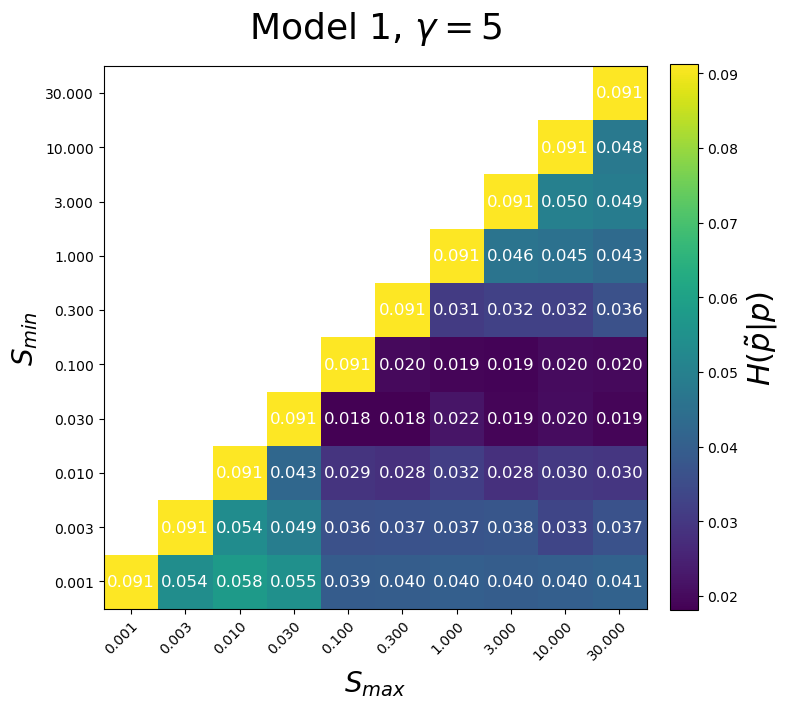}
        \includegraphics[width=0.32\linewidth]{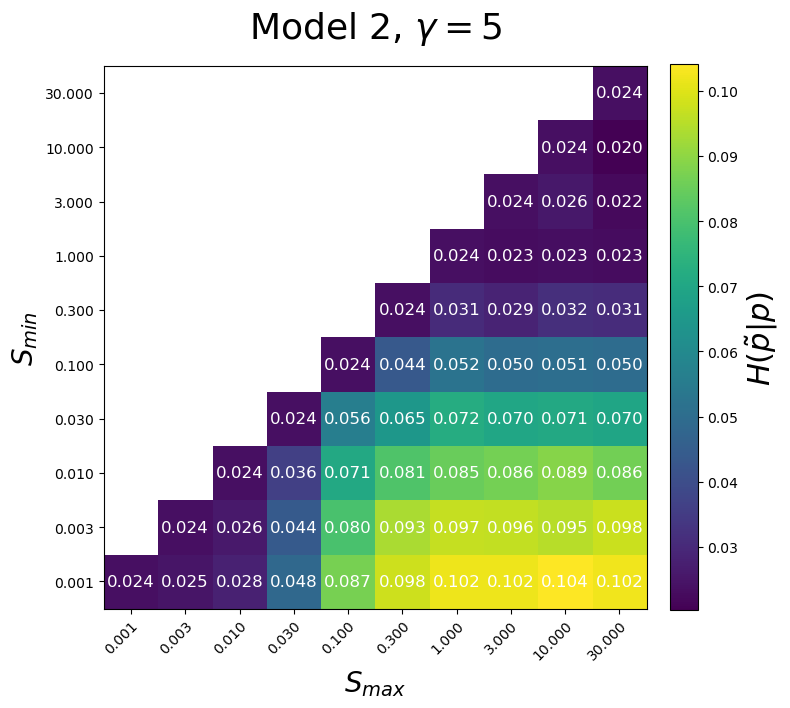}
        \includegraphics[width=0.32\linewidth]{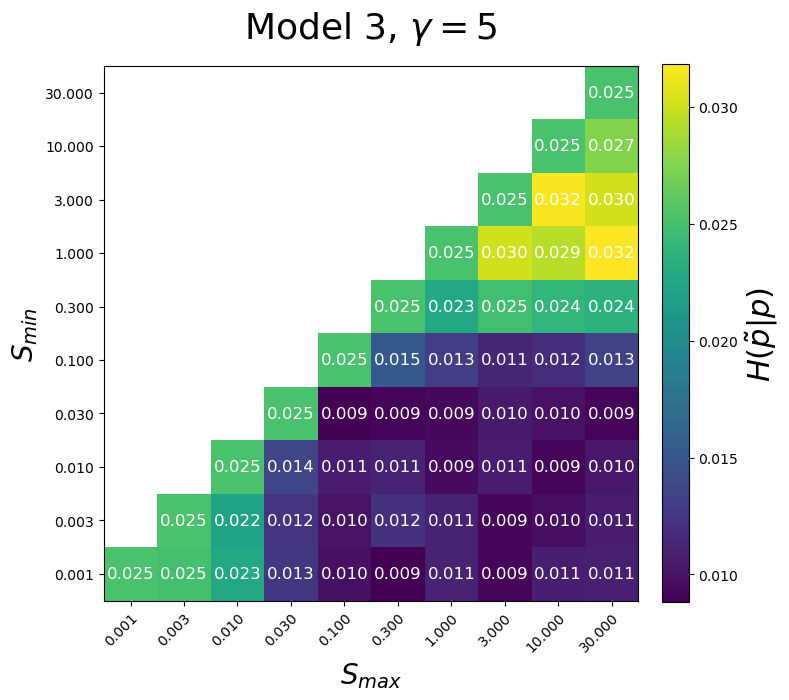}
        \caption{Colormaps for $H(\tilde{p}|p)$.}
        \label{fig:colormap_kl1}
    \end{subfigure}
    \begin{subfigure}[b]{\linewidth}
        \centering
        \includegraphics[width=0.32\linewidth]{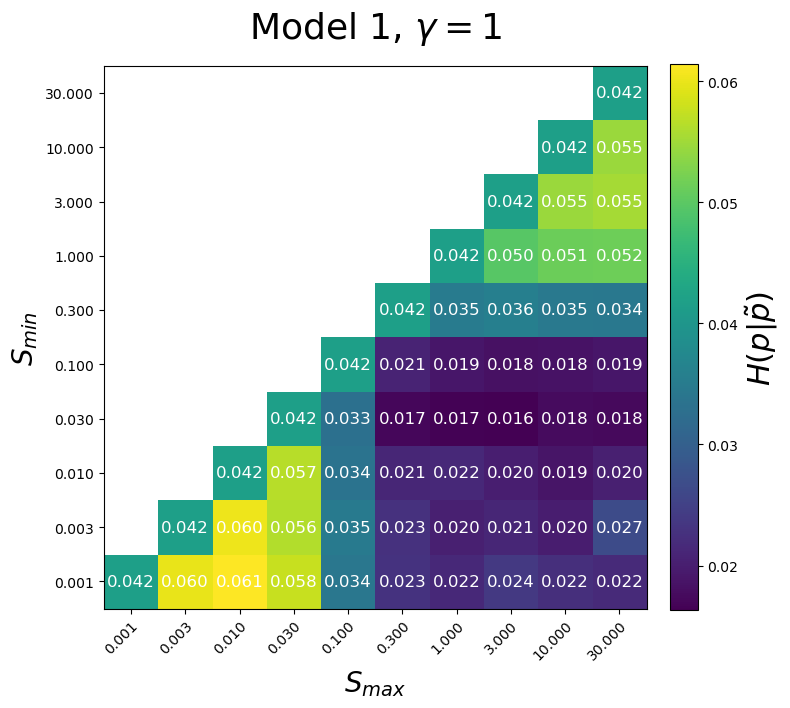}
        \includegraphics[width=0.32\linewidth]{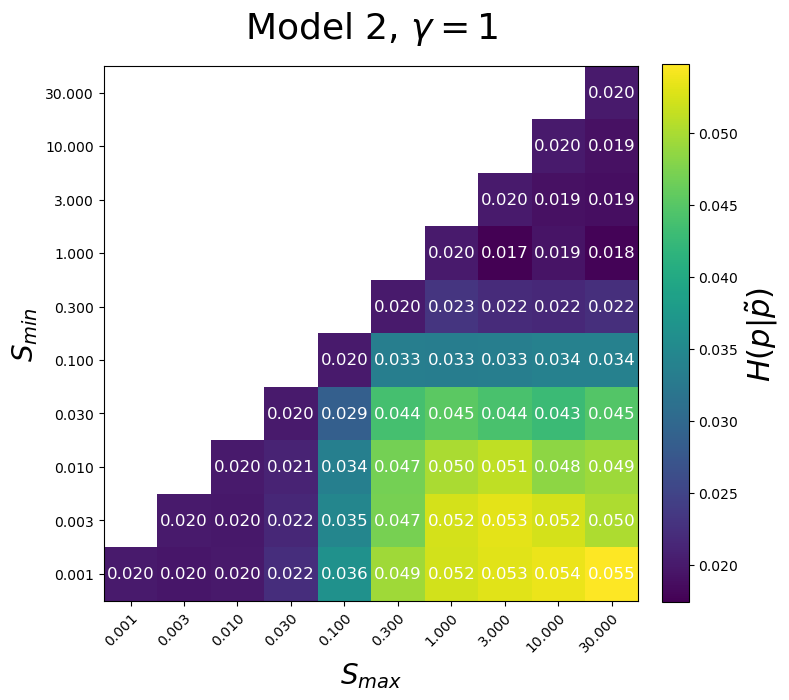}
        \includegraphics[width=0.32\linewidth]{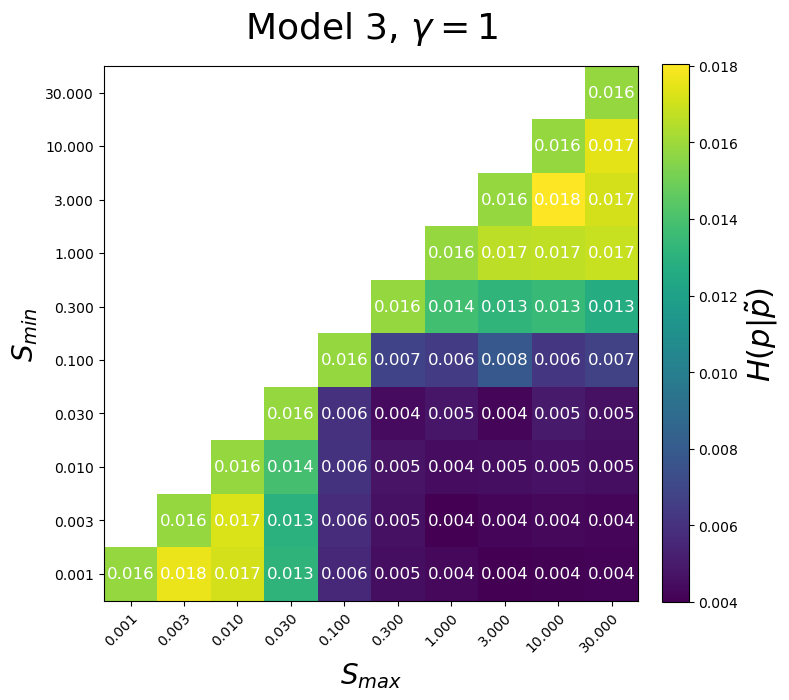}
        \includegraphics[width=0.32\linewidth]{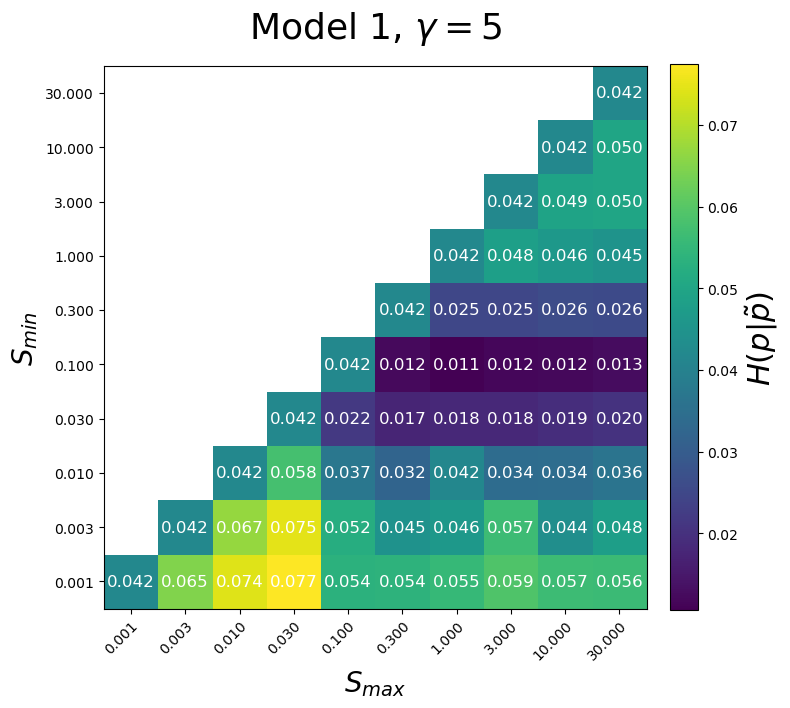}
        \includegraphics[width=0.32\linewidth]{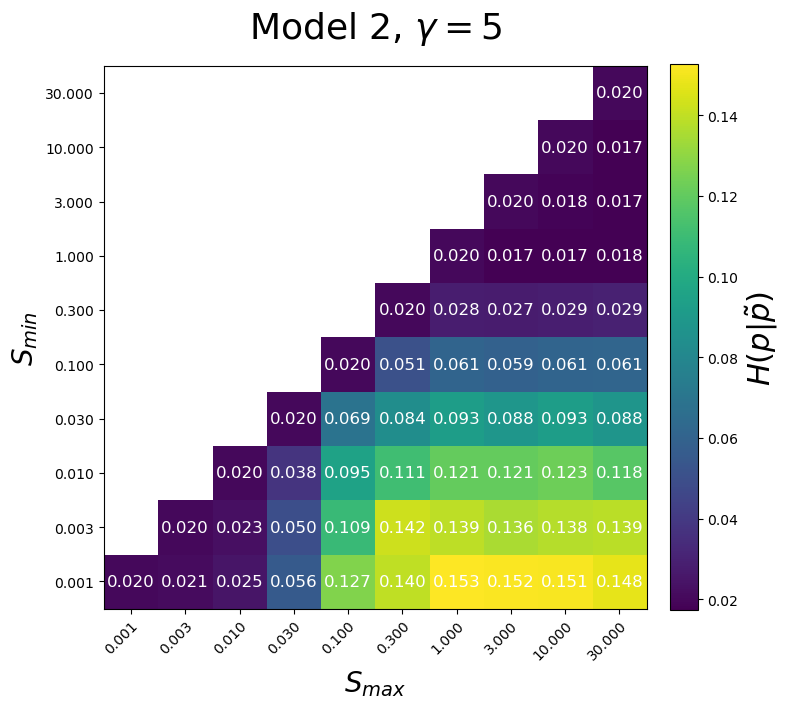}
        \includegraphics[width=0.32\linewidth]{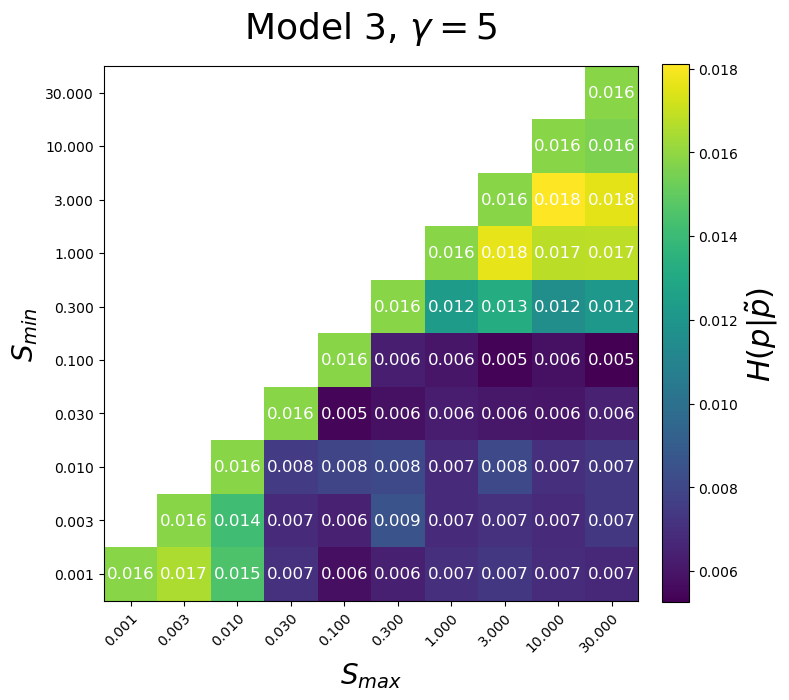}
        \caption{Colormaps for $H(p|\tilde{p})$.}
        \label{fig:colormap_kl2}
    \end{subfigure}
    \caption{Final KL divergences of the three models of Figure \ref{fig: gamma curves and error norms}. Although the values of both divergences (Figures \ref{fig:colormap_kl1} and \ref{fig:colormap_kl2}) can differ, the basic beneficial/detrimental regions are similar within each model.}
    \label{fig: colormaps}
\end{figure}

% CIFAR-10 colormaps
\begin{figure}[ht!]
    \centering
    \includegraphics[width=0.48\linewidth]{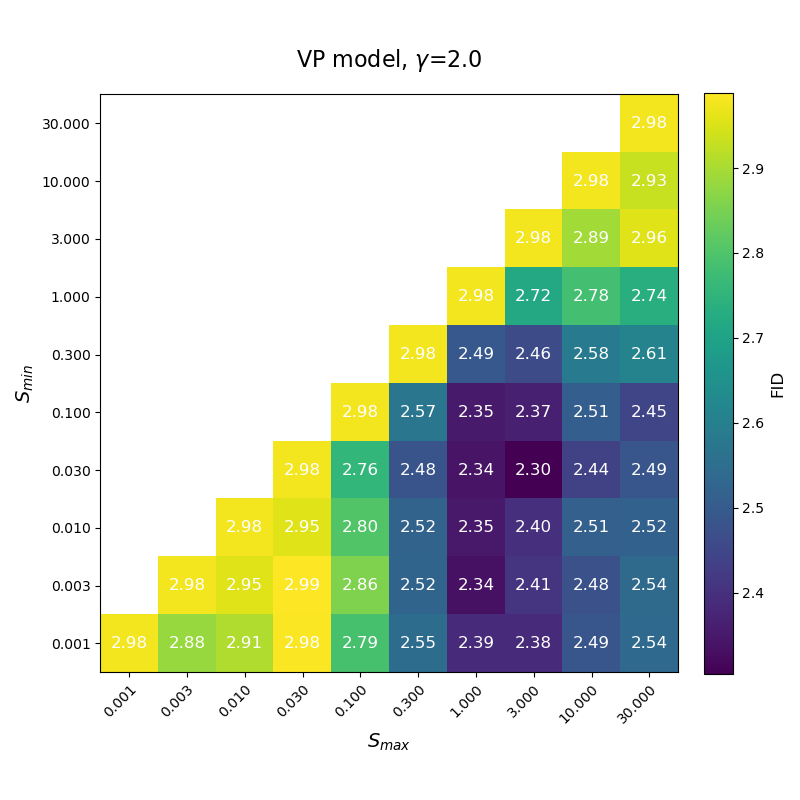}
    \includegraphics[width=0.48\linewidth]{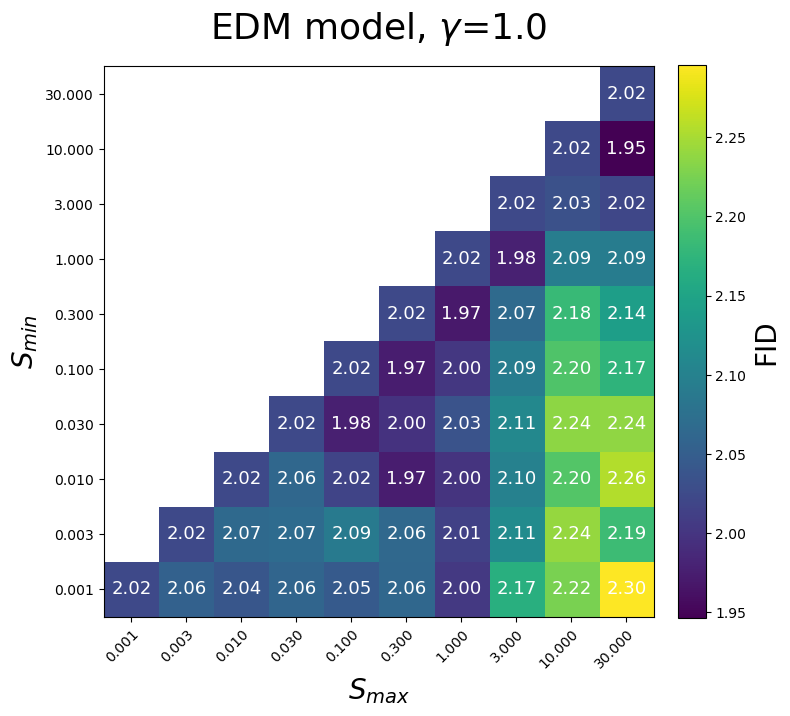}

    \caption{FID scores for the pretrained CIFAR-10 models from \cite{karras}.}
    \label{fig:colormaps_cifar10}
\end{figure}

\end{document}